\documentclass[10pt,journal,compsoc]{IEEEtran}
\usepackage{graphicx}
\usepackage{subcaption}
\usepackage{array}
\usepackage{color}
\usepackage[dvipsnames]{xcolor}
\usepackage{times}
\usepackage{amsmath}
\usepackage{amssymb}
\usepackage{eqnarray}
\usepackage{algorithm}
\usepackage{algorithmicx}
\usepackage{algpseudocode}
\usepackage{xcolor}
\usepackage{tabularx}
\usepackage{colortbl}
\usepackage{epstopdf}
\usepackage{multirow}
\usepackage{MnSymbol}
\usepackage[english]{babel}
\usepackage[utf8]{inputenc} 
\usepackage[T1]{fontenc}    
\usepackage{hyperref}       
\usepackage{cleveref}
\usepackage{url}            
\usepackage{booktabs}       
\usepackage{amsfonts}       
\usepackage{nicefrac}       
\usepackage{microtype}      

\usepackage{amsthm}
\usepackage{bm}
\usepackage{bbm}
\usepackage{pythonhighlight}

\newtheorem{definition}{Definition}[section]
\newtheorem{proposition}{Proposition}[section]
\newtheorem{lemma}{Lemma}[section]
\newtheorem{theorem}{Theorem}[section]
\newtheorem{assumption}{Assumption}[section]

\graphicspath{{./Figure/}}
\ifCLASSOPTIONcompsoc
  \usepackage[nocompress]{cite}
\else
  \usepackage{cite}
\fi

\begin{document}

\title{On the Transferability and Discriminability of Repersentation Learning in Unsupervised Domain Adaptation}

\author{Wenwen~Qiang,
        Ziyin~Gu,
        Lingyu~Si,
	Jiangmeng~Li,
	Changwen~Zheng,
	Fuchun~Sun,~\IEEEmembership{Fellow,~IEEE},
	and~Hui~Xiong,~\IEEEmembership{Fellow,~IEEE}
	\IEEEcompsocitemizethanks{\IEEEcompsocthanksitem W. Qiang, Z. Gu, L. Si, J. Li, and C. Zheng are with the National Key Laboratory of Space Integrated Information System, Institute of Software Chinese Academy of Sciences, Beijing, China. They are also with the University of Chinese Academy of Sciences, Beijing, China. E-mail: qiangwenwen, ziyin2020, lingyu, jiangmeng2019, changwen@iscas.ac.cn.
	\IEEEcompsocthanksitem F. Su is with the National Key Laboratory of Space Integrated Information System, Department of Computer Science and Technology, Tsinghua University, Beijing, China. E-mail: fcsun@tsinghua.edu.cn.
	\IEEEcompsocthanksitem H. Xiong is with the Artificial Intelligence Thrust, Information Hub, Department of Computer Science \& Engineering, School of Engineering, Hong Kong University of Science and Technology, guangzhou, China. E-mail: xionghui@ust.hk.
    \IEEEcompsocthanksitem Corresponding author: Jiangmeng Li.
}}

\markboth{Submitted to IEEE Transactions on Pattern Analysis and Machine Intelligence}%
{Qiang \MakeLowercase{\textit{et al.}}: RLGLC}

\IEEEtitleabstractindextext{%
\begin{abstract}
In this paper, we addressed the limitation of relying solely on distribution alignment and source-domain empirical risk minimization in Unsupervised Domain Adaptation (UDA). Our information-theoretic analysis showed that this standard adversarial-based framework neglects the discriminability of target-domain features, leading to suboptimal performance. To bridge this theoretical–practical gap, we defined “good representation learning” as guaranteeing both transferability and discriminability, and proved that an additional loss term targeting target-domain discriminability is necessary. Building on these insights, we proposed a novel adversarial-based UDA framework that explicitly integrates a domain alignment objective with a discriminability-enhancing constraint. Instantiated as Domain-Invariant Representation Learning with Global and Local Consistency (RLGLC), our method leverages Asymmetrically-Relaxed Wasserstein of Wasserstein Distance (AR-WWD) to address class imbalance and semantic dimension weighting, and employs a local consistency mechanism to preserve fine-grained target-domain discriminative information. Extensive experiments across multiple benchmark datasets demonstrate that RLGLC consistently surpasses state-of-the-art methods, confirming the value of our theoretical perspective and underscoring the necessity of enforcing both transferability and discriminability in adversarial-based UDA. 
\end{abstract}

\begin{IEEEkeywords}
	Unsupervised Domain Adaptation, Representation Learning, Information Theory, Transferability, Discriminability.
\end{IEEEkeywords}
}

\maketitle

\IEEEdisplaynontitleabstractindextext

\IEEEpeerreviewmaketitle

\IEEEraisesectionheading{\section{Introduction}\label{submission}}

Despite their impressive successes, standard machine learning models typically require that the training and test data share an identical distribution. In practice, however, differences in data collection procedures and limited availability of labeled samples can introduce covariate shift. Under these conditions, a model trained on source domain labeled data may fail to generalize to a target domain exhibiting a distinct distribution and lacking labels. To overcome this challenge, researchers have developed unsupervised domain adaptation (UDA) techniques. UDA focuses on learning models that can accommodate distributional disparities between the source and target domains, thereby improving generalization and performance in real-world applications.

Adversarial-based representation learning methods for UDA have made significant strides in both algorithmic development and theoretical foundations \cite{ganin2016domain, zhang2020unsupervised, chang2022unified, sun2024transvqa, gao2024learning}. Methodologically, these approaches concentrate on aligning latent feature distributions across source and target domains by employing divergence metrics such as the Wasserstein distance \cite{shen2018wasserstein, arjovsky2017wasserstein} or the Kullback–Leibler (KL) divergence \cite{ganin2016domain, zhang2019bridging, zhang2020unsupervised}. This alignment ensures that target domain features resemble those in the source domain, thereby allowing a classifier trained in the source domain to be effectively reused for target domain predictions. From a theoretical standpoint, these methods aim to enhance the approximation of the expected risk over the target domain. Typically, they establish an upper bound on this expected risk, which includes a term capturing the distribution discrepancy across domains and an empirical risk term defined over the source domain. This framework explains the success of adversarial-based representation learning: by concurrently minimizing both the discrepancy term and the empirical risk term, the upper bound of the expected risk on the target domain is tightened, leading to a lower actual risk in that domain.

According to recent research \cite{wang2024probability, bai2024prompt, sun2024transvqa, gu2022unsupervised, gao2024learning, hupseudo, wei2024class, zhu2024versatile, jiahua2022and, kang2019contrastive, qiang2021robust, qiang2021auxiliary}, we have observed that adding a discriminability-enhancing loss term to adversarial-based UDA objectives can markedly improve performance. For example, one might generate pseudo-labels for target-domain data using a specific heuristic, then introduce a cross-entropy loss for these samples in the target domain. However, existing theoretical frameworks do not fully explain this improvement. To bridge this theoretical–practical gap, we propose revisiting adversarial-based representation learning for UDA from the perspective of “good representation learning”, and provide an information-theoretic account of why this gap arises. Concretely, we leverage conditional mutual information to define “good representation learning” as simultaneously ensuring transferability and discriminability (see \textbf{Definition \ref{d1}} and \textbf{Definition \ref{d3}}). We then prove that simply minimizing domain-discrepancy measures and the empirical risk in the source domain guarantees both properties for source-domain features, but only transferability for target-domain features (see \textbf{Theorem \ref{t1}}). We further offer an intuitive rationale in the final paragraph of \textbf{Section \ref{098}}, illustrating that reducing distributional discrepancy alone merely ensures partial preservation of discriminative information. An additional constraint is thus needed to secure the discriminability of target-domain features. While the empirical risk term addresses this need in the source domain, it is absent for target-domain data. Consequently, these findings elucidate why including an extra loss term dedicated to enhancing target-domain feature discriminability within the UDA objective function is both logical and effective.

Motivated by the above discussion and guided by information theory, this paper presents a new adversarial-based representation learning framework for UDA, detailed in Equation (\ref{Eq:bbbd}). In this framework, the first and second terms constrain the transferability of feature representations from different domains, while the third and fourth terms address their discriminability. A fifth term regularizes the model parameters. Building on the proof of \textbf{Theorem \ref{t1}}, we show that the first and second terms in Equation (\ref{Eq:bbbd}) can be instantiated as a discrepancy measure between feature distributions in different domains, and the third term can be realized as an empirical risk term on the source domain data distribution, leading to Equation (\ref{Eq:udbsasdasf}). Furthermore, \textbf{Theorem \ref{qw}} provides theoretical support for ensuring that feature representations in both the source and target domains exhibit transferability and discriminability under this new UDA framework. Based on this framework, we propose a novel UDA learning method called Domain-Invariant Representation Learning with Global and Local Consistency (RLGLC). Compared with existing adversarial-based representation learning methods for UDA, the novelty of RLGLC lies in two aspects. First, we propose Asymmetrically-Relaxed Wasserstein of Wasserstein Distance (AR-WWD), corresponding to the Global Consistency Module (GCM), as a new approach for measuring distributional discrepancies. Second, we present a technique that effectively implements the fourth term in Equation (\ref{Eq:bbbd}), represented by the Local Consistency Module (LCM).

Specifically, the GCM is motivated by two key observations: 1) Existing Wasserstein distances are zero only when two distributions are perfectly aligned or identical. However, during training, sampling biases can introduce class imbalance across domains. For example, in a binary classification task, the source domain might feature a 5:5 ratio of positive to negative samples, while the target domain might have a 3:7 ratio. Rigidly aligning these distributions could shift some negative samples in the target domain into the positive sample cluster, causing classification errors that stem from forced distribution alignment rather than from any inherent limitation of the classifier; 2) From Equation (\ref{Eq:Wa}), the cost function \(c(\cdot)\) is typically the \(L_2\) norm, which measures distances between feature vectors. However, the \(L_2\) norm is insensitive to shifts in vector dimensions. For instance, given sample features \(X_1=(1,0,0,\ldots,0)\), \(X_2=(0,1,0,\ldots,0)\), and \(X_3=(0,0,1,\ldots,0)\), the \(L_2\) distance between \(X_1\) and \(X_2\) matches that between \(X_1\) and \(X_3\). Yet each feature dimension typically encodes distinct semantic information. From a semantic salience viewpoint, when certain dimensions carry greater importance, their differences should be reflected accordingly. The \(L_2\) norm, however, only captures numerical differences and disregards the semantic relevance of each dimension in the overall similarity measure. To mitigate these challenges, GCM first introduces a constraint making the distance between distributions zero if the target distribution is equal to or contained within the source distribution. This condition helps avert classification errors caused by strict distribution alignment when class ratios differ between domains. Furthermore, GCM implements \(c(\cdot)\) as a Wasserstein distance, treating the multiple dimensions of a sample’s feature representation as a distribution. This approach enables learning a joint distribution that can flexibly adjust the relative weights of different dimensions. Meanwhile, the LCM takes inspiration from Noise-Contrastive Estimation \cite{oord2018representation, chen2020simple, chen2020intriguing}. We theoretically demonstrate the effectiveness and convergence of this proposed approach in \textbf{Proposition \ref{gkhomcaomsg}}. 

Finally, by simultaneously enhancing transferability and target discriminability at the level of "learning a good representation", RLGLC offers a theoretically grounded solution supported by mutual information analysis. Theoretically, we connect our approach to the reduction of Bayes error rate in the target domain, ensuring that the learned target representations remain both domain-invariant and label-relevant. Empirical evaluations on multiple benchmark datasets confirm that RLGLC outperforms state-of-the-art methods, demonstrating the power of focusing on the capability of the feature extractor to discover both globally transferable and locally discriminative structures in the target domain. The major contributions of this paper are threefold:
\begin{itemize}
\item This paper reconsiders adversarial-based UDA through an information-theoretic lens, introducing “good representation learning” to unify domain transferability and label discriminability. By employing conditional mutual information, we show that existing UDA methods, e.g., focused solely on aligning feature distributions of different domains and minimizing source-domain errors, fail to guarantee discriminative features in the target domain. This theoretical gap clarifies why augmenting UDA objectives with a target-focused loss substantially improves practical performance, thus bridging the discrepancy between empirical observations and formal theory.
\item Building on this insight, this paper proposes a new adversarial-based representation learning framework that explicitly addresses cross-domain alignment and target-label relevance. The theoretical results (\textbf{Theorem \ref{qw}} and \textbf{Theorem \ref{gvbnjiout}}) demonstrate that this framework not only reduces the domain gap but also ensures discriminative features for both source and target domains, thereby explaining and overcoming the theoretical–practical gap in existing adversarial UDA methods.
\item Under this unified framework, this paper propose the RLGLC. It combines an AR-WWD for flexible global alignment (accounting for class imbalance and semantic dimension weighting) with a theoretically-guaranteed Local Consistency Module to preserve fine-grained target discriminability. Experimental results confirm the superior performance of RLGLC, validating its ability to jointly promote transferability and discriminability in UDA settings.
\end{itemize}

\section{Related works}
Unsupervised domain adaptation aims to transfer knowledge learned from a labeled source domain to a related unlabeled target domain \cite{kumar2020understanding, dhouib2020margin, balaji2020robust, cui2020gradually, combes2020domain, cui2020heuristic, hu2020unsupervised, kang2020pixel, tang2020unsupervised}. Remarkable advances have been achieved in UDA, especially these representation learning-based methods. The main idea behind these methods is to align the distributions of the source domain and target domain. Therefore, many works are proposed to design effective metrics to measure the differences between distributions. Maximum mean discrepancy \cite{gretton2012kernel, tzeng2014deep} is a nonparametric metric that measures the divergence of two distributions in the reproducing kernel Hilbert space. Deep correlation alignment \cite{sun2016return} aligns two distributions by minimizing the difference in the second-order statistics of the two distributions. Domain Adversarial Neural Network \cite{ganin2016domain} and S-disc \cite{kuroki2019unsupervised} align distributions by minimizing KL-divergence. Wasserstein distance guided representation learning \cite{shen2018wasserstein} introduces the Wasserstein distance for domain adaptation to make the training process stable. Sliced Wasserstein discrepancy \cite{lee2019sliced} proposes to utilize the sliced Wasserstein distance to accelerate the training process. Margin disparity discrepancy \cite{zhang2019bridging} aligns distributions based on the scoring function and margin loss. Domain-Symmetric Networks \cite{zhang2020unsupervised} proposes a multi-class scoring disagreement divergence. An asymmetrically-relaxed distribution alignment is proposed in \cite{wu2019domain}. Reliable weighted optimal transport \cite{xu2020reliable} proposes a novel shrinking subspace reliability and weighted optimal transport strategy for UDA. In \cite{li2020enhanced}, an enhanced transport distance for UDA is proposed. Probability-Polarized OT \cite{wang2024probability} proposes to guide the optimization direction of OT plan by characterizing the structure of OT plan explicitly. Prompt-based Distribution Alignment \cite{bai2024prompt} introduces distribution
alignment into prompt tuning and propose to use a two-branch
training paradigm to align the distribution of two domains. TransVQA \cite{sun2024transvqa} proposes a two-step alignment method to align the extracted cross-domain features and solve the domain shift problem. Different from these methods that mainly focus on learning transferable representations, this paper analyzes the UDA problem from an information theory perspective and aims to learn feature representations that are with both discriminability and transferability.

Other approaches also focus on designing an addition term to improve the discriminability of the feature representation of the target domain samples. BSP \cite{xyc19} penalizes the largest singular values so
that other eigenvectors can be relatively strengthened to boost
the feature discriminability of the samples of the target domain. Gradient Harmonization \cite{huang2024gradient} propose to alter the gradient angle between different tasks from an obtuse angle to an acute angle, thus resolving the conflict and increasing the feature discriminability of the target domain samples. AT-MCAN and RLPGA \cite{qiang2021robust, qiang2021auxiliary} propose to constrain the information entropy of the prediction results of individual samples as well as the information entropy of the prediction results of the entire dataset to improve the discriminative nature of the feature representation of the data in a target domain. Meanwhile, generating pseudo labels for the samples of target domain is an effective way to enhancing the feature discriminability \cite{wang2024probability, bai2024prompt, sun2024transvqa, gu2022unsupervised, gao2024learning, wei2024class, zhu2024versatile, jiahua2022and, kang2019contrastive}. In general, the discriminative nature of the target domain feature representation can be further improved by a cross-entropy loss term when the target domain samples are given a false label through a mechanism, e.g., PseudoCal \cite{hupseudo} approach UDA calibration as a target-domain-specific unsupervised problem and propose to use the inference-stage mixup to synthesize a labeled pseudo-target set for the real unlabeled target data. Cycle Self-Training \cite{liu2021cycle} proposes a principled self-training algorithm that explicitly enforces pseudo-labels to generalize across domains. Different from these methods, this paper is not concerned with how to design a good target-domain oriented loss function to improve the discriminative properties of the sample feature representations in the target domain, but rather with a theoretical analysis to answer the question of why the addition of a loss term that can improve the discriminative properties of the sample feature representations in the target domain is necessary.

On par with the domain adaptation algorithms, there are rich advances in the domain adaptation theoretical findings. In \cite{mansour2009domain, ben2010theory, zhu2024versatile}, a rigorous classification error bound based on the source domain classification error and the divergence between the source and target domains is proposed for UDA. Then, a series of theories have been proposed to extend this theory to different cases, e.g., multi-class classification setting, regression setting, conditional covariate shifts setting, label shifts setting, etc. \cite{mohri2012new, germain2013pac, cortes2015adaptation}. Then, based on reproducing kernel Hilbert space, \cite{redko2017theoretical} proposes to bound the target error by the Wasserstein distance. \cite{shen2018wasserstein} proposes to use the Kantorovich-Rubinstein dual formulation of the Wasserstein distance to obtain a generalization bound for UDA. In \cite{wu2019domain}, a good target domain performance is demonstrated theoretically under the setting of relaxed alignment for the label shift problem. In \cite{zhang2019bridging}, based on Rademacher complexity, a margin-aware generalization bound is provided to bridge the gaps between the theories and algorithms for multi-class classification problem in UDA. \cite{zhang2020unsupervised, gu2022unsupervised} extends the generalization bound provided in \cite{zhang2019bridging} and can better explain the effectiveness of UDA related to multiple classes. In \cite{zhao2019learning}, both the upper and lower bounds of the target classification error are provided. In \cite{dhouib2020margin}, based on the large margin separation, a new theoretical analysis is provided to uniform the margin, adversarial learning, and domain adaptation. In \cite{kumar2020understanding, gu2022unsupervised}, the self-train is proved to be effective for larger distribution shifts. Different from these theoretical findings, this paper is motivated by the information theory and the Bayes error rate.

\section{Preliminaries}
\label{sec:1}
In this section, we first introduce the problem definition of the unsupervised domain adaptation (UDA) and provide relevant symbolic notations. Next, we present a UDA learning framework based on representation learning and introduce the definition of the Wasserstein distance associated with it.

\subsection{Problem definition and notations} \label{sub:121}

This paper focuses on the classification task in Unsupervised Domain Adaptation (UDA) under the {\em covariate shift} assumption. Let $\mathcal{X}$ denote the input sample space and $\mathcal{Y}$ the label space. Let $\eta: \mathcal{X} \to \mathcal{Y}$ be the domain-invariant ground truth labeling function, where $X \in \mathcal{X}$ is a sample and $Y \in \mathcal{Y}$ is its corresponding label. Let $P_s(X)$ and $P_t(X)$ represent the input distributions of the source and target domains, respectively. We assume that $X_s \sim P_s(X)$ and $X_t \sim P_t(X)$, where $X_s$ and $X_t$ are random variables from the source and target domains. In UDA, the source domain contains labeled data, while the target domain has only unlabeled data, but both domains share the same label space. Under the covariate shift assumption, two conditions hold: (1) the marginal data distributions differ between the source and target domains, i.e., $P_s(X) \ne P_t(X)$; (2) the conditional distributions of labels given inputs are identical across domains, i.e., $P(Y_s \mid X_s) = P(Y_t \mid X_t)$. This assumption allows leveraging labeled source data to improve performance on unlabeled target data by accounting for the distribution shift in the input space.

Let $\mathcal {Z}$ be a latent space and $\Phi: X \to Z$ be a class of feature extractors, where $Z \in \mathcal {Z}$. For a domain $u \in \left\{ {s,t} \right\}$, $P_u^\varphi \left( Z_u \right) = P_u\left( {{\varphi ^{ - 1}}\left( Z_u \right)} \right)$ represents the $\varphi$-induced probability distribution over $\mathcal {Z}$, where $\varphi  \in \Phi $. For a random variable $Z_u$ of space $\mathcal {Z}$, let ${P }\left( { X_u \left| Z_u \right.} \right)$ be the induced conditional distribution over $ \mathcal {X}$ that satisfies $\int {P_u^\varphi \left( Z_u \right){P }\left( {X_u\left| Z_u \right.} \right)} d\left( Z_u \right) = {P}\left( X_u \right)$. Denote $\Psi :Z \to Y$ as a class of prediction functions. Then, the model learned by UDA can be represented as $\psi \left( {\varphi \left( \cdot \right)} \right)$, where $\psi  \in \Psi $. Given the labeled data of the source domain and the unlabeled data of the target domain, UDA assumes that the label space for source domain and target domain is the same. Then, the goal of UDA is to learn a model that can approximate $\eta$ better, thus can minimize the following expected risk in the target domain:
\begin{equation}
\label{Eq:risk}
{R_{{P_t}}}\left( X_t \right) = \int {{P_t}\left( X_t \right){\mathcal{L}}\left( {\eta \left( X_t \right),\psi \left( {\varphi \left( X_t \right)} \right)} \right)} d\left( X_t \right),
\end{equation}
where ${R_{{P_t}}}\left( X_t \right)$ denotes the expected target risk in the target domain and ${\mathcal{L}}\left( \cdot \right)$ is the loss function, e.g., 0-1 loss and quadratic loss. With a bit of symbol abuse, we denote the support set of $P_u \left( X \right)$ as $X^u$ and the support set of $P_u^\varphi \left( Z \right)$ as $Z^u$. We denote a sample in $X^u$ as $X_u$ and a sample in $Z^u$ as $Z_u$. Also, we denote the shared label distribution as $P(Y)$ and the corresponding random variable as $Y$.

\subsection{Representation learning framework for UDA}
The main idea of the representation learning based UDA is to align the feature distributions of different domains in the feature space, so that a classifier trained on the source domain can be applicable to the target domain. Generally, the learning framework of representation learning based UDA methods can be divided into three parts, including a metric to measure the difference between two distributions in the latent space $\mathcal {Z}$, a classification loss to measure the empirical risk in the source domain, and a regularization term. The whole objective is formulated as follows:
\begin{equation}
\label{asd}
\begin{array}{l}
\mathop {\min }\limits_{\varphi ,\psi } D( {P_s^\varphi ( Z_s ),P_t^\varphi ( Z_t )} )\\
\quad\quad\quad\quad + {{\mathcal{L}}_{cl}}( {\psi ( {\varphi ( {{X_s}} )} ),\eta ( {{X_s}} )} ) + \Delta (\varphi ,\psi ),
\end{array}
\end{equation}
where $D( {P_s^\varphi ( {Z_s} ) ,P_t^\varphi ( {Z_t} ) } )$ is a discrepancy metric for the source and target domain distributions in the latent space, ${\mathcal{L}}_{cl}\left( \cdot \right)$ is total classification loss over source domain, and $\Delta$ is the regularization term for the feature extractor $\varphi$ or the classifier $\psi$ or both. Currently, the commonly used methods for measuring the discrepancy between probability distributions include the KL-divergence and the Wasserstein distance. According to \cite{arjovsky2017wasserstein}, the Wasserstein distance exhibits better continuity properties than the KL divergence. Specifically, when the supports of two distributions are disjoint, the Wasserstein distance continues to vary with changes in the distributions, whereas the KL divergence becomes a constant value (often infinite) and fails to reflect the actual closeness between the distributions. Therefore, this paper considers adopting the Wasserstein distance as the discrepancy metric, and the $p$-th Wasserstein distance is defined as:
\begin{equation}
\label{Eq:Wa}
\begin{array}{l}
{W_p}( {P_s^\varphi ( Z_s ),P_t^\varphi ( Z_t )} )\\
\quad\quad\quad= {( {\mathop {\inf }\limits_{\mu ( {{Z_s},{Z_t}} ) \in \Pi ( {{Z_s},{Z_t}} )} \int {c{{( {{Z_s},{Z_t}} )}^p}d\mu } } )^{\frac{1}{p}}},
\end{array}
\end{equation}
where $Z_s$ and $Z_t$ are two random variables in the latent feature space $\mathcal{Z}$, with ${Z_s} \sim P_s^\varphi(Z_s)$ and ${Z_t} \sim P_t^\varphi(Z_t)$. The function $c\left( Z_s, Z_t \right)$ denotes the ground metric measuring the distance between samples from the source and target domains. The set $\Pi\left( Z_s, Z_t \right)$ consists of all joint distributions $\mu\left( Z_s, Z_t \right)$ that satisfy the marginal constraints: ${P_s^\varphi}(Z_s) = \int_{{Z_t}} {\mu\left( {{Z_s},{Z_t}} \right)} d {{Z_t}} ,{P_t^\varphi}(Z_t) = \int_{{Z_s}} {\mu\left( {{Z_s},{Z_t}} \right)} d {{Z_s}}$.

\begin{figure*}[ht]
	\centering
	\includegraphics[width=0.9\textwidth]{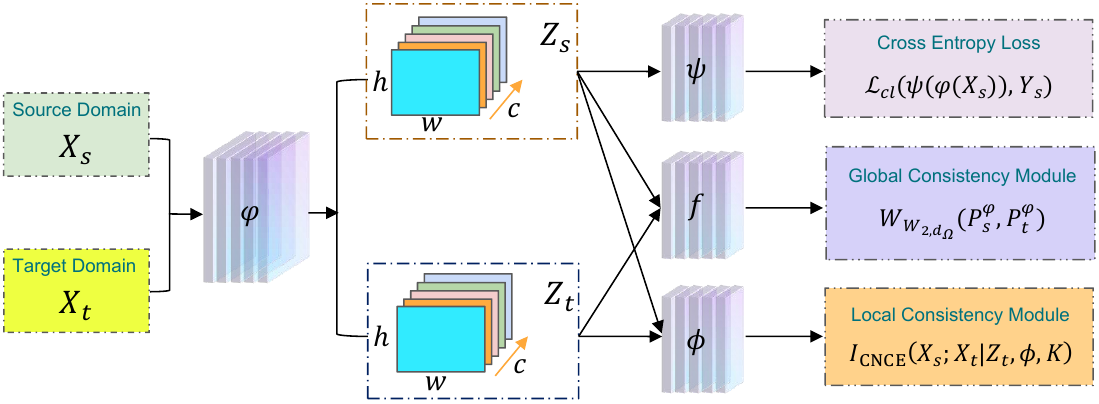}
	\caption{The outline of the proposed RLGLC. RLGLC is an adversarial-based representation learning method for UDA. First, we feed the source-domain data \(X_s\) and target-domain data \(X_t\) from the original input space into the feature extractor \(\varphi\), resulting in feature representations \(Z_s\) and \(Z_t\). Next, we pass \(Z_s\) into the classifier \(\psi\) to compute the cross-entropy loss on the source-domain samples. We then input both \(Z_s\) and \(Z_t\) into the network \(f\) to calculate the proposed AR-WWD metric, \({W_{1,{W_{2,{d_\Omega }}}}}(P_s^\varphi, P_t^\varphi)\). Finally, we feed \(Z_s\) and \(Z_t\) into the network \(\phi\) to compute the proposed conditional mutual information \(I_{\text{CNCE}}(X_s; X_t \mid Z_t, \phi, K)\). For optimization, we begin by fixing \(\varphi\) and \(\psi\). We maximize \({W_{1,{W_{2,{d_\Omega }}}}}(P_s^\varphi, P_t^\varphi)\) and \(I_{\text{CNCE}}(X_s; X_t \mid Z_t, \phi, K)\) to update \(f\) and \(\phi\), respectively. Next, we fix \(f\) and \(\phi\) and minimize Equation (\ref{Eq:bbb}) to update \(\varphi\) and \(\psi\). This two-step procedure ensures proper adversarial interplay between the feature extractor, the classifier, and the networks \(f\) and \(\phi\).}
	\label{1234}
\end{figure*}

\section{Methodology}

In this section, we first provide an analysis of the existing representation learning-based UDA framework presented by \Cref{asd} from the perspective of information theory. We then derive the proposed method, called domain-invariant representation learning with global and local consistency (RLGLC) by instantiating the proposed framework with a new discrepancy and a new self-supervised loss to improve the target-aware mutual information.

\subsection{Information-theoretical based analysis}
\label{098}

From the perspective of the information bottleneck principle \cite{tishby2000information}, a good feature representation should not only contain as little task-independent information as possible, but also contain as much task-related information as possible. For UDA \cite{xyc19}, a good feature representation should not only be discriminative for the downstream tasks, but also be transferable between different domains. To bridge the conceptual gap between these two perspectives, we propose to evaluate the representations learned by UDA from two aspects, including discriminability and transferability, and quantify them with information theory. We start from the following assumption and definitions, where $I(X; Y)$ is the mutual information between two random variables $X$ and $Y$ and $I(X; Y\left| Z \right.)$ is the conditional mutual information between $X$ and $Y$ given random variable $Z$.

\begin{assumption}
	\label{a1}
	Assume that each domain contains all the information relevant to the task, e.g., $I( {X_{s \cup t}}; Y ) = I( {{X_s}; Y} ) = I( {{X_t}; Y} )=H(Y)$, where ${X_{s \cup t}}$ can be regarded as a new variable composed of $X_s$ and $X_t$ which satisfies $H( {X_{s \cup t}} ) = H( {{X_s}} ) + H( {{X_t}} ) - I( {{X_s};{X_t}} )$.
\end{assumption}

\begin{definition}
	\label{d1}
	Discriminability: The amount of discriminative information contained in the representation $Z_u$ of $X_u$ can be defined as $I\left(X_s;X_t\left| Z_u \right. \right)$, $u \in \left\{ {s,t} \right\}$. The smaller $I\left(X_s;X_t\left| Z_u \right. \right)$ is, the more discriminative information $Z_u$ contains.
\end{definition}

\begin{definition}
	\label{d3}
	Transferability: The amount of transferable information contained in the representation $Z_u$ of $X_u$ and the representation $Z_t$ of $X_t$ can be defined as $I\left(X_t; Z_t\left| X_s \right. \right)$ and $I\left(X_s; Z_s\left| X_t \right. \right)$, respectively. The smaller $I\left(X_t; Z_t\left| X_s \right. \right)$ and $I\left(X_s; Z_s\left| X_t \right. \right)$ is, the more transferable information $Z_s$ and $Z_t$ contains, respectively.
\end{definition}

From the above assumptions and definitions, it is evident that Discriminability and Transferability are measured from an information-theoretic perspective, emphasizing properties of distributions rather than individual samples. According to the defined metrics, without considering the feature distribution, evaluating whether the feature representation of a single sample independently exhibits Discriminability or Transferability is both theoretically flawed and practically meaningless. Therefore, our definitions underscore that these properties are assessed based on the information embedded in the sample feature distribution within the feature space, reflecting the capacity of a feature extractor to effectively capture discriminative and transferable information.

Under the covariate shift assumption, ${P_s}\left( X \right) \ne {P_t}\left( X \right)$ and $P_s\left( {Y\left| {{X}} \right.} \right) = P_t\left( {Y\left| {{X}} \right.} \right)$ hold. Then, we can safely obtain that the two domains contain the same predictive information, hence we can obtain that the shared information $I\left( {{X_s},{X_t}} \right)$ between two domains contain all the task-related information \cite{sridharan2008information, xu2013survey, tsai2020self}. Let us take a cancer diagnosis as an example. Cancer is generally divided into three stages, including early, intermediate, and advanced stages. Suppose the downstream task is to diagnose the stage of the cancer patient's disease. For every cancer patients, we can get an X-ray image and an MRI image. Then, all the MRI images of cancer patients are regarded as the source domain, and their distribution is denoted as $P_s(X)$, and all the X-ray images of cancer patients are regarded as the target domain, and their distribution is denoted as $P_t(X)$. Then, $I\left( {{X_s},{X_t}} \right)$ can be interpreted as the shared information related to cancer diagnosis in two domains. In this scenario, \textbf{Assumption \ref{a1}} states that the mutual information between data and labels in each domain less than or equal to the mutual information among data of both domains. Also, from a generative point of view, the following Markov chain holds: ${Z_t} \leftarrow {X_t} \leftrightarrow Y \leftrightarrow {X_s} \to {Z_s}$ (see Appendix \textbf{Lemma \ref{l1}}). Applying Data Processing Inequality \cite{cover2012elements} to the above Markov chain, we have: $I\left( {{X_s};{X_t}} \right) \ge I\left( {{X_s};Y} \right) \ge I\left( {{Z_s};Y} \right)$ and $I\left( {{X_s};{X_t}} \right) \ge I\left( {{X_t};Y} \right) \ge I\left( {{Z_t};Y} \right)$.

From \textbf{Definition \ref{d1}}, a representation $Z_u$ with full discriminability (i.e., $I\left(X_s;X_t\left| Z_u \right. \right) = 0$) satisfies $I\left(X_u; Y\right) =  I\left(Z_u; Y\right)$ (see Appendix \textbf{Proposition \ref{p1}}), which means that a discriminative representation $Z_u$ can predict $Y$ as accurately as the original data $X_u$. Based on the chain rule of mutual information, we can obtain that $I\left( {{X_s};{Z_s}} \right) = I\left( {{X_s};{Z_s}\left| {{X_t}} \right.} \right) + I\left( {X_s; {X_t};{Z_s}} \right)$ and $I\left( {{X_t};{Z_t}} \right) = I\left( {{X_t};{Z_t}\left| {{X_s}} \right.} \right) + I\left( {X_s;{X_t};{Z_t}} \right)$, where $I\left( {{X_s};{Z_s}\left| {{X_t}} \right.} \right)$ represents the domain-dependent information in $Z_s$ that is not shared with the domain $X_t$, and $I\left( {{X_t};{Z_t}\left| {{X_s}} \right.} \right)$ represents the domain-dependent information in $Z_t$ that is not shared with domain $X_s$. Therefore, from \textbf{Definition \ref{d3}}, if a representation $Z_u$ is fully transferable (i.e., $I\left(X_t; Z_t\left| X_s \right. \right) = I\left(X_s; Z_s\left| X_t \right. \right) = 0$), we can obtain that $I\left( {{X_s};{Z_s}} \right) = I\left( {{X_s};{X_t};{Z_t}} \right)$ and $I\left( {{X_t};{Z_t}} \right) = I\left( {{X_s};{X_t};{Z_s}} \right)$. Then based on \textbf{Defination \ref{d1}}, maximizing $I\left( {{X_s};{X_t};{Z_t}} \right)$ and $I\left( {{X_s};{X_t};{Z_s}} \right)$ can make the learned representation contain as much predictive information as possible. At the extreme, $X_s$ and $X_t$ share only label information, in which case our method is equivalent to the supervised information bottleneck method without needing to access the labels (see Appendix \textbf{Proposition \ref{p3}}).

We now provide intuitive explanations about the discriminability and transferability. From the decoupling point of view, we suppose that $X_u$ can be broken down into two parts, including the label-related part $X_Y$ (the information related to cancer diagnosis in two domains) and the label-unrelated part $X^{Uu}_Y$ (background of X-ray or MRI images), where $u \in \left\{ {s,t} \right\}$. Also, we assume that $X_Y$ and $X^{Uu}_Y$ are independent of each other. Because UDA assumes that two domains share the same label space, so, we assume the label-related part $X_Y$ between two domains is the same. From the perspective of the information theory, we can safely obtain that $H\left( {{X_u}} \right) = H\left( {{X_Y}} \right) + H\left( {X_Y^{Uu}} \right)$. Also, we assume that $X^{Uu}_Y$ is only related to domain $u$. Then, we have $I\left( {{X_s};{X_t}} \right) = H\left( {{X_Y}} \right)$. Because that ${Z_u} = \varphi \left( {{X_u}} \right)$, we have $H\left( {{Z_u}} \right) \le H\left( {{X_u}} \right)$. Note that a discriminant representation should retain as much of the label-related part $X_Y$ as possible. Also, from the decoupling point of view, $H\left( {{Z_u}} \right)$ contains a part of $H\left( {{X_Y}} \right)$ and a part of $H\left( {X_Y^{Uu}} \right)$. We denote the part related to $H\left( {{X_Y}} \right)$ in $H\left( {{Z_u}} \right)$ as $H\left( {X_u^Y} \right)$. Then, we have $H\left( {X_u^Y} \right) = I\left( {{X_s};{X_t}} \right) - I\left( {{X_s};{X_t}\left| {{Z_u}} \right.} \right) = H\left( {{X_Y}} \right) - I\left( {{X_s};{X_t}\left| {{Z_u}} \right.} \right)$. Therefore, \textbf{Definition \ref{d1}} is reasonable. Note that the label-unrelated part is domain-specific, 
and thus can not be transferred from one domain to another. So, we assume that a transferable representation should contain as little of the label-unrelated part $X^{Uu}_Y$ as possible. We denote the part related to $H\left( {X_Y^{Uu}} \right)$ in $H\left( {{Z_u}} \right)$ as $H\left( {{}^uX_Y^{Uu}} \right)$. Then, we have $H\left( {{}^sX_Y^{Us}} \right) = I\left( {{X_s};{Z_s}\left| {{X_t}} \right.} \right)$ and $H\left( {{}^tX_Y^{Ut}} \right) = I\left( {{X_t};{Z_t}\left| {{X_s}} \right.} \right)$. Therefore, \textbf{Definition \ref{d3}} is reasonable.

In Equation (\ref{asd}), the first term aims to align the distributions of the source and target domains in the latent space $\mathcal{Z}$, the second term aims to associate the learned feature representation with the corresponding label in the source domain. Then, we have:

\begin{theorem}
	\label{t1}
	Suppose the representations $Z_s$ and $Z_t$ for the source domain and the target domain are obtained by minimizing the objective function (\ref{asd}). Then, the discriminability and transferability of $Z_s$ are increased, while only the transferability of $Z_t$ is improved.
\end{theorem}

The proof is presented in Appendix \textbf{Theorem \ref{asa}}. From \textbf{Theorem \ref{t1}}, we can find that only the transferability of the representation for samples in the target domain is considered, the discriminability of the representation for samples in the target domain is not fully considered. This could cause performance degradation of UDA models. To enhance understanding, we provide an intuitive explanation of this theorem, highlighting its fundamental concepts and underlying logic. First, $Z_s$ and $Z_t$, derived from $X_s$ and $X_t$ respectively, contain only a subset of the information present in $X_s$ and $X_t$. Then, according to \textbf{Definition \ref{d3}}, $Z_s$ and $Z_t$ are the sole variables in the conditional mutual information terms $I(X_s; Z_s \mid X_t)$ and $I(X_t; Z_t \mid X_s)$. These mutual information terms achieve their minimum values under specific conditions for $Z_s$ and $Z_t$, namely when: 1) $Z_s$ and $Z_t$ contain all task-related information common to both domains; or 2) $Z_s$ and $Z_t$ contain only a part of the task-related information. Based on the proof process in Appendix \textbf{Theorem \ref{asa}}, we deduce that the task-related information in $Z_s$ is also regulated by the classification loss ${{\mathcal{L}}_{cl}}( \psi ( \varphi ( {X_s} ) ), \eta ( {X_s} ) )$. Therefore, minimizing both $I(X_s; Z_s \mid X_t)$ and ${{\mathcal{L}}_{cl}}( \psi ( \varphi ( {X_s} ) ), \eta ( {X_s} ) )$ ensures that $Z_s$ encompasses all task-related information, providing both discriminability and transferability. In contrast, $Z_t$ lacks constraints on task-related information since there is no classification loss applied to it. Consequently, minimizing $I(X_t; Z_t \mid X_s)$ may result in $Z_t$ containing only partial task-related information. Therefore, the learned $Z_t$ may not be optimal, possessing transferability but lacking discriminability. To this end, we propose a novel representation learning framework to address this problem.

\begin{figure}
	\centering
	\includegraphics[width=0.45\textwidth]{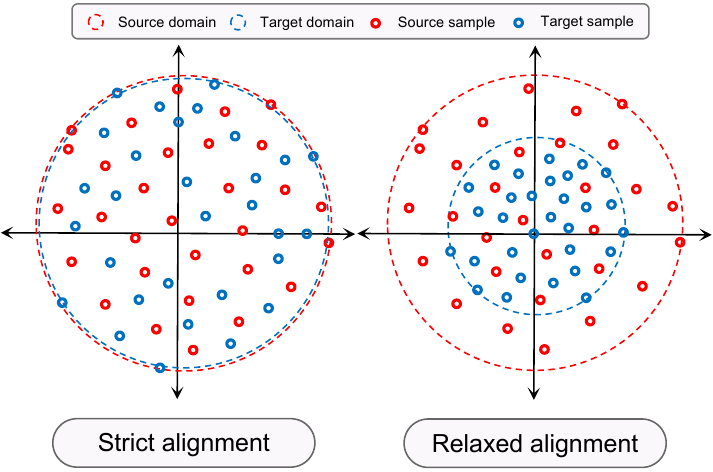}
	\caption{The red solid circles represent the source domain samples, the blue solid circles represent the target domain samples, the red dashed area represents the source domain data distribution, and the blue dashed area represents the target domain distribution area. The left subfigure shows that the two domain distributions are strictly aligned, and the right subfigure shows that the target domain distribution is included in the source domain distribution.}
	\label{123456}
\end{figure}

\subsection{The proposed learning framework}
\label{aass}

From the information-theoretical based analysis shown above, we can obtain that an ideal UDA representation learning framework should ensure that the features of both the source and target domains exhibit discriminability and transferability. Thus, the proposed framework can be formulated as follows:
\begin{equation}
\label{Eq:bbbd}
    \begin{array}{*{20}{l}}
{\mathop {\min }\limits_{\varphi ,\psi } I\left( {{X_s};{Z_s}\left| {{X_t}} \right.} \right) + I\left( {{X_t};{Z_t}\left| {{X_s}} \right.} \right) + I\left( {{X_s};{X_t}\left| {{Z_s}} \right.} \right)}\\
{\quad \quad \quad \quad \quad \quad  + I\left( {{X_s};{X_t}\left| {{Z_t}} \right.} \right) + \Delta (\varphi ,\psi )}
\end{array}
\end{equation}
where ${\varphi ,\psi }$ are network parameters and $\Delta (\varphi ,\psi )$ is the regular term about the network parameter. Based on \textbf{Definition \ref{d1}} and \textbf{Definition \ref{d3}}, we can obtain that minimizing the first two terms of Equation (\ref{Eq:bbbd}) allows for maximum transferability of the learned source and target domain features, and minimizing the third and fourth terms of Equation (\ref{Eq:bbbd}) allows for maximum discriminability of the learned source and target domain features. Therefore, we have the following conclusion:

\begin{theorem}
	\label{qw}
Assuming \textbf{Assumption \ref{a1}} holds, if a UDA representation learning framework can simultaneously minimize $I\left( {{X_s};{Z_s}\left| {{X_t}} \right.} \right)$, $I\left( {{X_t};{Z_t}\left| {{X_s}} \right.} \right)$, $I\left( {{X_s};{X_t}\left| {{Z_s}} \right.} \right)$, and  $I\left( {{X_s};{X_t}\left| {{Z_t}} \right.} \right)$, the resulting source and target domain features will exhibit both discriminability and transferability.
\end{theorem}					

The proof of \textbf{Theorem \ref{qw}} can be obtained directly based on \textbf{Definition \ref{d1}} and \textbf{Definition \ref{d3}}. From the proof presented in Appendix \textbf{Theorem \ref{asa}}, we obtain that: 1) minimizing $I\left( {{X_s};{Z_s}\left| {{X_t}} \right.} \right)$ and $I\left( {{X_t};{Z_t}\left| {{X_s}} \right.} \right)$ can be approximated by minimizing the divergence between the source domain distribution and target domain distribution in the feature space; 2) we can minimizing $I\left( {{X_s};{X_t}\left| {{Z_s}} \right.} \right)$ by minimizing ${{\mathcal{L}}_{cl}}( {\psi ( {\varphi ( {{X_s}} )} ),\eta ( {{X_s}} )} )$. Thus, the proposed representation learning framework Equation (\ref{Eq:bbbd}) can be rewritten as: 
\begin{equation}
\label{Eq:udbsasdasf}
\begin{array}{*{20}{l}}
{\mathop {\min }\limits_{\varphi ,\psi } D(P_s^\varphi ({Z_s}),P_t^\varphi ({Z_t})) + {{\cal L}_{cl}}(\psi (\varphi ({X_s})),\eta ({X_s}))}\\
{\quad \quad \quad \quad \quad \quad  + I\left( {{X_s};{X_t}\left| {{Z_t}} \right.} \right) + \Delta (\varphi ,\psi )}
\end{array}
\end{equation}

From Equation (\ref{Eq:udbsasdasf}), to better ensure that the extracted source and target domain features possess both discriminability and transferability, it is crucial to design a measure that accurately quantifies the distribution differences between domains and to effectively achieve $I\left( {{X_s};{X_t}\left| {{Z_t}} \right.} \right)$. In the following, based on the proposed framework, we further propose a novel UDA method, namely domain-invariant representation learning with global and local consistency (RLGLC), the global consistency is intended to better implement $D(P_s^\varphi ({Z_s}),P_t^\varphi ({Z_t}))$ and the local consistency is intended to better implement $I\left( {{X_s};{X_t}\left| {{Z_t}} \right.} \right)$.

\subsection{The proposed UDA method: RLGLC}

An overview of the proposed RLGLC is illustrated in Figure \ref{1234}. RLGLC mainly consists of two new modules, including the global consistency module and the local consistency module. The global consistency module is employed to align the distributions of the two domains in the latent space, which is related to the item $D( {P_s^\varphi ,P_t^\varphi } )$. The local consistency module aims to present the specific implementation of 
conditional mutual information $I\left( {{X_s};{X_t}\left| {{Z_t}} \right.} \right)$.

\subsubsection{Global consistency module}

The global consistency module is designed to more accurately measure the discrepancy between $P_s^\varphi\left( Z\right)$ and $P_t^\varphi\left( Z\right)$. Generally, the Wasserstein distance equals zero only when the distributions $P_s^\varphi\left( Z\right)$ and $P_t^\varphi\left( Z\right)$ are identical. However, during training, the inherent randomness of mini-batch sampling can lead to class imbalance between the source and target domains. For example, suppose the training samples in both domains contain two classes (positive and negative). In a given mini-batch, the source domain might have a positive-to-negative sample ratio of $5:5$, while the target domain has a ratio of $3:7$. Strictly enforcing equality between these two distributions could force some negative samples in the target domain to be misclassified as positive, negatively impacting model performance. Moreover, according to Equation (\ref{Eq:Wa}), $c(\cdot)$ is typically implemented as the $L_2$ norm, which calculates the distance between the feature vectors of two input samples. A characteristic of the $L_2$ norm is its insensitivity to shifts in feature vector dimensions. For instance, consider three sample features $X_1=(1,0,0,\dots,0)$, $X_2=(0,1,0,\dots,0)$, and $X_3=(0,0,1,\dots,0)$. The $L_2$ distances between $X_1$ and $X_2$, and between $X_1$ and $X_3$ are identical. In feature vectors, different dimensions often correspond to distinct semantics. From the perspective of semantic significance, when the semantics associated with different dimensions have varying levels of importance, their differences should be weighted accordingly in similarity measures. However, the $L_2$ norm only considers the numerical differences across dimensions without accounting for the importance of semantics of each dimension in the similarity evaluation.

To address the aforementioned challenges, we first propose relaxing the condition that "the Wasserstein distance equals zero only when $P_s^\varphi(Z)$ and $P_t^\varphi(Z)$ are identical" to "the Wasserstein distance equals zero when $P_t^\varphi(Z)$ is equal to or contained within $P_s^\varphi(Z)$." Here, "containment" means that for any value of $Z$, $P_s^\varphi(Z) \ge P_t^\varphi(Z)$. The advantage of this relaxation is that, even if class imbalance occurs in mini-batches due to random sampling, constraining the target domain distribution to be contained within the source domain distribution allows us to leverage a classifier trained on source data to correctly classify all target domain samples. Figure \ref{123456} provides an intuitive illustration of this "containment" concept. Secondly, to incorporate semantic importance into the distance measurement between feature vectors, we propose implementing the cost function $c(\cdot)$ in the Wasserstein distance as another Wasserstein distance. Specifically, we treat a vector as a distribution, where different elements of the vector are considered as samples drawn from this distribution. This approach enables us to assign a probability density value to each element using a probability density function, reflecting the importance of each vector element. However, since the distribution formed by the feature vector elements is discrete during training, the probability densities of the elements are set equally. This raises the question: how can we reflect semantic importance through the probability density function in this case? Our answer is based on existing works \cite{bengio2013representation, zhou2016learning, lin2017focal, mikolov2013efficient}, which suggest that important semantics in a feature vector are often associated with multiple dimensions, that is, multiple dimensions collectively represent one semantic concept. Therefore, the sum of the probability densities of these dimensions can reflect the importance of the semantic. The varying numbers of vector dimensions corresponding to different semantics thus reflect their different levels of importance.

To this end, we present the proposed asymmetrically-relaxed Wasserstein of Wasserstein distance (AR-WWD). In AR-WWD, we implement "containment" as a constraint: ${\sup _{Z \in \mathcal{Z}}}\frac{{P_t^\varphi \left( Z_t \right)}}{{P_s^\varphi \left( Z_s \right)}} \le 1 - \beta$, and  $1>\beta >0$. Meanwhile, for AR-WWD, the first "Wasserstein" refers to the Wasserstein distance between the probability distributions of two domains in the image space. The second "Wasserstein" indicates that we use the Wasserstein distance as the ground metric $c(\cdot)$, therefore, the ground metric between $Z_s$ and $Z_t$ is defined as:
\begin{equation}
\begin{array}{l}
c\left( {{Z_s},{Z_t}} \right) \\
\quad\quad\quad= {W_{q,{d_\Omega }}}\left( {{Z_s},{Z_t}} \right)\\
\quad\quad\quad= {\left( {\mathop {\inf }\limits_{\mu \left( {{z_s},{z_t}} \right) \in \Pi \left( {{z_s},{z_t}} \right)} \int {{d_\Omega }{{\left( {{z_s},{z_t}} \right)}^q}d\mu } } \right)^{\frac{1}{q}}}, 
\end{array}
\end{equation}
where ${z_s} \sim {Z_s}$ and ${z_t} \sim {Z_t}$  denote two random variables in the feature vector element space, in other word, ${z_s}$ and ${z_t}$ denote the random dimensions of feature vectors ${Z_s}$ and ${Z_t}$, respectively. $\Pi\left(z_s,z_t\right)$ denotes the set of all joint distributions $\mu \left( {z_s,z_t} \right)$ that satisfies $z_s = \int_{z_t} {\mu \left( {z_s,z_t} \right)d{z_t}} $, $z_t = \int_{z_s} {\mu \left( {z_s,z_t} \right)d{z_s}}$, and ${{d_\Omega }\left( {z^s,z^t} \right)}$ is defined as the spatial distance, e.g., the Euclidean distance, between two position coordinate vectors. Moreover, it is important to note that although ${z_s}$ and ${z_t}$ each represent a specific dimension of the source domain feature vector and the target domain feature vector, they are implemented as vectors during the training process. For example, if ${z_s}$ denotes the $i$-th element of vector ${Z_s}$, then ${z_s}$ is represented as a vector identical in dimensionality to ${Z_s}$. In this vector ${z_s}$, the $i$-th element is equal to the $i$-th element of ${Z_s}$, while all other elements are set to zero. This representation ensures that ${z_s}$ maintains the original vector structure, allowing for consistent vector operations within the training algorithm. According to Equation (\ref{Eq:Wa}), by setting $p=1$ and $q=2$, AR-WWD is defined as:
\begin{equation}
\begin{array}{l}
{W_{1,{W_{2,{d_\Omega }}}}}\left( {P_s^\varphi\left( Z_s \right) ,P_t^\varphi \left( Z_t \right)} \right) \\\quad\quad\quad= \mathop {\inf }\limits_{\mu \left( {{Z_s},{Z_t}} \right) \in \Pi \left( {{Z_s},{Z_t}} \right)} \int {{W_{2,{d_\Omega }}}\left( {{Z^s},{Z^t}} \right)} d\mu,
\end{array}
\end{equation}
where $Z_s$ and $Z_t$ are two random variables in the latent feature space $\mathcal{Z}$ with ${Z_s} \sim P_s^\varphi(Z_s)$ and ${Z_t} \sim P_t^\varphi(Z_t)$, $\Pi \left( {Z_s,Z_t} \right)$ denotes the set of all joint distributions $\mu \left( {Z_s,Z_t} \right)$ that satisfies $\left( {1 - \beta } \right){P_t^\varphi}  \ge  \int_{{Z_s}} {u\left( {{Z_s},{Z_t}} \right)} d {{Z_s}}, {P_s^\varphi} = \int_{{Z_t}} {u\left( {{Z_s},{Z_t}} \right)} d {{Z_t}}$.

Based on \cite{Solomon_Rustamov_Guibas_Butscher_2014, Rubner_2000, Li_2018, dukler2019wasserstein, villani2009optimal}, the duality of AR-WWD can be derived as follows:
\begin{equation}
\label{qw:wwd}
\begin{array}{l}
{W_{1,{W_{2,{d_\Omega }}}}}\left( {P_s^\varphi\left( Z \right) ,P_t^\varphi \left( Z \right)} \right) \\\quad\quad= \mathop  {\sup }\limits_{f \in C\left( Z  \right)} {\mathbbm{E}_{{P_s^\varphi}\left( Z \right)}}f\left( Z \right) - \left( 1 - \beta\right)  {\mathbbm{E}_{{P_t^\varphi}\left( Z \right)}}f\left( Z \right)\\
s.t.{\rm{  }}\int_\Omega  {\left\| {{\nabla _z}{\delta _Z}f\left( Z \left( z \right)\right)} \right\|_{{d_\Omega }}^2Z\left( z \right)dz \le 1},
\end{array}
\end{equation}
where $C: Z \to \mathbb{R}$ represents the function that are continuous and bounded everywhere, $\nabla$ is the gradient operator in the feature vector element space $\Omega$, $1 > \beta  > 0$ is a per-given hyper-parameter, and $\delta$ is the $L_2$ gradient in the latent feature space $\mathcal{Z}$.

Based on \cite{Gulrajani_Ahmed_Arjovsky_Dumoulin_Courville, Petzka_Fischer_Lukovnikov_2017}, for the discrete case, we can rewrite Equation (\ref{qw:wwd}) as
\begin{equation}
\label{qw:dwwd}
\resizebox{0.9\linewidth}{!}{$
\begin{array}{l}
\mathop {\sup }\limits_{f \in C\left( Z \right)} \left[ {\frac{1}{N}\sum\limits_{i = 1}^N {f\left( {{Z_{s,i}}} \right)}  - \frac{{1 - \beta }}{N}\sum\limits_{i = 1}^N {f\left( {{Z_{t,i}}} \right)} } + \frac{\lambda }{2NM} \sum\limits_{k\in \left \{ s,t \right \} }   \right.\\
\quad \left. { \sum\limits_{i = 1}^N\sum\limits_{j = 1}^M {{{\left( {\left( {{\nabla _{Z_{k,i}(j)}}{\delta _{Z_{k,i}}}f\left( Z_{k,i}(j)\right)} \right)^2Z_{k,i}(j) - 1} \right)}^2}} } \right],
\end{array}$}
\end{equation}
where $Z_{s,i}$ denotes the $i$-th sample in the mini-batch of source domain, $Z_{t,i}$ denotes the $i$-th sample in the mini-batch of target domain, $N$ is the number of samples in the mini-batch, $M$ is the dimension of vector $Z_{k,i}$, $Z_{k,i}(j)$ represents element of the $j$-th dimension of $Z_{k,i}$, and $\lambda$ is the hyperparameter.

\subsubsection{Local consistency module}
In this subsubsection, we present the specific implementation of 
conditional mutual information $I\left( {{X_s};{X_t}\left| {{Z_t}} \right.} \right)$. Generally, the conditional mutual information is defined as:
\begin{equation} \label{qww:CMI}
\resizebox{0.9\linewidth}{!}{$
I(X_s; X_t \mid Z_t) = \mathbb{E}_{P(Z_t)} \left[ \mathbb{E}_{P(X_s, X_t \mid Z_t)} \left[ \log \frac{P(X_t \mid X_s, Z_t)}{P(X_t \mid Z_t)} \right]\right],$}
\end{equation}
where $P(X_s, X_t \mid Z_t)$ is the joint probability distribution of $X_s$ and $X_t$ under condition $Z_t$, $P(X_t \mid Z_t)$ is the probability distribution of $X_t$ under condition $Z_t$, and $P(Z_t)$ is the probability distribution of $Z_t$. $I(X_s; X_t \mid Z_t)$ measures how much information $X_s$ and $X_t$ share given $Z_t$.

To derive a tractable estimator of $I(X_s; X_t \mid Z_t)$, motivated by \cite{oord2018representation, chen2020simple, chen2020intriguing}, we define the conditional-noise-contrastive-estimation-based (CNCE) estimator of the conditional mutual information as:
\begin{equation} \label{qwasdasdf}
\resizebox{0.9\linewidth}{!}{$
\begin{array}{l}
I_{\text{CNCE}}(X_s; X_t \mid Z_t, \phi, K) := \\
\mathbb{E}_{P(Z_t,X_s,X_t)} \mathbb{E}_{P(X_t^{(2:K)}\mid Z_t)} \left[ \log \frac{e^{\phi(X_s,X_t,Z_t)}}{\frac{1}{K}\sum_{k=1}^{K} e^{\phi(X_s,X_t^{(k)},Z_t)}} \right],
\end{array}$}
\end{equation}
where $K$ is the total number of “negative” samples, \(\phi: {X}_s \times {X}_t \times {Z}_t \to \mathbb{R}\) is a differentiable function, \(X_t^{(1)}=X_t\) is the “positive” sample from \(P(X_t \mid X_s,Z_t)\), and \(X_t^{(2)},\dots,X_t^{(K)}\) are “negative” samples drawn i.i.d. from \(P(X_t \mid Z_t)\). We have:

\begin{proposition}\label{gkhomcaomsg}
Define the Conditional Noise Contrastive Estimation (CNCE) functional as Equation (\ref{qwasdasdf}), then the following hold:

1) For any choice of \(\phi\) and any \(K \ge 1\),
\begin{equation}
I_{\text{CNCE}}(X_s; X_t \mid Z_t, \phi, K) \;\le\; I(X_s; X_t \mid Z_t).
\end{equation}

2) There exists a function \(\phi^*\) that attains the supremum. Specifically, if
\begin{equation}
\phi^*(X_s,X_t,Z_t) = \log \frac{P(X_t \mid X_s,Z_t)}{P(X_t \mid Z_t)} + c(X_s,Z_t),
\end{equation}
where \(c(X_s,Z_t)\) is a function independent of \(X_t\), then we have:
\begin{equation}
\begin{array}{l}
\sup_{\phi} I_{\text{CNCE}}(X_s; X_t \mid Z_t, \phi, K) \\
= I_{\text{CNCE}}(X_s; X_t \mid Z_t, \phi^*, K) \\
= I(X_s; X_t \mid Z_t).
\end{array}
\end{equation}

3) As the number of negative samples \(K\) grows, the approximation tightens. In particular, we have:
\begin{equation}\label{hlyguty}
\lim_{K \to \infty} \sup_{\phi} I_{\text{CNCE}}(X_s; X_t \mid Z_t, \phi, K) = I(X_s; X_t \mid Z_t).
\end{equation}
\end{proposition}

The proof is shown in Appendix \textbf{Theorem \ref{vbjty}}. From Equation (\ref{qwasdasdf}), we can obtain that we need to give the formulation of $P(X_t^{(2:K)}\mid Z_t)$ and $P(Z_t,X_s,X_t)$. Because that $\{Z_t,X_t\}$ is a pair, thus, we implement $P(Z_t,X_s,X_t)$ as $\{ {Z_t},{X_t},{X_{s,j}}\}$, where $X_{s,j}$ is the closest sample to $X_t$ in the source domain mini-batch. For $P(X_t^{(2:K)}\mid Z_t)$, we set $K$ to the size of the mini-batch minus 1, and we implement $P(X_t^{(2:K)}\mid Z_t)$ as so other samples in the target domain mini-batch except $X_t$.

\subsubsection{Overall objective}
Based on the global and local consistency modules, we present the proposed RLGLC. RLGLC first maps the input data into the latent space to obtain the feature representation by a projection function $\varphi $. Then, a classifier $\psi $ is learned from the labeled source domain samples in the latent space. The overall objective is formulated as follows:
\begin{equation}
\label{Eq:bbb}
\begin{array}{l}
\mathop {\min }\limits_{\varphi ,\psi } {W_{1,{W_{2,{d_\Omega }}}}}\left( {P_s^\varphi ,P_t^\varphi } \right)+I_{\text{CNCE}}(X_s; X_t \mid Z_t, \phi, K) \\+ {\mathcal{L}_{cl}}\left( {\psi \left( \varphi\left(X_s \right) \right),Y_s} \right)+ \alpha \Delta(\phi ,\psi),
\end{array}
\end{equation}
where $\alpha$ is a hyper-parameter, $Y_s$ is the labels of source domain samples, $\mathcal{L}_{cl}$ is the cross-entropy loss, and $\Delta$ is the regularization term to punish large parameters in $\varphi $ and $\psi $. The feature extractor $\varphi $ and the classifier $\psi $ can be jointly trained by back-propagation.

Because the computation of ${W_{1,{W_{2,{d_\Omega }}}}}\left( {P_s^\varphi ,P_t^\varphi } \right)$ and $I_{\text{CNCE}}(X_s; X_t \mid Z_t, \phi, K)$ requires a maximization procedure, the optimization of Equation (\ref{Eq:bbb}) is implemented through adversarial learning. Concretely, the calculation of $\phi$ is as follows: first, project $X_s$ and $X_t$ into the $Z_t$ feature space to obtain $Z_s^+$ and $Z_t^+$. Next, define $\frac{1}{2}(Z_s^{ + {\rm{T}}}{Z_t} + Z_s^{ + {\rm{T}}}{Z_t})$ as the output of $\phi$. Building on this formulation, the adversarial learning framework is executed in two distinct steps:


\textbf{Step 1}: Keeping $\varphi$ and $\psi$ fixed, we first optimize $f$ following Equation (\ref{qw:dwwd}), and then optimize $\phi$ based on Equation (\ref{hlyguty}).

\textbf{Step 2}: With $f$ and $\phi$ fixed, we compute $W_{1,{W_{2,{d_\Omega }}}}(P_s^\varphi,P_t^\varphi)$ and $I_{\text{CNCE}}(X_s; X_t \mid Z_t,\phi,K)$. Guided by these values, we then optimize $\varphi$ and $\psi$ following Equation (\ref{Eq:bbb}).

\section{Bayes error rate}

In this section, we provide a theoretical analysis of the generalization classification error for the learned feature representation by RLGLC. We utilize the Bayes error rate \cite{feder1994relations} to measure the quality of the learned feature representation. Note that the Bayes error rate cannot be estimated empirically and can only be used for theoretical analysis. Specifically, let $P_e$ be the Bayes error rate of the learned representation $Z_t$ and ${{{\hat Y}_t}}$ be
the estimation of the ground truth label $Y_t$ by the learned classifier. Then, we have:
\begin{equation}
{P_e}: = {\mathbb{E}_{{Z_t} \sim P_t^\varphi }}\left[ {1 - \mathop {\max }\limits_{{y_t} \in {Y_t}} P\left( {{{\hat Y}_t} = {y_t}\left| {{Z_t}} \right.} \right)} \right].
\end{equation}

Formally, let $\left| {{Y_t}} \right|$ be the cardinality of $Y_t$, and let ${\rm{Th}}\left( x \right) = \min \left\{ {\max \left\{ {x,0} \right\},1 - {1 \mathord{\left/
			{\vphantom {1 {\left| {{Y_t}} \right|}}} \right.
			\kern-\nulldelimiterspace} {\left| {{Y_t}} \right|}}} \right\}$ be a thresholding function, we can obtain: ${P_e} = {\rm{Th}}\left( {{{P}_e}} \right)$. Then, we can establish the connection between the Bayes error rate and the information theory-based discriminability measure:
            
\begin{theorem}
	\label{qww}
For an arbitrary learned feature representation pair \((Z_s,Z_t)\) derived from \((X_s,X_t)\), and for any decision rule that outputs a hypothesis \(\hat{Y}\) of \(Y\), we have the following upper bounds on the average error probability \(\bar{P}_e\):

1. Bound with \( I(X_s; Z_s \mid X_t) \):
\begin{equation}
   \bar{P}_e \leq 1 - \exp\bigl[-H(Y) + I(X_s; Z_s \mid X_t)\bigr].
   \end{equation}

2. Bound with \( I(X_t; Z_t \mid X_s) \):
\begin{equation}
   \bar{P}_e \leq 1 - \exp\bigl[-H(Y) + I(X_t; Z_t \mid X_s)\bigr].
   \end{equation}

3. Bound with \( I(X_s; X_t \mid Z_s) \):
   \begin{equation}
   \bar{P}_e \leq 1 - \exp\bigl[-H(Y) + I(X_s; X_t \mid Z_s)\bigr].
   \end{equation}

4. Bound with \( I(X_s; X_t \mid Z_t) + \Delta(\phi,\psi) \):
   \begin{equation}
   \bar{P}_e \leq 1 - \exp\bigl[-H(Y) + I(X_s; X_t \mid Z_t) + \Delta(\phi,\psi)\bigr].
   \end{equation}

\end{theorem}

\begin{theorem}\label{gvbnjiout}
Let \(\bar{P}_e\) be the average error probability of predicting \(Y\) given \((X_s, X_t, Z_s, Z_t)\) and a corresponding decision rule. Suppose we have the four upper bounds shown in \textbf{Theorem \ref{qww}}, then, we can unify these into a single upper bound:
\begin{equation}
\begin{array}{l}
{{\bar P}_e} \le 1 - \exp [ - H(Y) + \min \{ I({X_s};{Z_s}\mid {X_t}),{\mkern 1mu} \\
\quad\quad\quad I({X_t};{Z_t}\mid {X_s}),{\mkern 1mu} I({X_s};{X_t}\mid {Z_s}),{\mkern 1mu} \\
\quad\quad\quad\quad\quad\quad I({X_s};{X_t}\mid {Z_t})+ \Delta (\phi ,\psi )\} ].
\end{array}
\end{equation}
\end{theorem}

The proofs are presented in Appendix \textbf{Theorem \ref{qwwq}} and \textbf{Theorem \ref{qww159357}}. From \textbf{Theorem \ref{qww}} and \textbf{Theorem \ref{gvbnjiout}}, we can decrease the corresponding Bayes error rate by reducing $I({X_s};{Z_s}\mid {X_t})$, $I({X_t};{Z_t}\mid {X_s})$, $I({X_s};{X_t}\mid {Z_s})$ and $I({X_s};{X_t}\mid {Z_t})$. This demonstrate that learning by Equation (\ref{Eq:bbb}) is effective. Compared RLGLC with other UDA methods, we can obtain that RLGLC can achieve a lower upper bound, this is because that RLGLC offers a more precise measurement of $I(X_s; Z_s \mid X_t)$ and $I(X_t; Z_t \mid X_s)$, facilitating a more effective maximization of these information measures. This improvement leads to a more substantial reduction in the Bayes error rate. Moreover, by incorporating an additional term $I(X_s; X_t \mid Z_t)$, RLGLC provides more nuanced control over the Bayes error rate, enabling finer adjustments to achieve optimal performance.

\begin{table}[t]
    \centering
	\caption{Accuracy (\%) and standard deviation on the Office-31 dataset for unsupervised domain adaptation by using ResNet-50 as the backbone (except TransVQA, which uses ViT-B as a backbone). The best results are highlighted in \textbf{bold}.}
    \resizebox{\linewidth}{!}{
			\begin{tabular}{l|cccccc|c}
				\toprule
				Method & A$\to$W & D$\to$W & W$\to$D & A$\to$D & D$\to$A & W$\to$A & Avg \\
				\midrule
				ResNet-50 & 68.4$ \pm $0.2 & 96.7$ \pm $0.1 & 99.3$ \pm $0.1 & 68.9$ \pm $0.2 & 62.5$ \pm $0.3 & 60.7$ \pm $0.3 & 76.1 \\
				
				DAN & 80.5$ \pm $0.4 & 97.1$ \pm $0.2 & 99.6$ \pm $0.1 & 78.6$ \pm $0.2 & 63.6$ \pm $0.3 & 62.8$ \pm $0.2 & 80.4 \\
				
				DANN & 82.0$ \pm $0.4 & 96.9$ \pm $0.2 & 99.1$ \pm $0.1 & 79.7$ \pm $0.4 & 68.2$ \pm $0.4 & 67.4$ \pm $0.5 & 82.2 \\
				
				JAN & 85.4$ \pm $0.3 & 97.4$ \pm $0.2 & 99.8$ \pm $0.2 & 84.7$ \pm $0.3 & 68.6$ \pm $0.3 & 70.0$ \pm $0.4 & 84.3 \\
				
				GTA & 89.5$ \pm $0.5 & 97.9$ \pm $0.3 & 99.8$ \pm $0.4 & 87.7$ \pm $0.5 & 72.8$ \pm $0.3 & 71.4$ \pm $0.4 & 86.5 \\
				
				CDAN & 93.1$ \pm $0.2 & 98.2$ \pm $0.2 & \textbf{100.0$ \pm $0.0} & 89.8$ \pm $0.3 & 70.1$ \pm $0.4 & 68.0$ \pm $0.4 & 86.6 \\
				
				CDAN+E & 94.1$ \pm $0.1 & 98.6$ \pm $0.1 & \textbf{100.0$ \pm $0.0} & 92.9$ \pm $0.2 & 71.0$ \pm $0.3 & 69.3$ \pm $0.3 & 87.7 \\
				
				BSP+DANN & 93.0$ \pm $0.2 & 98.0$ \pm $0.2 & \textbf{100.0$ \pm $0.0} & 90.0$ \pm $0.4 & 71.9$ \pm $0.3 & 73.0$ \pm $0.3 & 87.7 \\
				
				BSP+CDAN & 93.3$ \pm $0.2 & 98.2$ \pm $0.2 & \textbf{100.0$ \pm $0.0} & 93.0$ \pm $0.2 & 73.6$ \pm $0.3 & 72.6$ \pm $0.3 & 88.5 \\
				
				ADDA & 86.2$ \pm $0.5 & 96.2$ \pm $0.3 & 98.4$ \pm $0.3 & 77.8$ \pm $0.3 & 69.5$ \pm $0.4 & 68.9$ \pm $0.5 & 82.9 \\
				
				MCD & 88.6$ \pm $0.2 & 98.5$ \pm $0.1 & \textbf{100.0$ \pm $0.0} & 92.2$ \pm $0.2 & 69.5$ \pm $0.1 & 69.7$ \pm $0.3 & 86.5 \\
				
				MDD & 94.5$ \pm $0.3 & 98.4$ \pm $0.1 & \textbf{100.0$ \pm $0.0} & 93.5$ \pm $0.2 & 74.6$ \pm $0.3 & 72.2$ \pm $0.1 & 88.9 \\
				
				SymmNets & 94.2$ \pm $0.1 & 98.8$ \pm $0.0 & \textbf{100.0$ \pm $0.0} & 93.5$ \pm $0.3 & 74.4$ \pm $0.1 & 73.4$ \pm $0.2 & 89.1 \\
				
				GVB-GD & 94.8$ \pm $0.5 & 98.7$ \pm $0.3 & \textbf{100.0$ \pm $0.0} & 95.0$ \pm $0.4 & 73.4$ \pm $0.3 & 73.7$ \pm $0.4 & 89.3
				\\
				
				CAN & 94.5$ \pm $0.3 & 99.1$ \pm $0.2 & 99.8$ \pm $0.2 & 95.0$ \pm $0.3 & 78.0$ \pm $0.3 & 77.0$ \pm $0.3 & 90.6 \\
				
				ETD & 92.1$ \pm $0.4 & \textbf{100.0$\pm$0.2} & \textbf{100.0$ \pm $0.3} & 88.0$ \pm $0.2 & 71.0$ \pm $0.3 & 67.8$ \pm $0.3 & 86.5 \\
				
				SRDC & 95.7$ \pm $0.2 & 99.2$ \pm $0.1 & \textbf{100.0$ \pm $0.0} & 95.8$ \pm $0.2 & 76.7$ \pm $0.3 & 77.1$ \pm $0.1 & 90.8 \\
				
				ACTIR & 94.9$ \pm $0.2 & 98.2$ \pm $0.2 & 99.9$ \pm $0.1 & 92.7$ \pm $0.4 & 75.6$ \pm $0.2 & 74.1$ \pm $0.4 & 89.2 \\
				
				TCM & 93.1$ \pm $0.4 & 98.2$ \pm $0.3 & 99.8$ \pm $0.3 & 90.1$ \pm $0.2 & 72.3$ \pm $0.2 & 74.7$ \pm $0.5 & 88.0 \\
				
				ERM & 89.7$ \pm $0.4 & \textbf{100.0$ \pm $0.0} & 99.4$ \pm $0.3 & 94.7$ \pm $0.4 & 74.6$ \pm $0.4 & 74.1$ \pm $0.4 & 88.8 \\
				
				ICDA & 92.8$ \pm $0.3 & 99.1$ \pm $0.4 & \textbf{100.0$ \pm $0.0} & 95.1$ \pm $0.2 & 72.7$ \pm $0.4 & 75.1$ \pm $0.3 & 89.1 \\
				
				iMSDA & 94.5$ \pm $0.3 & 98.9$ \pm $0.5 & \textbf{100.0$ \pm $0.0} & 95.2$ \pm $0.3 & 74.1$ \pm $0.3 & 75.2$ \pm $0.1 & 89.7 \\
				
				UniOT & 93.7$ \pm $0.4 & 99.4$ \pm $0.3 & \textbf{100.0$ \pm $0.0} & 93.1$ \pm $0.4 & 75.0$ \pm $0.3 & 76.1$ \pm $0.1 & 89.6 \\
                    WDGRL & 92.1$ \pm $0.2 & 97.9$ \pm $0.3 & 99.9$ \pm $0.4 & 93.1$ \pm $0.3 & 73.8$ \pm $0.2 & 74.9$ \pm $0.2 & 88.6 \\
				PPOT & 90.7$ \pm $0.2 & 99.1$ \pm $0.5 & \textbf{100.0$ \pm $0.0} & 93.8$ \pm $0.2 & 78.1$ \pm $0.4 & 76.4$ \pm $0.6 & 89.6 \\
                CPH & 90.6$ \pm $0.4 & 95.8$ \pm $0.2 & 99.9$ \pm $0.2 & 89.1$ \pm $0.2 & 76.8$ \pm $0.4 & 77.4$ \pm $0.3 & 88.3 \\
                SSRT+GH++ & 96.7$ \pm $0.3 & 91.4$ \pm $0.4 & \textbf{100.0$ \pm $0.0} & 94.2$ \pm $0.5 & 77.8$ \pm $0.4 & 75.6$ \pm $0.4 & 89.2\\
                TransVQA & 96.4$ \pm $0.1 & 99.6$ \pm $0.2 & \textbf{100.0$ \pm $0.0} & 97.4$ \pm $0.2 & 79.7$ \pm $0.4 & \textbf{80.1$ \pm $0.4} & 92.2\\
				\midrule
				\textbf{RLGC} & 95.6$ \pm $0.3 & 98.9$ \pm $0.3 & \textbf{100.0$ \pm $0.0} & 95.7$ \pm $0.3 & 74.9$ \pm $0.2 & 76.8$ \pm $0.2 & 90.3 \\
				\textbf{RLGC*} & 96.6$ \pm $0.2 & 99.6$ \pm $0.2 & \textbf{100.0$ \pm $0.0} & 96.1$ \pm $0.3 & 77.6$ \pm $0.5 & 77.4$ \pm $0.2 & 91.2 \\
				\textbf{RLGLC} & \textbf{97.8$ \pm $0.3} & \textbf{100.0$ \pm $0.0} & \textbf{100.0$ \pm $0.0} & \textbf{98.1$ \pm $0.4} & \textbf{80.1$ \pm $0.3} & 78.9$ \pm $0.3 & \textbf{92.5} \\
				\bottomrule
			\end{tabular}}

	
	\label{tab:1}
\end{table}

\section{Experiments}
\label{123}

In this section, we validate the effectiveness of the proposed RLGLC through experiments. The main contents of this chapter report the experimental setup, experimental details, experimental results, statistical analysis, and ablation experiments. We repeat the experiments 5 times and then report the average accuracy of the five results. For a fair comparison, the reported results of most comparison methods come from their original papers. When the original papers do not have relevant experimental results, we obtained the results by manually reproducing.

\subsection{Experimental setup} 
\label{s}
We evaluate the proposed approach on multiple datasets: 1) Office-Home \cite{vh17}, including 4 domains; 2) Office-31 dataset \cite{ks10}, including 3 domains; 3) VisDa-2017 dataset \cite{pxc17c}, including 12 categories; 4) Digits datasets \cite{ganin2016domain}, including SVHN (S) \cite{hjj02}, USPS (U) \cite{ny11}, and MNIST (M) \cite{lcy98}; 5) DomainNet dataset \cite{peng2019moment}, including 12 categories. We compare our proposed RLGLC with state-of-the-art unsupervised domain adaptation methods including: ResNet-50\cite{he2016deep}, ResNet-101\cite{he2016deep}, DAN\cite{msl15}, DANN\cite{ganin2016domain}, JAN\cite{lms16}, GTA\cite{ss17}, ADDA\cite{te17}, UNIT\cite{myl17}, CyCADA\cite{jh18}, CDAN\cite{lms18}, CDAN+E\cite{lms18}, BSP+DANN\cite{xyc19}, BSP+ADDA\cite{xyc19}, BSP+CDAN\cite{xyc19}, ADDA\cite{te17}, MCD\cite{saito2018maximum}, MDD\cite{ycz19}, SWD\cite{lee2019sliced}, CAN\cite{kang2019contrastive}, SymmNets\cite{zhang2020unsupervised}, GVB-GD\cite{cui2020gradually}, ETD\cite{li2020enhanced}, SRDC\cite{tang2020unsupervised}, HDAN\cite{cui2020heuristic}, ACTIR\cite{jiang2022invariant}, TCM\cite{yue2021transporting}, ERM\cite{shen2022connect}, ICDA\cite{gulrajani2022identifiability}, iMSDA\cite{kong2022partial}, UniOT\cite{chang2022unified}, DCAN \cite{wu2018dcan}, CBST \cite{zou2018unsupervised}, ADV \cite{vu2019advent}, SWLS \cite{dong2019semantic}, MSL \cite{chen2019domain}, PyCDA \cite{lian2019constructing}, CrCDA \cite{huang2020contextual}, CSCL \cite{dong2020cscl}, LSE \cite{subhani2020learning}, FADA \cite{wang2020classes}, SS-UDA \cite{pan2020unsupervised}, PIT \cite{lv2020cross}, ELS-DA \cite{dong2020can}, WDGRL \cite{shen2018wasserstein}, PPOT \cite{wang2024probability}, SSRT+GH++\cite{huang2024gradient}, TransVQA \cite{sun2024transvqa}, PDA \cite{bai2024prompt}, CPH\cite{cui2024effective}, SAMB-D \cite{li2024semantic}, and TCPL\cite{gao2024learning}. The classification accuracies including the average accuracy on each dataset and specific transfer task accuracies are used as the performance measure. 

\begin{table*}[!ht]
	\caption{Accuracy (\%) and standard deviation on the Office-Home dataset for unsupervised domain adaptation by using ResNet-50 as the backbone. The best results are highlighted in \textbf{bold}.}
			\centering
\resizebox{\linewidth}{!}{
			\begin{tabular}{l|cccccccccccc|c}
				\toprule
				Method & A$\to$C & A$\to$P & A$\to$R & C$\to$A & C$\to$P & C$\to$R & P$\to$A & P$\to$C & P$\to$R & R$\to$A & R$\to$C & R$\to$P & Avg \\
				\midrule
ResNet-50 & 34.9$\pm$0.2 & 50.0$\pm$0.1 & 58.0$\pm$0.1 & 37.4$\pm$0.1 & 41.9$\pm$0.2 & 46.2$\pm$0.3 & 38.5$\pm$0.1 & 31.2$\pm$0.2 & 60.4$\pm$0.2 & 53.9$\pm$0.1 & 41.2$\pm$0.1 & 59.9$\pm$0.3 & 46.1 \\
DAN & 43.6$\pm$0.1 & 57.0$\pm$0.2 & 67.9$\pm$0.2 & 45.8$\pm$0.2 & 56.5$\pm$0.1 & 60.4$\pm$0.2 & 44.0$\pm$0.2 & 43.6$\pm$0.1 & 67.7$\pm$0.2 & 63.1$\pm$0.1 & 51.5$\pm$0.1 & 74.3$\pm$0.2 & 56.3 \\
DANN & 45.6$\pm$0.3 & 59.3$\pm$0.1 & 70.1$\pm$0.2 & 47.0$\pm$0.1 & 58.5$\pm$0.2 & 60.9$\pm$0.2 & 46.1$\pm$0.1 & 43.7$\pm$0.1 & 68.5$\pm$0.1 & 63.2$\pm$0.1 & 51.8$\pm$0.1 & 76.8$\pm$0.2 & 57.6 \\
JAN & 45.9$\pm$0.2 & 61.2$\pm$0.2 & 68.9$\pm$0.1 & 50.4$\pm$0.2 & 59.7$\pm$0.1 & 61.0$\pm$0.2 & 45.8$\pm$0.1 & 43.4$\pm$0.2 & 70.3$\pm$0.2 & 63.9$\pm$0.1 & 52.4$\pm$0.2 & 76.8$\pm$0.2 & 58.3 \\
CDAN & 49.0$\pm$0.1 & 69.3$\pm$0.2 & 74.5$\pm$0.2 & 54.4$\pm$0.2 & 66.0$\pm$0.2 & 68.4$\pm$0.1 & 55.6$\pm$0.1 & 48.3$\pm$0.2 & 75.9$\pm$0.3 & 68.4$\pm$0.2 & 55.4$\pm$0.2 & 80.5$\pm$0.2 & 63.8 \\
CDAN+E & 50.7$\pm$0.1 & 70.6$\pm$0.1 & 76.0$\pm$0.1 & 57.6$\pm$0.2 & 70.0$\pm$0.1 & 70.0$\pm$0.2 & 57.4$\pm$0.1 & 50.9$\pm$0.1 & 77.3$\pm$0.1 & 70.9$\pm$0.1 & 56.7$\pm$0.2 & 81.6$\pm$0.3 & 65.8 \\
BSP+DANN & 51.4$\pm$0.2 & 68.3$\pm$0.2 & 75.9$\pm$0.1 & 56.0$\pm$0.1 & 67.8$\pm$0.1 & 68.8$\pm$0.2 & 57.0$\pm$0.1 & 49.6$\pm$0.1 & 75.8$\pm$0.2 & 70.4$\pm$0.2 & 57.1$\pm$0.2 & 80.6$\pm$0.3 & 64.9 \\
BSP+CDAN & 52.0$\pm$0.2 & 68.6$\pm$0.2 & 76.1$\pm$0.1 & 58.0$\pm$0.2 & 70.3$\pm$0.1 & 70.2$\pm$0.2 & 58.6$\pm$0.1 & 50.2$\pm$0.1 & 77.6$\pm$0.2 & 72.2$\pm$0.1 & 59.3$\pm$0.1 & 81.9$\pm$0.2 & 66.3 \\
MDD & 54.9$\pm$0.2 & 73.7$\pm$0.1 & 77.8$\pm$0.1 & 60.0$\pm$0.1 & 71.4$\pm$0.1 & 71.8$\pm$0.2 & 61.2$\pm$0.1 & 53.6$\pm$0.2 & 78.1$\pm$0.2 & 72.5$\pm$0.1 & 60.2$\pm$0.1 & 82.3$\pm$0.2 & 68.1 \\
SymmNets & 48.1$\pm$0.3 & 74.3$\pm$0.2 & 78.7$\pm$0.1 & 64.6$\pm$0.2 & 71.8$\pm$0.2 & 74.1$\pm$0.1 & 64.4$\pm$0.1 & 50.0$\pm$0.1 & 80.2$\pm$0.1 & 74.3$\pm$0.1 & 53.1$\pm$0.3 & 83.2$\pm$0.1 & 68.1 \\
ETD & 51.3$\pm$0.2 & 71.9$\pm$0.1 & 85.7$\pm$0.1 & 57.6$\pm$0.1 & 69.2$\pm$0.2 & 73.7$\pm$0.1 & 57.8$\pm$0.2 & 51.2$\pm$0.2 & 79.3$\pm$0.1 & 70.2$\pm$0.1 & 57.5$\pm$0.2 & 82.1$\pm$0.1 & 67.3 \\
GVB-GD & 57.0$\pm$0.2 & 74.7$\pm$0.1 & 79.8$\pm$0.1 & 64.6$\pm$0.2 & 74.1$\pm$0.1 & 74.6$\pm$0.2 & 65.2$\pm$0.1 & 55.1$\pm$0.2 & 81.0$\pm$0.1 & 74.6$\pm$0.1 & 59.7$\pm$0.2 & 84.3$\pm$0.2 & 70.4 \\
HDAN & 56.8$\pm$0.2 & 75.2$\pm$0.1 & 79.8$\pm$0.1 & 65.1$\pm$0.2 & 73.9$\pm$0.1 & 75.2$\pm$0.2 & 66.3$\pm$0.1 & \textbf{56.7$\pm$0.2} & 81.8$\pm$0.1 & 75.4$\pm$0.1 & 59.7$\pm$0.2 & 84.7$\pm$0.2 & 70.9 \\
SRDC & 52.3$\pm$0.2 & 76.3$\pm$0.1 & 81.0$\pm$0.1 & 69.5$\pm$0.2 & 76.2$\pm$0.1 & 78.0$\pm$0.2 & 68.7$\pm$0.1 & 53.8$\pm$0.1 & 81.7$\pm$0.1 & 76.3$\pm$0.2 & 57.1$\pm$0.2 & 85.0$\pm$0.3 & 71.3 \\
ACTIR & 50.6$\pm$0.1 & 74.5$\pm$0.2 & 82.1$\pm$0.1 & 66.4$\pm$0.1 & 75.4$\pm$0.2 & 77.9$\pm$0.3 & 68.4$\pm$0.2 & 53.1$\pm$0.1 & 81.9$\pm$0.3 & 75.4$\pm$0.1 & 58.1$\pm$0.2 & 84.7$\pm$0.2 & 70.7 \\
TCM & 58.6$\pm$0.2 & 74.4$\pm$0.2 & 79.6$\pm$0.2 & 64.5$\pm$0.1 & 74.0$\pm$0.2 & 75.1$\pm$0.2 & 64.6$\pm$0.2 & 56.2$\pm$0.1 & 80.9$\pm$0.2 & 74.6$\pm$0.2 & 60.7$\pm$0.1 & 84.7$\pm$0.2 & 70.7 \\
ERM & 54.2$\pm$0.1 & 72.9$\pm$0.1 & 80.1$\pm$0.2 & 67.2$\pm$0.2 & 75.1$\pm$0.2 & 76.2$\pm$0.1 & 67.1$\pm$0.1 & 52.1$\pm$0.1 & 78.2$\pm$0.2 & 72.7$\pm$0.2 & 58.2$\pm$0.1 & 82.1$\pm$0.3 & 69.7 \\
ICDA & 53.8$\pm$0.3 & 74.3$\pm$0.3 & 82.7$\pm$0.2 & 67.9$\pm$0.2 & 74.7$\pm$0.3 & 76.1$\pm$0.2 & 66.9$\pm$0.3 & 55.8$\pm$0.2 & 80.7$\pm$0.2 & 74.9$\pm$0.2 & 58.6$\pm$0.1 & 83.7$\pm$0.2 & 70.8 \\
iMSDA & 55.1$\pm$0.3 & 73.7$\pm$0.3 & 78.1$\pm$0.2 & 66.6$\pm$0.2 & 75.1$\pm$0.1 & 76.6$\pm$0.1 & 66.4$\pm$0.2 & 53.8$\pm$0.3 & 80.9$\pm$0.2 & 75.1$\pm$0.1 & 58.9$\pm$0.3 & 83.2$\pm$0.2 & 70.3 \\
UniOT & 54.1$\pm$0.2 & 73.9$\pm$0.1 & 80.9$\pm$0.1 & 68.1$\pm$0.2 & 75.1$\pm$0.1 & 76.7$\pm$0.2 & 68.1$\pm$0.1 & 55.9$\pm$0.1 & 81.1$\pm$0.1 & 75.1$\pm$0.2 & 59.6$\pm$0.2 & 84.3$\pm$0.3 & 71.1 \\
WDGRL & 53.8$\pm$0.3 & 73.2$\pm$0.2 & 79.8$\pm$0.1 & 65.1$\pm$0.2 & 73.7$\pm$0.2 & 75.9$\pm$0.2 & 63.2$\pm$0.1 & 50.7$\pm$0.2 & 78.5$\pm$0.1 & 70.9$\pm$0.2 & 57.2$\pm$0.2 & 82.1$\pm$0.3 & 68.7 \\
PPOT & 53.5$\pm$0.2 & 76.2$\pm$0.2 & 81.4$\pm$0.1 & 70.6$\pm$0.3 & 77.5$\pm$0.2 & 79.5$\pm$0.4 & 70.8$\pm$0.2 & 56.0$\pm$0.2 & 81.0$\pm$0.1 & 74.6$\pm$0.2 & 60.1$\pm$0.3 & 83.2$\pm$0.1 & 72.0 \\
PDA & 53.0$\pm$0.2 & 76.7$\pm$0.2 & 84.1$\pm$0.3 & 70.2$\pm$0.2 & 73.2$\pm$0.1 & 74.2$\pm$0.3 & \textbf{71.0$\pm$0.2} & 55.2$\pm$0.1 & 83.0$\pm$0.2 & 72.4$\pm$0.2 & 53.4$\pm$0.1 & 82.1$\pm$0.1 & 70.7\\
SAMB-D & 57.8$\pm$0.3 & 74.6$\pm$0.3 & 86.1$\pm$0.2 & 70.5$\pm$0.1 & 76.2$\pm$0.1 & \textbf{81.0$\pm$0.1} & 69.3$\pm$0.2 & 56.2$\pm$0.2 & 82.9$\pm$0.2 & \textbf{75.9$\pm$0.2} & 60.2$\pm$0.1 & 80.4$\pm$0.1 & 72.6 \\
TCPL & 56.9$\pm$0.2 & 75.0$\pm$0.1 & 85.2$\pm$0.3 & 71.0$\pm$0.2 & 76.5$\pm$0.2 & 79.0$\pm$0.3 & 70.9$\pm$0.2 & 56.5$\pm$0.3 & 55.8$\pm$0.2 & 73.0$\pm$0.2 & 60.3$\pm$0.1 & 81.6$\pm$0.3 & 70.1 \\
\midrule
\textbf{RLGC} & 56.9$\pm$0.1 & 75.9$\pm$0.1 & 83.6$\pm$0.1 & 68.9$\pm$0.1 & 75.7$\pm$0.1 & 77.7$\pm$0.2 & 65.1$\pm$0.1 & 53.9$\pm$0.2 & 80.8$\pm$0.1 & 71.5$\pm$0.1 & 59.8$\pm$0.1 & 84.9$\pm$0.1 & 71.2 \\
\textbf{RLGC*} & 58.9$\pm$0.2 & 77.2$\pm$0.1 & 85.4$\pm$0.1 & 71.2$\pm$0.2 & 76.7$\pm$0.1 & 79.3$\pm$0.1 & 66.3$\pm$0.1 & 54.6$\pm$0.2 & 83.2$\pm$0.1 & 72.6$\pm$0.1 & 61.6$\pm$0.1 & 85.9$\pm$0.2 & 72.7 \\
\textbf{RLGLC} & \textbf{59.9$\pm$0.1} & \textbf{78.9$\pm$0.2} & \textbf{87.6$\pm$0.1} & \textbf{72.7$\pm$0.2} & \textbf{78.4$\pm$0.1} & 80.1$\pm$0.1 & 68.4$\pm$0.1 & 56.6$\pm$0.2 & \textbf{84.5$\pm$0.1} & 74.7$\pm$0.1 & \textbf{62.7$\pm$0.1} & \textbf{87.5$\pm$0.1} & \textbf{74.3} \\
				\bottomrule
			\end{tabular}
}
	\label{tab:2}
\end{table*}

\subsection{Implementation details}
\label{id}
All experiments are implemented with PyTorch and optimized by the Adam optimizer. All experimental results of the compared methods are either used as reported in the original paper or are obtained by reproducing according to the settings of the original paper. The hyperparameter $\alpha$, $\beta $ and $\lambda $ in our proposed RLGLC are fixed. The hyperparameter $\beta$ is selected from the range of $\left\{ 0.1, 0.2, 0.3,...,0.9 \right\}$ and the hyperparameters $\alpha$ and  $\lambda$ are selected from the range of $\{ {10^{ - 3}},{10^{ - 2}},{10^{ - 1}},1, 10, {10^2},{10^3}\} $. For a fair comparison, we use the same network architecture as the compared methods on each dataset. Specifically, the backbone on the Office-31 dataset, the Office-Home dataset, and the Digits dataset is set to ResNet-50, the backbone on the VisDa-2017 dataset, GTA5 and Cityscapes dataset, and SYNTHIA and Cityscapes dataset is set as ResNet-101. For all datasets, the backbone is first pretrained on ImageNet and then fine-tuned. The \(f\) and \(\phi\) in RLGLC are all implemented by a convolutional neural network with 3 hidden layers and leaky ReLU activations. For the experimental results in all tables, the 'Avg' represents the average classification accuracy of all specific transfer tasks.

\begin{table*}[!ht]
	\caption{Accuracy (\%) and standard deviation on the VisDA-2017 dataset for unsupervised domain adaptation by using ResNet-101 as the backbone. The best results are highlighted in \textbf{bold}.}
            \centering
\resizebox{\linewidth}{!}{
			\begin{tabular}{l|cccccccccccc|c}
				\toprule
				Method & plane & bcybl & bus & car & horse & knife & mcyle & person & plant & sktbrd & train & truck & Avg \\
				\midrule
ResNet-101 & 55.1$\pm$0.2 & 53.3$\pm$0.1 & 61.9$\pm$0.1 & 59.1$\pm$0.2 & 80.6$\pm$0.2 & 17.9$\pm$0.2 & 79.7$\pm$0.1 & 31.2$\pm$0.2 & 81.0$\pm$0.1 & 26.5$\pm$0.1 & 73.5$\pm$0.2 & 8.5$\pm$0.2 & 52.4 \\
DAN & 87.1$\pm$0.1 & 63.0$\pm$0.2 & 76.5$\pm$0.1 & 42.0$\pm$0.1 & 90.3$\pm$0.1 & 42.9$\pm$0.2 & 85.9$\pm$0.1 & 53.1$\pm$0.1 & 49.7$\pm$0.2 & 36.3$\pm$0.1 & 85.8$\pm$0.2 & 20.7$\pm$0.2 & 61.1 \\
DANN & 81.9$\pm$0.1 & 77.7$\pm$0.2 & 82.8$\pm$0.1 & 44.3$\pm$0.1 & 81.2$\pm$0.3 & 29.5$\pm$0.1 & 65.1$\pm$0.1 & 28.6$\pm$0.1 & 51.9$\pm$0.2 & 54.6$\pm$0.1 & 82.8$\pm$0.2 & 7.8$\pm$0.1 & 57.4 \\
MCD & 87.0$\pm$0.1 & 60.9$\pm$0.1 & 83.7$\pm$0.3 & 64.0$\pm$0.1 & 88.9$\pm$0.3 & 79.6$\pm$0.1 & 84.7$\pm$0.1 & 76.9$\pm$0.3 & 88.6$\pm$0.1 & 40.3$\pm$0.2 & 83.0$\pm$0.1 & 25.8$\pm$0.1 & 71.9 \\
CDAN & 85.2$\pm$0.1 & 66.9$\pm$0.2 & 83.0$\pm$0.1 & 50.8$\pm$0.1 & 84.2$\pm$0.2 & 74.9$\pm$0.1 & 88.1$\pm$0.1 & 74.5$\pm$0.3 & 83.4$\pm$0.2 & 76.0$\pm$0.1 & 81.9$\pm$0.2 & 38.0$\pm$0.3 & 73.9 \\
BSP+DANN & 92.2$\pm$0.2 & 72.5$\pm$0.2 & 83.8$\pm$0.1 & 47.5$\pm$0.1 & 87.0$\pm$0.2 & 54.0$\pm$0.1 & 86.8$\pm$0.2 & 72.4$\pm$0.1 & 80.6$\pm$0.3 & 66.9$\pm$0.2 & 84.5$\pm$0.1 & 37.1$\pm$0.3 & 72.1 \\
BSP+CDAN & 92.4$\pm$0.2 & 61.0$\pm$0.2 & 81.0$\pm$0.1 & 57.5$\pm$0.2 & 89.0$\pm$0.1 & 80.6$\pm$0.1 & 90.1$\pm$0.1 & 77.0$\pm$0.1 & 84.2$\pm$0.2 & 77.9$\pm$0.1 & 82.1$\pm$0.1 & 38.4$\pm$0.1 & 75.9 \\
SWD & 90.8$\pm$0.1 & 82.5$\pm$0.1 & 81.7$\pm$0.3 & 70.5$\pm$0.2 & 91.7$\pm$0.3 & 69.5$\pm$0.1 & 86.3$\pm$0.1 & 77.5$\pm$0.2 & 87.4$\pm$0.1 & 63.6$\pm$0.3 & 85.6$\pm$0.1 & 29.2$\pm$0.1 & 76.4 \\
CAN & 97.0$\pm$0.2 & 87.2$\pm$0.1 & 82.5$\pm$0.1 & 74.3$\pm$0.2 & 97.8$\pm$0.1 & 96.2$\pm$0.3 & 90.8$\pm$0.1 & 80.7$\pm$0.2 & 96.6$\pm$0.1 & 96.3$\pm$0.2 & 87.5$\pm$0.1 & 59.9$\pm$0.1 & 87.2 \\
ACTIR & 91.9$\pm$0.3 & 83.7$\pm$0.2 & 82.9$\pm$0.2 & 70.3$\pm$0.3 & 94.7$\pm$0.1 & 71.2$\pm$0.2 & 88.9$\pm$0.2 & 78.9$\pm$0.1 & 86.9$\pm$0.2 & 76.3$\pm$0.3 & 88.0$\pm$0.1 & 56.9$\pm$0.3 & 80.9 \\
TCM & 91.4$\pm$0.2 & 80.3$\pm$0.1 & 81.1$\pm$0.1 & 72.6$\pm$0.2 & 90.9$\pm$0.1 & 66.7$\pm$0.3 & 85.6$\pm$0.1 & 76.1$\pm$0.2 & 86.3$\pm$0.1 & 63.9$\pm$0.2 & 83.1$\pm$0.1 & 31.7$\pm$0.1 & 75.8 \\
ERM & 92.4$\pm$0.3 & 84.9$\pm$0.2 & 84.7$\pm$0.3 & 70.3$\pm$0.2 & 93.6$\pm$0.2 & 79.1$\pm$0.1 & 91.3$\pm$0.3 & 74.6$\pm$0.2 & 85.9$\pm$0.2 & 74.3$\pm$0.3 & 88.3$\pm$0.3 & 41.9$\pm$0.1 & 80.1 \\
ICDA & 93.1$\pm$0.3 & 84.7$\pm$0.2 & 82.3$\pm$0.3 & 73.6$\pm$0.3 & 94.3$\pm$0.2 & 72.9$\pm$0.1 & 89.9$\pm$0.2 & 79.3$\pm$0.3 & 91.0$\pm$0.3 & 72.9$\pm$0.3 & 86.9$\pm$0.1 & 39.1$\pm$0.1 & 79.9 \\
iMSDA & 95.1$\pm$0.1 & 84.9$\pm$0.1 & 82.1$\pm$0.3 & 72.7$\pm$0.2 & 92.9$\pm$0.3 & 78.6$\pm$0.1 & 89.1$\pm$0.1 & 79.3$\pm$0.2 & 92.8$\pm$0.1 & 73.6$\pm$0.3 & 85.9$\pm$0.1 & 40.1$\pm$0.1 & 80.6 \\
UniOT & 92.8$\pm$0.2 & 85.7$\pm$0.2 & 83.2$\pm$0.1 & 72.4$\pm$0.3 & 90.9$\pm$0.2 & 73.2$\pm$0.3 & 87.1$\pm$0.2 & 76.1$\pm$0.1 & 86.1$\pm$0.2 & 69.7$\pm$0.2 & 81.5$\pm$0.2 & 34.6$\pm$0.3 & 77.8 \\
WDGRL & 93.5$\pm$0.2 & 84.1$\pm$0.2 & 83.0$\pm$0.2 & 70.1$\pm$0.2 & 89.2$\pm$0.2 & 92.0$\pm$0.3 & 89.1$\pm$0.1 & 77.2$\pm$0.3 & 91.4$\pm$0.3 & 89.1$\pm$0.2 & 85.0$\pm$0.2 & 52.6$\pm$0.1 & 83.0 \\
PPOT & 96.2$\pm$0.3 & 86.9$\pm$0.2 & 85.6$\pm$0.2 & 72.8$\pm$0.1 & 93.5$\pm$0.2 & 94.8$\pm$0.1 & 89.9$\pm$0.2 & 78.0$\pm$0.2 & 93.4$\pm$0.1 & 90.2$\pm$0.2 & 82.5$\pm$0.3 & 53.6$\pm$0.1 & 84.8 \\
PDA & 96.5$\pm$0.2 & 86.9$\pm$0.1 & 85.0$\pm$0.3 & 74.6$\pm$0.2 & 95.1$\pm$0.2 & 90.6$\pm$0.2 & 89.9$\pm$0.1 & 76.5$\pm$0.2 & 93.8$\pm$0.2 & 91.4$\pm$0.2 & 84.0$\pm$0.3 & 55.0$\pm$0.1 & 84.9 \\
SAMB-D & 95.2$\pm$0.3 & 85.6$\pm$0.2 & 84.9$\pm$0.2 & 75.1$\pm$0.3 & 94.8$\pm$0.2 & 93.9$\pm$0.2 & 91.0$\pm$0.2 & 77.8$\pm$0.2 & 94.9$\pm$0.2 & 95.8$\pm$0.4 & 90.0$\pm$0.1 & 57.2$\pm$0.2 & 86.4 \\
TCRL & 97.3$\pm$0.1 & 87.3$\pm$0.2 & 85.1$\pm$0.1 & 74.9$\pm$0.4 & 96.0$\pm$0.2 & 95.8$\pm$0.2 & 92.1$\pm$0.2 & 79.6$\pm$0.3 & 94.9$\pm$0.2 & 95.1$\pm$0.3 & 90.1$\pm$0.1 & 57.3$\pm$0.2 & 87.2 \\
\midrule
DANN + \textbf{LM} & 83.1$\pm$0.2 & 78.7$\pm$0.2 & 84.1$\pm$0.2 & 44.5$\pm$0.1 & 82.2$\pm$0.2 & 30.9$\pm$0.2 & 66.2$\pm$0.2 & 28.9$\pm$0.1 & 52.6$\pm$0.1 & 55.6$\pm$0.2 & 83.2$\pm$0.2 & 9.2$\pm$0.2 & 58.4 \\
SWD + \textbf{LM} & 91.9$\pm$0.2 & 83.6$\pm$0.2 & 82.9$\pm$0.2 & 72.7$\pm$0.1 & 93.1$\pm$0.2 & 69.7$\pm$0.2 & 87.1$\pm$0.3 & 79.6$\pm$0.1 & 88.5$\pm$0.2 & 64.7$\pm$0.2 & 86.1$\pm$0.2 & 31.1$\pm$0.2 & 77.6 \\
CAN + \textbf{LM} & 98.2$\pm$0.2 & 88.4$\pm$0.2 & 83.7$\pm$0.2 & 76.1$\pm$0.1 & 98.1$\pm$0.2 & 97.5$\pm$0.2 & 92.3$\pm$0.2 & \textbf{81.9$\pm$0.3} & 97.1$\pm$0.2 & 97.8$\pm$0.1 & 87.5$\pm$0.2 & \textbf{60.5$\pm$0.2} & 88.3 \\
\midrule
\textbf{RLGC} & 95.8$\pm$0.1 & 86.5$\pm$0.2 & 84.8$\pm$0.1 & 70.6$\pm$0.2 & 92.6$\pm$0.1 & 94.4$\pm$0.2 & 90.8$\pm$0.1 & 78.4$\pm$0.1 & 93.9$\pm$0.3 & 90.7$\pm$0.1 & 85.4$\pm$0.2 & 54.9$\pm$0.1 & 84.9 \\
\textbf{RLGC*} & 97.6$\pm$0.1 & 87.4$\pm$0.1 & 86.0$\pm$0.2 & 74.7$\pm$0.1 & 95.6$\pm$0.3 & 96.6$\pm$0.1 & 93.3$\pm$0.1 & 79.2$\pm$0.2 & 95.7$\pm$0.1 & 94.8$\pm$0.1 & 89.7$\pm$0.2 & 56.6$\pm$0.1 & 87.3 \\
\textbf{RLGLC} & \textbf{99.4$\pm$0.2} & \textbf{88.3$\pm$0.1} & \textbf{88.6$\pm$0.2} & \textbf{77.3$\pm$0.1} & \textbf{98.8$\pm$0.1} & \textbf{97.8$\pm$0.2} & \textbf{93.7$\pm$0.1} & 81.6$\pm$0.1 & \textbf{97.7$\pm$0.3} & \textbf{98.6$\pm$0.1} & \textbf{91.4$\pm$0.2} & 59.6$\pm$0.1 & \textbf{89.4} \\
				\bottomrule
			\end{tabular}
}
	
	\label{tab:3}
\end{table*}

\begin{table*}[t]
	\caption{Accuracy (\%) and standard deviation on the DomainNet dataset for unsupervised domain adaptation by using ResNet-50 as the backbone. The best results are highlighted in \textbf{bold}.}
	\centering
\resizebox{\linewidth}{!}{
		\begin{tabular}{l|cccccccccccc|c}
			\toprule
			Method & R$\to$C & R$\to$P & R$\to$S & C$\to$R & C$\to$P & C$\to$S & P$\to$R &P$\to$C & P$\to$S & S$\to$R & S$\to$C & S$\to$P & Avg \\
			\midrule
			ResNet-50 & 65.8$\pm$0.3 & 68.8$\pm$0.1 & 59.2$\pm$0.2 & 77.7$\pm$0.3 & 60.6$\pm$0.2 & 57.9$\pm$0.1 & 84.5$\pm$0.1 & 62.3$\pm$0.1 & 65.1$\pm$0.2 & 77.1$\pm$0.2 & 63.0$\pm$0.3 & 59.2$\pm$0.3 & 66.8 \\
MCD & 62.0$\pm$0.3 & 69.3$\pm$0.4 & 56.3$\pm$0.3 & 79.8$\pm$0.2 & 56.6$\pm$0.3 & 53.7$\pm$0.3 & 83.4$\pm$0.3 & 58.3$\pm$0.4 & 61.0$\pm$0.4 & 81.7$\pm$0.2 & 56.3$\pm$0.3 & 66.8$\pm$0.4 & 65.4 \\
SWD & 63.2$\pm$0.2 & 70.4$\pm$0.4 & 56.6$\pm$0.3 & 80.1$\pm$0.3 & 56.1$\pm$0.4 & 53.5$\pm$0.4 & 84.9$\pm$0.4 & 61.2$\pm$0.3 & 63.1$\pm$0.4 & 82.4$\pm$0.3 & 56.1$\pm$0.2 & 67.4$\pm$0.3 & 66.3 \\
DAN & 64.4$\pm$0.4 & 70.7$\pm$0.2 & 58.4$\pm$0.2 & 79.4$\pm$0.4 & 56.8$\pm$0.5 & 60.1$\pm$0.3 & 84.6$\pm$0.4 & 61.6$\pm$0.4 & 62.2$\pm$0.2 & 79.7$\pm$0.5 & 65.0$\pm$0.2 & 62.0$\pm$0.4 & 67.1 \\
JAN & 65.6$\pm$0.2 & 73.6$\pm$0.3 & 67.6$\pm$0.4 & 85.0$\pm$0.3 & 65.0$\pm$0.3 & 67.2$\pm$0.3 & 87.1$\pm$0.4 & 67.9$\pm$0.5 & 66.1$\pm$0.2 & 84.5$\pm$0.3 & 72.8$\pm$0.5 & 67.5$\pm$0.4 & 72.5 \\
BSP+DANN & 67.3$\pm$0.5 & 73.5$\pm$0.3 & 69.3$\pm$0.3 & 86.5$\pm$0.4 & 67.5$\pm$0.3 & 70.9$\pm$0.5 & 86.8$\pm$0.3 & 70.3$\pm$0.3 & 68.8$\pm$0.4 & 84.3$\pm$0.3 & 72.4$\pm$0.5 & 71.5$\pm$0.2 & 74.1 \\
DANN & 63.4$\pm$0.3 & 73.6$\pm$0.4 & 72.6$\pm$0.5 & 86.5$\pm$0.4 & 65.7$\pm$0.3 & 70.6$\pm$0.4 & 86.9$\pm$0.4 & 73.2$\pm$0.3 & 70.2$\pm$0.4 & 85.7$\pm$0.3 & 75.2$\pm$0.5 & 70.0$\pm$0.4 & 74.5 \\
MDD & 63.5$\pm$0.2 & 71.4$\pm$0.5 & 73.7$\pm$0.5 & 85.7$\pm$0.3 & 67.8$\pm$0.3 & 71.5$\pm$0.3 & 86.2$\pm$0.4 & 73.7$\pm$0.5 & 72.1$\pm$0.4 & 88.3$\pm$0.4 & 76.4$\pm$0.5 & 72.1$\pm$0.4 & 75.2 \\
ACTIR & 78.9$\pm$0.3 & 76.1$\pm$0.4 & 71.4$\pm$0.5 & 88.8$\pm$0.4 & 70.8$\pm$0.2 & 73.1$\pm$0.4 & 89.1$\pm$0.3 & 72.9$\pm$0.5 & 73.4$\pm$0.4 & 88.2$\pm$0.5 & 80.6$\pm$0.3 & 75.3$\pm$0.5 & 78.2 \\
TCM & 79.8$\pm$0.4 & 74.1$\pm$0.5 & 72.9$\pm$0.2 & 89.2$\pm$0.3 & 71.5$\pm$0.4 & 72.7$\pm$0.3 & 88.2$\pm$0.2 & 74.8$\pm$0.3 & 74.1$\pm$0.5 & 88.1$\pm$0.5 & 76.8$\pm$0.5 & 73.7$\pm$0.3 & 78.0 \\
ICDA & 74.1$\pm$0.5 & 72.7$\pm$0.3 & 69.8$\pm$0.4 & 82.7$\pm$0.4 & 66.6$\pm$0.5 & 71.4$\pm$0.4 & 84.1$\pm$0.5 & 61.3$\pm$0.4 & 68.1$\pm$0.4 & 85.7$\pm$0.3 & 66.4$\pm$0.4 & 70.1$\pm$0.3 & 72.8 \\
iMSDA & 71.9$\pm$0.4 & 72.1$\pm$0.2 & 72.5$\pm$0.3 & 84.1$\pm$0.2 & 62.8$\pm$0.4 & 70.8$\pm$0.3 & 86.7$\pm$0.3 & 64.9$\pm$0.3 & 66.9$\pm$0.4 & 87.9$\pm$0.4 & 69.1$\pm$0.4 & 69.7$\pm$0.5 & 73.3 \\
UniOT & 68.9$\pm$0.5 & 70.1$\pm$0.4 & 64.8$\pm$0.5 & 82.2$\pm$0.4 & 56.0$\pm$0.3 & 61.7$\pm$0.5 & 84.7$\pm$0.5 & 62.8$\pm$0.5 & 65.7$\pm$0.3 & 86.0$\pm$0.4 & 62.7$\pm$0.5 & 70.1$\pm$0.4 & 69.6 \\
PPOT & 83.2$\pm$0.3 & 75.6$\pm$0.2 & 74.8$\pm$0.2 & 91.6$\pm$0.3 & 73.0$\pm$0.4 & 77.5$\pm$0.5 & 89.6$\pm$0.4 & 83.5$\pm$0.1 & 74.3$\pm$0.2 & 89.6$\pm$0.2 & 84.5$\pm$0.2 & 73.8$\pm$0.1 & 80.9 \\
PDA & 84.0$\pm$0.2 & 75.6$\pm$0.2 & 74.9$\pm$0.2 & 90.2$\pm$0.1 & 74.5$\pm$0.2 & 78.6$\pm$0.2 & 89.2$\pm$0.2 & 81.6$\pm$0.1 & 74.9$\pm$0.2 & 90.2$\pm$0.1 & 83.4$\pm$0.2 & 74.1$\pm$0.3 & 80.9 \\
SAMB-D & 84.2$\pm$0.1 & 74.9$\pm$0.3 & 75.2$\pm$0.3 & 93.8$\pm$0.1 & 75.0$\pm$0.2 & 80.2$\pm$0.4 & 91.0$\pm$0.2 & 84.1$\pm$0.3 & 75.8$\pm$0.3 & 92.1$\pm$0.1 & 81.0$\pm$0.4 & 73.0$\pm$0.2 & 81.7 \\
TCRL & 85.0$\pm$0.1 & 75.6$\pm$0.2 & 74.0$\pm$0.4 & 92.1$\pm$0.3 & 73.9$\pm$0.3 & 81.0$\pm$0.2 & 92.0$\pm$0.1 & 83.9$\pm$0.3 & 74.9$\pm$0.2 & 92.3$\pm$0.4 & 84.0$\pm$0.1 & 74.6$\pm$0.2 & 81.9 \\
\midrule
DANN + \textbf{LM} & 64.5$\pm$0.2 & 75.2$\pm$0.2 & 73.9$\pm$0.3 & 88.1$\pm$0.3 & 66.1$\pm$0.2 & 72.1$\pm$0.5 & 87.9$\pm$0.3 & 74.5$\pm$0.2 & 71.4$\pm$0.3 & 86.7$\pm$0.2 & 77.1$\pm$0.4 & 71.2$\pm$0.3 & 75.7 \\
SWD + \textbf{LM} & 64.5$\pm$0.2 & 71.6$\pm$0.3 & 57.7$\pm$0.4 & 82.3$\pm$0.2 & 57.3$\pm$0.3 & 53.6$\pm$0.3 & 85.5$\pm$0.3 & 62.6$\pm$0.3 & 64.5$\pm$0.3 & 83.6$\pm$0.2 & 57.8$\pm$0.3 & 68.6$\pm$0.2 & 67.5 \\
MDD + \textbf{LM} & 64.6$\pm$0.2 & 72.6$\pm$0.3 & 74.9$\pm$0.2 & 86.8$\pm$0.3 & 68.9$\pm$0.2 & 72.9$\pm$0.2 & 87.7$\pm$0.6 & 74.8$\pm$0.3 & 74.1$\pm$0.2 & 89.8$\pm$0.3 & 77.9$\pm$0.2 & 73.7$\pm$0.4 & 76.6 \\
\midrule
\textbf{RLGC} & 83.9$\pm$0.3 & 73.9$\pm$0.2 & 72.9$\pm$0.3 & 91.9$\pm$0.4 & 75.1$\pm$0.4 & 78.7$\pm$0.3 & 90.4$\pm$0.2 & 83.1$\pm$0.2 & 74.9$\pm$0.3 & 90.8$\pm$0.1 & 82.8$\pm$0.2 & 73.3$\pm$0.3 & 81.0 \\
\textbf{RLGC*} & 85.1$\pm$0.2 & 76.4$\pm$0.3 & 74.8$\pm$0.4 & 93.1$\pm$0.3 & 76.2$\pm$0.3 & 80.5$\pm$0.4 & 92.1$\pm$0.4 & 84.8$\pm$0.3 & 76.7$\pm$0.1 & 92.5$\pm$0.3 & 84.5$\pm$0.3 & 75.1$\pm$0.2 & 82.7 \\
\textbf{RLGLC} & \textbf{86.7$\pm$0.2} & \textbf{78.8$\pm$0.2} & \textbf{77.2$\pm$0.3} & \textbf{95.0$\pm$0.2} & \textbf{77.7$\pm$0.3} & \textbf{82.4$\pm$0.3} & \textbf{94.2$\pm$0.2} & \textbf{85.8$\pm$0.2} & \textbf{77.0$\pm$0.2} & \textbf{93.9$\pm$0.2} & \textbf{85.3$\pm$0.1} & \textbf{77.0$\pm$0.2} & \textbf{84.3} \\
			\bottomrule
		\end{tabular}
}
	\label{tab:domatne}
\end{table*}

\subsection{Results and discussions}

The classification results and standard deviations on the Office-31 dataset are reported in Table \ref{tab:1}. The average classification accuracy of our proposed RLGLC is the highest among all compared methods. Also, RLGLC achieves better results than all compared methods on 5 specific transfer tasks, which are the most numerous. It is worth noting that RLGLC obtains the best results on two hard specific transfer tasks: A $\to$ D and D $\to$ A. Meanwhile, the proposed method achieves similar or even better results than the more powerful ViT-B-based baseline, i.e., TransVQA, based only on the ResNet backbone, further demonstrating its effectiveness. The classification results and standard deviations on the Office-Home dataset are shown in Table \ref{tab:2}. RLGLC obtains the highest average result among the compared methods. The specific transfer tasks of this dataset are quite challenging and RLGLC wins 8 of the 12 tasks, which is also the most numerous of all methods. Table \ref{tab:3} 
shows the classification results and standard deviations on the VisDa-2017 dataset. From the results, we can observe that RLGLC gains a improvement, i.e., the average result of RLGLC is 1.1\% higher than the second-ranked method CAN. For specific transfer tasks, RLGLC wins the most numerous (10 of 12) tasks on this dataset. Table \ref{tab:domatne} shows the the classification results and standard deviations on the DomainNet dataset. We can obtain that the average classification accuracy of the proposed RLGLC is the highest. As for the specific transfer tasks, we can observe that RLGLC wins all of the tasks, thus showing its superiority. We further compare RLGLC with previous methods on the Digits dataset, the results and standard deviations are shown in Table \ref{tab:4}. Compared with the Office-31 dataset, the size of this dataset is much larger. Again, the classification results of RLGLC exceed all compared approaches in terms of average accuracy. Also, RLGLC achieves the best performance on all of specific transfer tasks. Therefore, we can conclude that our proposed RLGLC is effective and contains compressed label-relevant information.

\subsection{Statistical test}

We further analyze the experimental results through the Friedman test \cite{friedman1937use}. The Friedman test is based on the
average accuracy ranks of different methods on the same dataset. Then, we compute the following:
\begin{equation}\label{dsdsd}
    \mathcal{X}_F^2 = \frac{{12{n_t}}}{{{m_t}\left( {{m_t} + 1} \right)}}\left[ {\sum\limits_j {R_j^2}  - \frac{{{m_t}{{\left( {{m_t} + 1} \right)}^2}}}{4}} \right],
\end{equation}
where $m_t$ represents the number of compared methods, $n_t$ represents the number of tasks, ${R_j} = \frac{1}{{{n_t}}}\sum\limits_i {r_i^j}$, and ${r_i^j}$ denotes the accuracy rank of the $j$-th method on the $i$-th task. Finally, the Friedman statistic is calculated as following:
\begin{equation}\label{dssddsd}
{\mathbbm{F}_F} = \frac{{\left( {{n_t} - 1} \right)\mathcal{X}_F^2}}{{{n_t}\left( {{m_t} - 1} \right) - \mathcal{X}_F^2}},
\end{equation}
where ${\mathbbm{F}_F}$ follows an $F$-distribution with $m_t - 1$ and $(m_t - 1)(n_t - 1)$ degrees of freedom.

Based on Equation (\ref{dsdsd}) and Equation (\ref{dssddsd}), we can obtain the accuracy rankings of different methods under different tasks (for more details, please refer to Table \ref{tab:a1}, Table \ref{tab:a2}, Table \ref{tab:a3}, Table \ref{tab:domatnade}, and Table \ref{tab:a4} in Appendix B). Then, for Table \ref{tab:1}, we obtain $\mathcal{X}_F^2  \approx  83.58$ and ${\mathbbm{F}_F} \approx  3.97$ with $(30, 180)$ degrees of freedom for the $F$-distribution. From the table of critical values for $F$-distribution, we obtain that the critical values of $F(30, 180)$ are approximately 1.46, 1.63, and 1.79 for the significance levels of ${\alpha _F} = 0.1$, ${\alpha _F} = 0.05$, and ${\alpha _F} = 0.025$, respectively. For Table \ref{tab:2}, we obtain $\mathcal{X}_F^2  \approx  254.62$ and ${\mathbbm{F}_F} \approx 31.7$ with $(27, 324)$ degrees of freedom for the $F$-distribution. From the table of critical values for the $F$-distribution, we obtain that the critical values of $F(27, 324)$ are approximately 1.46, 1.63, and 1.79 for the significance levels of ${\alpha _F} = 0.1$, ${\alpha _F} = 0.05$, and ${\alpha _F} = 0.025$, respectively. For Table \ref{tab:3}, we obtain $\mathcal{X}_F^2  =  244.69$ and ${\mathbbm{F}_F} \approx   36.56$ with $(25, 300)$ degrees of freedom for the $F$-distribution. From the table of critical values for the $F$-distribution, we obtain that the critical values of $F(25, 300)$ are approximately 1.52, 1.71, and 1.91 for the significance levels of ${\alpha _F} = 0.1$, ${\alpha _F} = 0.05$, and ${\alpha _F} = 0.025$, respectively. For Table \ref{tab:domatne}, we obtain $\mathcal{X}_F^2  \approx  249.94$ and ${\mathbbm{F}_F} \approx 83.17$ with $(22, 264)$ degrees of freedom for the $F$-distribution. From the table of critical values for the $F$-distribution, we obtain that the critical values of $F(22, 264)$ are approximately 1.52, 1.71, and 1.91 for the significance levels of ${\alpha _F} = 0.1$, ${\alpha _F} = 0.05$, and ${\alpha _F} = 0.025$, respectively. For Table \ref{tab:4}, we obtain $\mathcal{X}_F^2  \approx  69.77$ and ${\mathbbm{F}_F} \approx 11.48$ with $(22, 66)$ degrees of freedom for the $F$-distribution. From the table of critical values for the $F$-distribution, we obtain that the critical values of $F(22, 66)$ are approximately 1.56, 1.77, and 1.98 for the significance levels of ${\alpha _F} = 0.1$, ${\alpha _F} = 0.05$, and ${\alpha _F} = 0.025$, respectively. These results illustrate that there is a significant difference between these compared methods since all the real values of ${\alpha _F}$ are much larger than the critical values. Because the average ranking of the proposed RLGLC is the lowest, RLGLC is shown to be the most effective.

\subsection{Generality}

To verify the generality of the proposed RLGLC, we have conducted experiments on two additional and distinct downstream tasks including: object detection and semantic segmentation. Specifically, the datasets used in our experiments for the object detection task are PASCAL VOC \cite{everingham2010pascal}, Clipart \cite{inoue2018cross}, and Watercolor \cite{inoue2018cross}. For the semantic segmentation task, our experiments are based on the GTA5 dataset \cite{richter2016playing} and the Cityscapes dataset \cite{cordts2016cityscapes}, and conduct more complex transfer learning comparison experiments on SYNTHIA \cite{ros2016synthia} and Cityscapes dataset \cite{cordts2016cityscapes} following \cite{dong2021and}. In these two tasks, we compare our proposed RLGLC with state-of-the-art unsupervised domain adaptation methods including: CaCo \cite{huang2022category}, MGA \cite{zhou2022multi}, KL \cite{nguyen2021kl}, and GRDA \cite{xu2022graph}. The transfer learning results for the semantic segmentation task are presented in Table \ref{tab:44} and Table \ref{tab:add_tpami_6}, while Table \ref{tab:45} summarizes the outcomes for object detection tasks. For the semantic segmentation task, the experimental settings follow those evaluated in CaCo (Table \ref{tab:44} and Table \ref{tab:add_tpami_6}). Similarly, for the object detection tasks (Table \ref{tab:45}), we adopt the experimental settings used in MGA. For the final average results, whether for detection tasks or segmentation tasks, our proposed RLGLC consistently attains the best performance. For instance, in Table \ref{tab:add_tpami_6}, RLGLC surpasses all compared baselines by at least 1.5\%. Regarding individual transfer tasks, Table \ref{tab:add_tpami_6} shows that RLGLC achieves the best performance on 18 out of 19 tasks, and Table \ref{tab:45} indicates that RLGLC maintains an advantage across every transfer task. These findings highlight the effectiveness of RLGLC in diverse downstream tasks and underscore the generalizability of the proposed theory to various tasks and methods.


\begin{figure}[th]
    \centering
    \includegraphics[width=0.9\linewidth]{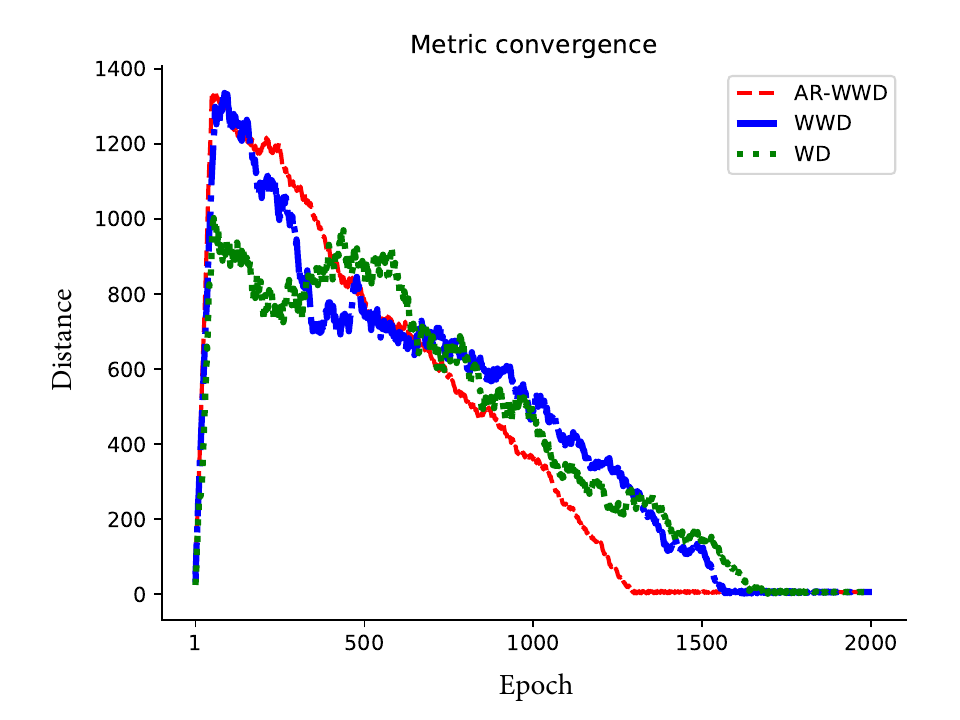}
    \caption{Distances on ${\rm{P}} \to C$ task of Office-home dataset.}
    \label{fig:g}
\end{figure}

\begin{table}[t]
	\caption{Accuracy (\%) and standard deviation on the Digits dataset for unsupervised domain adaptation by using ResNet-50 as the backbone. The best results are highlighted in \textbf{bold}.}
		\centering
			\begin{tabular}{l|ccc|c}
                \toprule
				Method & M$\to$U & U$\to$M & S$\to$M & Avg \\
				\midrule
				DANN & 90.4 $\pm$ 0.1 & 94.7 $\pm$ 0.2 & 84.2 $\pm$ 0.1 & 89.8 \\
				ADDA & 89.4 $\pm$ 0.2  & 90.1 $\pm$ 0.3  & 86.3 $\pm$ 0.1  & 88.6 \\
				UNIT & 96.0 $\pm$ 0.1  & 93.6 $\pm$ 0.2  & 90.5 $\pm$ 0.2  & 93.4 \\
				CyCADA & 95.6 $\pm$ 0.3  & 96.5 $\pm$ 0.1  & 90.4 $\pm$ 0.1  & 94.2 \\
				CDAN & 93.9 $\pm$ 0.2  & 96.9 $\pm$ 0.2 & 88.5 $\pm$ 0.1 & 93.1 \\
				CDAN+E & 95.6 $\pm$ 0.1 & 98.0 $\pm$ 0.2 & 89.2 $\pm$ 0.2 & 94.3 \\
				BSP+CDAN & 95.0 $\pm$ 0.2 & 98.1 $\pm$ 0.1 & 92.1 $\pm$ 0.1 & 95.1 \\
				ETD & 96.4 $\pm$ 0.1 & 96.3 $\pm$ 0.1 & 97.9 $\pm$ 0.2 & 96.9 \\
				
				ACTIR & 95.7 $\pm$ 0.2 & 97.1 $\pm$ 0.2 & 94.6 $\pm$ 0.2 & 95.8 \\
				
				TCM & 96.1 $\pm$ 0.2 & 96.7 $\pm$ 0.2 & 95.2 $\pm$ 0.3 & 96.0 \\
				
				ERM & 96.8 $\pm$ 0.1 & 96.7 $\pm$ 0.1 & 94.9 $\pm$ 0.3 & 96.1 \\
				
				ICDA & 94.9 $\pm$ 0.2 & 95.6 $\pm$ 0.1 & 95.1 $\pm$ 0.1 & 95.2 \\
				
				iMSDA & 95.7 $\pm$ 0.2 & 97.6 $\pm$ 0.2 & 93.4 $\pm$ 0.2 & 95.6 \\
				
				UniOT & 94.4 $\pm$ 0.1 & 97.7 $\pm$ 0.1 & 94.2 $\pm$ 0.1 & 95.4 \\
PPOT & 96.4 $\pm$ 0.2 & 98.3 $\pm$ 0.3 & 96.9 $\pm$ 0.1 & 97.1 \\
SSRT+GH++ & 96.4 $\pm$ 0.1 & 98.0 $\pm$ 0.3 & 96.2 $\pm$ 0.2 & 96.9 \\
PDA & 96.0 $\pm$ 0.1 & 97.6 $\pm$ 0.2 & 95.0 $\pm$ 0.1 & 96.2 \\
CPH & 97.8 $\pm$ 0.1 & 98.9 $\pm$ 0.2 & 97.4 $\pm$ 0.2 & 97.9 \\
SAMB-D & 97.2 $\pm$ 0.1 & 98.4 $\pm$ 0.2 & 97.0 $\pm$ 0.3 & 97.5 \\
TCRL & 97.5 $\pm$ 0.2 & 98.7 $\pm$ 0.4 & 97.9 $\pm$ 0.3 & 98.0 \\
				
				\midrule
				\textbf{RLGC} & 97.1 $\pm$ 0.1 & 98.7 $\pm$ 0.2 & 95.9 $\pm$ 0.2 & 97.2 \\
				\textbf{RLGC*} & 97.9 $\pm$ 0.2 & 98.9 $\pm$ 0.1 & 96.8 $\pm$ 0.1 & 97.9 \\
				\textbf{RLGLC} & \textbf{98.1 $\pm$ 0.2} & \textbf{99.6 $\pm$ 0.1} & \textbf{98.6 $\pm$ 0.2} & \textbf{98.8} \\
				\bottomrule
			\end{tabular}
	
	\label{tab:4}
\end{table}

\begin{figure*}[ht]
    \centering
    \begin{minipage}{0.33\textwidth}
        \centering
        \includegraphics[width=\linewidth]{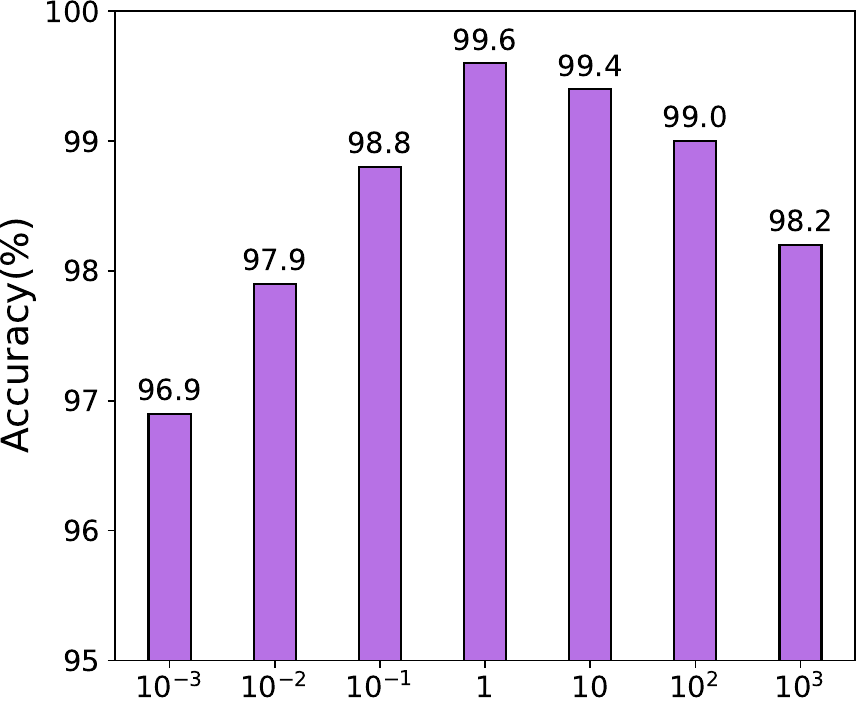}
        \subcaption{The influence of $\alpha$}
        \label{fig: sub_figure1}
    \end{minipage}\hfill
    \begin{minipage}{0.33\textwidth}
        \centering
        \includegraphics[width=\linewidth]{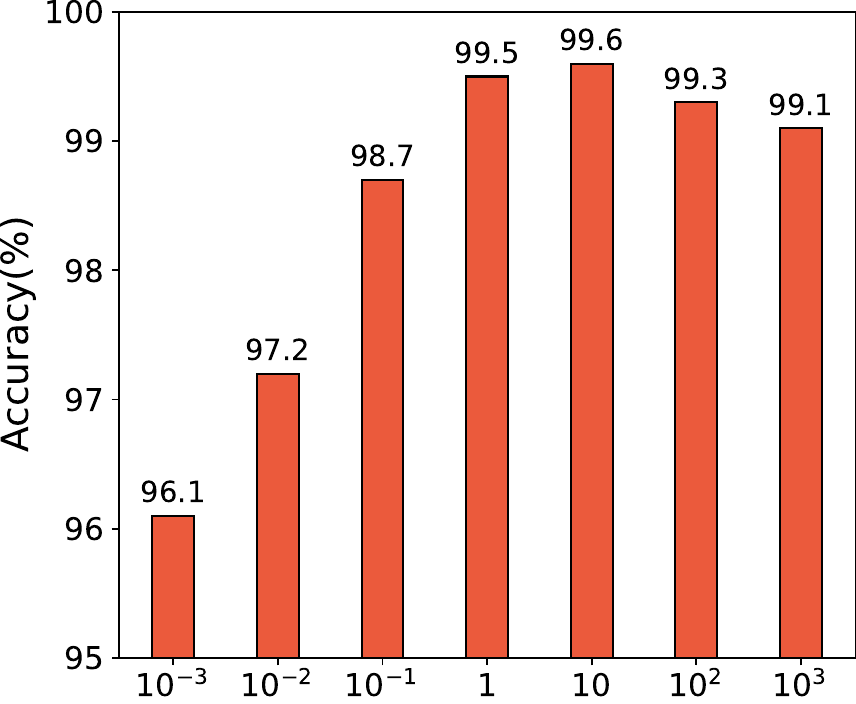}
        \subcaption{The influence of $\lambda$}
        \label{fig: sub_figure2}
    \end{minipage}\hfill
    \begin{minipage}{0.33\textwidth}
        \centering
        \includegraphics[width=\linewidth]{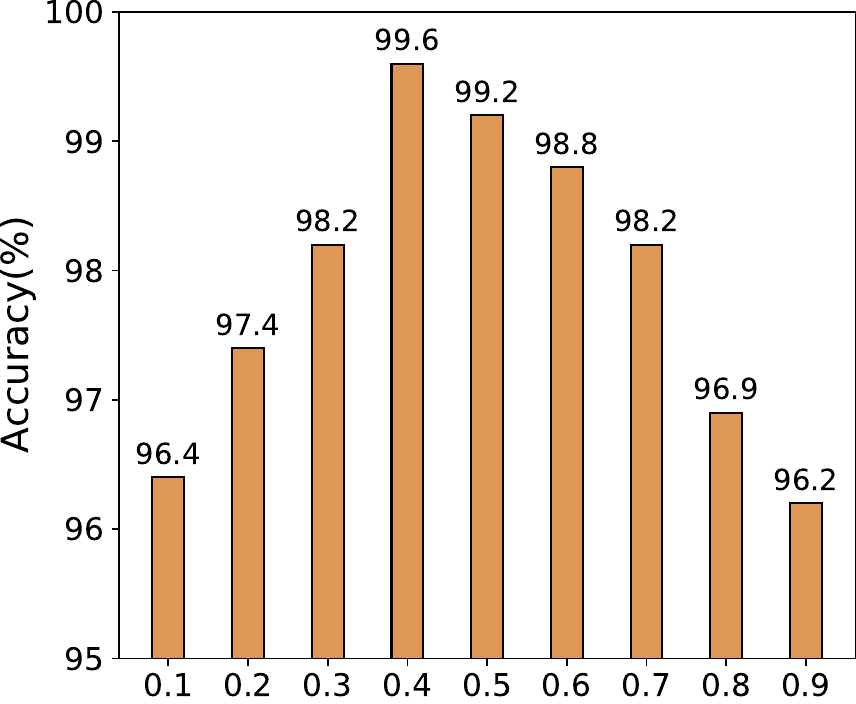}
        \subcaption{The influence of $\beta$}
        \label{fig: sub_figure3}
    \end{minipage}
    \caption{Ablation study on the influence of hyperparameters. We report the results for three hyperparameters, e.g., $\alpha$, $\lambda$, and $\beta$.}
    \label{fig: fig1}
\end{figure*}

\begin{table*}[!ht]
	
	\caption{Results on semantic segmentation task GTA5 $\to$ Cityscapes for unsupervised domain adaptation by using DeepLab-V2 \cite{chen2017deeplab} with ResNet101 as the segmentation backbone. The best results are highlighted in \textbf{bold}.}
    \centering
	\resizebox{\textwidth}{!}{
		
			\begin{tabular}{l|ccccccccccccccccccc|c}
				\toprule
				Method & Road & SW & Build & Wall & Fence & Pole & TL & TS & Veg & Terrain & Sky & PR & Rider & Car & Truck & Bus & Train & Motor & Bike & mIoU \\
\midrule
LtA      & 86.5 & 36.0 & 79.9 & 23.4 & 23.3 & 23.9 & 35.2 & 14.8 & 83.4 & 33.3 & 75.6 & 35.4 & 3.9  & 30.1 & 28.1 & 42.4 & 25.0 & 30.5 & 28.1 & 42.4 \\
        ADV      & 89.4 & 33.1 & 81.0 & 26.6 & 26.8 & 27.2 & 33.5 & 24.7 & 83.9 & 36.7 & 78.6 & 44.5 & 1.7  & 31.6 & 32.5 & 45.5 & 26.5 & 32.6 & 33.0 & 45.5 \\
        SWLS     & 92.7 & 48.0 & 78.8 & 25.7 & 27.2 & 36.0 & 42.2 & 45.3 & 80.6 & 31.6 & 66.0 & 45.6 & 16.8 & 34.7 & 47.2 & 47.2 & 28.2 & 34.8 & 36.5 & 47.2 \\
        SSF      & 90.3 & 38.9 & 81.7 & 24.8 & 22.9 & 30.5 & 37.0 & 21.2 & 84.8 & 38.5 & 76.9 & 38.1 & 5.9  & 28.6 & 36.9 & 45.4 & 25.7 & 31.0 & 36.9 & 45.4 \\
        PyCDA    & 90.5 & 34.6 & 84.6 & 32.4 & 28.7 & 34.6 & 34.5 & 21.5 & 85.6 & 27.9 & 85.6 & 46.1 & 18.0 & 22.9 & 39.9 & 48.6 & 31.0 & 32.2 & 39.9 & 48.6 \\
        CrCDA    & 92.4 & 55.3 & 83.5 & 31.2 & 29.1 & 32.5 & 33.2 & 35.6 & 83.5 & 34.6 & 84.4 & 46.1 & 2.1  & 31.1 & 32.7 & 48.6 & 29.4 & 34.0 & 41.2 & 48.6 \\
        CSCL     & 89.6 & 50.4 & 81.0 & 35.6 & 26.9 & 31.1 & 37.3 & 35.1 & 83.5 & 40.1 & 85.4 & 47.3 & 0.5  & 34.5 & 33.7 & 48.6 & 31.6 & 33.8 & 39.7 & 48.6 \\
        LSE      & 90.2 & 40.2 & 81.0 & 31.9 & 26.4 & 32.6 & 38.7 & 37.5 & 81.0 & 34.2 & 84.6 & 45.9 & 6.7  & 29.1 & 30.6 & 47.5 & 28.3 & 32.1 & 38.2 & 47.5 \\
        FADA     & 92.5 & 47.5 & 85.1 & 37.6 & \textbf{32.8} & 33.4 & 33.8 & 18.4 & 85.3 & 37.3 & \textbf{87.5} & 49.2 & 1.6  & 34.9 & 39.5 & 49.2 & 31.6 & 35.1 & 42.3 & 49.2 \\
        TPLD     & 83.2 & 46.3 & 74.9 & 29.8 & 21.3 & 33.1 & 36.0 & 24.2 & 86.7 & 43.2 & 87.1 & 36.9 & 0.0  & 29.7 & 40.0 & 44.7 & 30.2 & 33.2 & 40.6 & 44.7 \\
        SS-UDA   & 90.6 & 37.1 & 82.1 & 30.1 & 19.1 & 29.5 & 32.4 & 20.6 & 85.7 & 40.5 & 79.7 & 48.3 & 0.0  & 30.2 & 35.8 & 46.3 & 28.7 & 34.2 & 38.7 & 46.3 \\
        DTST     & 90.6 & 44.7 & 84.8 & 34.3 & 28.7 & 31.6 & 35.0 & 37.6 & 84.7 & 43.3 & 85.3 & 48.5 & 1.9  & 30.4 & 39.0 & 49.2 & 31.4 & 35.9 & 40.3 & 49.2 \\
        LTIR     & 92.9 & 55.0 & 85.3 & 34.2 & 31.1 & 34.9 & 40.7 & 31.4 & 85.2 & 40.1 & 87.1 & 42.6 & 0.3  & 36.4 & 46.1 & 50.2 & 33.0 & 38.5 & 43.2 & 50.2 \\
        PIT      & 87.5 & 43.4 & 78.8 & 31.2 & 30.2 & 36.3 & 39.9 & 42.0 & 79.2 & 37.1 & 79.3 & 46.0 & 25.7 & 23.5 & 49.9 & 50.6 & 34.2 & \textbf{41.6} & \textbf{44.9} & 50.6 \\
        ELS-DA   & 89.4 & 50.1 & 83.9 & 35.9 & 27.0 & 32.4 & 38.6 & 37.5 & 84.5 & 39.6 & 85.7 & 50.4 & 0.3  & 33.6 & 32.1 & 49.2 & 32.4 & 37.3 & 42.8 & 49.2 \\
				CaCo & 91.9 & 54.3 & 82.7 & 31.7 & 25.0 & 38.1 & \textbf{46.7} & 39.2 & 82.6 & 39.7 & 76.2 & 63.5 & 23.6 & \textbf{85.1} & 38.6 & 47.8 & 10.3 & 23.4 & 35.1 & 49.2 \\
				MGA & 92.1 & 53.3 & 83.4 & 32.1 & 25.2 & 34.3 & 42.9 & 38.2 & 85.6 & 41.3 & 79.1 & 61.2 & 27.1 & 83.2 & 38.1 & 48.9 & 11.7 & 25.2 & 34.7 & 50.1  \\
				KL & 90.1 & 52.7 & 84.2 & 33.6 & 27.0 & 33.9 & 43.5 & 32.8 & 84.7 & 42.8 & 81.1 & 65.1 & 25.2 & 82.9 & 36.7 & 49.9 & 12.3 & 22.4 & 38.9 & 49.5 \\
				GRDA & 93.1 & 51.3 & 85.2 & 33.6 & 27.2 & \textbf{38.9} & 43.1 & 40.1 & 84.1 & 36.7 & 78.2 & 66.5 & 25.6 & 82.1 & 35.6 & 50.7 & 14.1 & 25.2 & 36.1 & 49.9 \\
                PPOT & 94.2 & 54.2 & 83.2 & 34.5 & 24.0 & 33.2 & 50.0 & 42.1 & 83.5 & 41.2 & 73.9 & 62.8 & 28.0 & 80.1 & 35.9 & 50.2 & 14.1 & 24.7 & 35.6 & 44.5 \\
                PDA & 93.0 & 50.2 & 85.1 & 33.6 & 25.1 & 36.0 & 41.2 & 40.2 & 82.1 & 42.5 & 77.6 & 62.5 & 27.6 & 82.0 & 34.8 & 50.3 & 11.9 & 23.2 & 33.5 & 49.1 \\
                CPH & 94.5 & 50.2 & \textbf{86.0} & 32.0 & 25.6 & 38.8 & 43.2 & 40.1 & 85.0 & 42.8 & 36.7 & \textbf{79.1} & 65.9 & 28.1 & \textbf{84.9} & 36.2 & \textbf{50.1} & 14.5 & 25.9 & 50.0 \\
			    \midrule
                \textbf{RLGLC} & \textbf{96.4} & \textbf{56.7} & 83.1 & \textbf{36.7} & 23.6 & 37.6 & 43.2 & \textbf{44.9} & \textbf{87.9} & \textbf{45.7} & 77.1 & 63.8 & \textbf{30.2} & 83.2 & 41.3 & \textbf{53.4} & 14.1 & 26.7 & 33.7 & \textbf{51.5} \\
				\bottomrule
			\end{tabular}
	}
	\label{tab:44}
\end{table*}

\begin{table*}[!ht]
    \caption{Results on semantic segmentation task SYNTHIA $\to$ Cityscapes for unsupervised domain adaptation by using DeepLab-V2 \cite{chen2017deeplab} with ResNet101 as the segmentation backbone. The best results are highlighted in \textbf{bold}.}
    \centering
    \label{tab:add_tpami_6}
    \resizebox{\textwidth}{!}{
    \begin{tabular}{l|ccccccccccccccccccc|c}
        \toprule
        Method & Road & SW & Build & Wall & Fence & Pole & TL & TS & Veg & Terrain & Sky & PR & Rider & Car & Truck & Bus & Train & Motor & Bike & mIoU \\
        \midrule
        DCAN    & 81.5 & 33.4 & 72.4 & 7.9 & 0.2 & 20.0 & 8.6 & 10.5 & 71.0 & 25.0 & 68.7 & 51.5 & 18.7 & 75.3 & 42.5 & 22.7 & 12.8 & 15.5 & 28.1 & 36.5\\
        CBST    & 53.6 & 23.7 & 75.0 & 12.5 & 0.3 & 36.4 & 23.5 & 26.3 & 84.8 & 31.3 & 74.7 & 67.2 & 17.5 & 84.5 & 45.6 & 28.4 & 15.2 & 21.6 & 28.1 & 42.5\\
        ADV     & 85.6 & 42.2 & 79.7 & 8.7 & 0.4 & 25.9 & 5.4 & 8.1 & 80.4 & 29.5 & 84.1 & 57.9 & 23.8 & 73.3 & 48.1 & 36.4 & 14.2 & 24.7 & 33.0 & 41.2\\
        SWLS    & 68.4 & 30.1 & 74.2 & 21.5 & 0.4 & 29.2 & 29.3 & 25.1 & 81.5 & 27.2 & 63.1 & 63.1 & 16.4 & 75.6 & 13.5 & 44.7 & 26.1 & 19.3 & 51.9 & 36.5 \\
        MSL     & 82.9 & 40.7 & 80.3 & 10.2 & 0.8 & 25.8 & 12.8 & 18.2 & 82.5 & 34.1 & 53.1 & 53.1 & 18.0 & 79.0 & 31.4 & 48.9 & 10.4 & 21.5 & 35.6 & 41.4 \\
        PyCDA   & 75.5 & 30.9 & 83.3 & 20.8 & 0.7 & 32.7 & 27.3 & 33.5 & 85.0 & 36.5 & 64.1 & 64.1 & 25.4 & 85.0 & 45.2 & 51.2 & 32.0 & 22.4 & 32.1 & 46.7 \\
        CrCDA   & 86.2 & 44.9 & 79.5 & 8.3 & 0.7 & 27.8 & 9.4 & 11.8 & 78.6 & 31.7 & 57.2 & 57.2 & 26.1 & 76.8 & 39.9 & 49.3 & 21.5 & 25.3 & 32.1 & 42.9 \\
        CSCL    & 80.2 & 41.1 & 78.9 & 23.6 & 0.6 & 31.0 & 27.1 & 29.5 & 82.5 & 29.8 & 62.1 & 62.1 & 26.8 & 81.5 & 37.2 & 50.7 & 27.3 & 23.5 & 42.9 & 47.2 \\
        LSE     & 82.9 & 43.1 & 78.1 & 9.3 & 0.6 & 28.2 & 9.1 & 14.4 & 77.0 & 28.1 & 58.1 & 58.1 & 25.9 & 71.9 & 38.0 & 48.0 & 29.4 & 26.2 & 31.2 & 42.6 \\
        FADA    & 84.5 & 40.1 & 83.1 & 4.8 & 0.0 & 34.3 & 20.1 & 27.2 & 84.8 & 36.5 & 53.5 & 53.5 & 22.6 & 85.4 & 43.7 & 52.0 & 26.8 & 27.8 & 25.6 & 45.2 \\
        SS-UDA      & 84.3 & 37.7 & 79.5  & 5.3  & 0.4   & 24.9 & 9.2  & 8.4  & 80.8 & 34.1    & 57.2 & 23.0 & 19.8  & 78.0  & 24.6  & 28.3  & 12.0  & 20.1  & 36.5 & 41.7 \\
        PIT         & 83.1 & 27.6 & 78.6  & 8.9  & 0.3   & 21.8 & 26.4 & 33.8 & 76.4 & 27.6    & 31.3 & 31.4 & 15.2  & 76.1  & 18.4  & 26.7  & 10.4  & 19.7  & 31.3 & 44.0 \\
        ELS-DA      & 81.7 & 43.8 & 80.1  & 22.3 & 0.5   & 29.4 & 28.6 & 21.2 & 83.4 & 33.7    & 26.3 & 48.4 & 20.4  & 79.2  & 26.7  & 30.1  & 13.7  & 24.1  & 40.2 & 47.2 \\
        PPOT & 91.2 & 43.5 & 82.1 & 22.9 & 0.7 & 32.9 & 31.2 & 29.0 & 83.1 & 36.5 & 85.4 & 65.2 & 30.0 & 83.9 & 45.2 & 55.1 & 32.2 & 27.6 & 47.8 & 47.0 \\
        PDA & 90.1 & 43.5 & 82.1 & 25.1 & 0.6 & 33.8 & 30.2 & 27.6 & 84.6 & 39.1 & 87.0 & 63.2 & 30.0 & 85.1 & 44.9 & 55.2 & 30.1 & 29.7 & 47.2 & 48.3 \\
        CPH & 92.1 & 43.2 & 84.9 & 25.1 & 0.8 & 35.0 & 30.2 & 29.8 & 84.1 & 38.6 & 85.7 & 64.6 & 30.2 & 86.3 & 47.1 & 55.2 & 30.7 & 29.0 & \textbf{53.1} & 50.0 \\
        \midrule
\textbf{RLGLC} & \textbf{93.8} & \textbf{45.6} & \textbf{86.2} & \textbf{27.7} & \textbf{1.2} & \textbf{36.3} & \textbf{34.7} & \textbf{32.2} & \textbf{87.1} & \textbf{41.4} & \textbf{89.6} & \textbf{67.4} & \textbf{32.6} & \textbf{87.6} & \textbf{49.1} & \textbf{58.1} & \textbf{33.5} & \textbf{30.9} & 46.9 & \textbf{51.7} \\
        
        \bottomrule
    \end{tabular}}
\end{table*}

\begin{table}[!ht]
	\caption{mAP on two object detection tasks including PASCAL VOC $\to$ Clipart (P$\to$C) and PASCAL VOC $\to$ Watercolor (P$\to$W) for unsupervised domain adaptation by using ResNet-101 and Faster-RCNN as the backbone}
	\centering
			\begin{tabular}{l|cc|c}
				\toprule
				Method & P $\to$ C & P $\to$ W & Avg \\
				\midrule
				CaCo & 43.9 & 58.7 & 51.3 \\
				MGA & 44.8 & 58.1 & 51.5 \\
				KL & 44.1 & 57.4 & 50.8\\
				GRDA & 42.9 & 54.7 & 50.8 \\
				\midrule
				\textbf{RLGLC} & \textbf{46.2} & \textbf{60.3} & \textbf{53.3} \\
				\bottomrule
			\end{tabular}
	\label{tab:45}
\end{table}

\begin{table}[ht]
\centering
\caption{Accuracy (\%) for Office-31 D$\rightarrow$A (ResNet-50)}
\label{tab:accuracy_curve_1}
\begin{tabular}{cccc}
\toprule
\textbf{Epoch} & \textbf{RLGC} & \textbf{RLGC*} & \textbf{RLGLC} \\
\midrule
2000 & 66.2 & 70.1 & 70.7 \\
2500 & 69.0 & 71.6 & 74.2 \\
3000 & 73.9 & 74.2 & 77.6 \\
3500 & 74.7 & 77.6 & 80.0 \\
4000 & 74.9 & 77.4 & 80.1 \\
\bottomrule
\end{tabular}
\end{table}

\begin{table}[ht]
\centering
\caption{Accuracy (\%) for Digits M$\rightarrow$U (ResNet-50)}
\label{tab:accuracy_curve_2}
\begin{tabular}{cccc}
\toprule
\textbf{Epoch} & \textbf{RLGC} & \textbf{RLGC*} & \textbf{RLGLC} \\
\midrule
2000 & 92.8 & 93.0 & 92.3 \\
2500 & 94.6 & 95.7 & 96.0 \\
3000 & 96.0 & 97.1 & 97.3 \\
3500 & 97.2 & 97.9 & 98.0 \\
4000 & 97.1 & 97.8 & 98.1 \\
\bottomrule
\end{tabular}
\end{table}

\subsection{Ablation study and parameter sensitivity}

{\bf Ablation study on the component modules.}
RLGLC is mainly composed of two parts, including the global consistency module and the local consistency module. To evaluate the effects of the two modules, we construct a simplified version of RLGLC by eliminating the local consistency module and denote it by RLGC*. From Table \ref{tab:1}, \ref{tab:2}, \ref{tab:3}, \ref{tab:domatne}, and \ref{tab:4}, we can observe that RLGC* obtains comparable classification results with other methods on both specific transfer tasks and average classification accuracy, e.g., in Table \ref{tab:2}, \ref{tab:domatne}, and \ref{tab:4}, the performance of RLGC* is better than that of all the baselines compared. Also for the specific transfer task R$\to$P in Table \ref{tab:1}, we can obtain that the performance of RLGC* is better than all compared baselines. This demonstrates the effectiveness of the global consistency module. Comparing RLGLC with RLGC*, we observe that RLGLC outperforms RLGC* on all datasets, both in terms of specific transfer tasks and final average results, which demonstrates that the proposed local consistency module can indeed improve the discriminativeness of the learned features in the target domain. Based on the specific transfer task D$\to$A of the Office31 dataset and the specific transfer task M$\to$U of the Digits dataset, we record the accuracy curves of RLGC* and RLGLC during the training process in Table \ref{tab:accuracy_curve_1} and Table \ref{tab:accuracy_curve_2}. When epoch$>$2500, the accuracy of RLGLC is greater than RLGC*. This further indicates that the local consistency module is effective in making the learned representations of target domain samples more discriminative.

{\bf Ablation study on the proposed discrepancy metric and the inequality constraint.}
Based on the specific transfer task P$\to$C of the Office-home dataset, we compare the trends of three different metrics during the training process, including the proposed AR-WWD, the Wasserstein of Wasserstein distance (WWD) which is obtained by setting the hyperparameter $\beta$ in the AR-WWD to 0, and the Wasserstein distance (WD).   The differences in metrics are shown in Figure \ref{fig:g}, where the abscissa indicates the number of training epochs and the ordinate indicates the distribution distance. As we can see, the training curve of the proposed AR-WWD is the most stable and also the fastest convergent. The Wasserstein distance is the most oscillating and converges the slowest. These observations demonstrate the effectiveness of the proposed AR-WWD. Also, comparing AR-WWD with WWD, we conclude that relaxing the exact aligning constraint to a loose one can better align the distributions. Comparing WWD with WD, we conclude that using the WD as the ground metric in another nested WD is more effective in comparing image data. Also, we denote the method that uses WWD as the measurement of distributional diversity as RLGC. Comparing RLGC with RLGC*, the only difference is located in the measurement of distributional diversity. The baselines that related to WD include WDGRL and SWD. The experimental results of RLGC are shown in Table \ref{tab:1}, \ref{tab:2}, \ref{tab:3}, \ref{tab:domatne}, and \ref{tab:4}. We can observe that on all datasets, both the specific migration task and the final average result, the performance of the methods related to WD is basically lower than that of RLGC, which is basically lower than the character of RLGC*, this futher demonstrates the effectiveness of the proposed AR-WWD.

{\bf Ablation study on the proposed regularization item.}
Based on the specific transfer task U$\to$M of the Digits dataset, we evaluate the effects of the proposed regularization item. 
As shown in Figure \ref{fig: sub_figure1}, the classification accuracy of RLGLC in case $\alpha=10^{-1}$ is obviously lower than the classification accuracy of RLGLC in case $\alpha=1$. This demonstrates that the regularization item is important for improving the performance of the proposed model.

{\bf Ablation study on the local consistency module.} Compared with the existing UDA framework in Equation (\ref{asd}), our proposed framework in Equation (\ref{Eq:udbsasdasf}) adds a conditional mutual information term to enhance the discriminability of target-domain feature representations. To assess the standalone effectiveness of this term in traditional UDA frameworks, we designed experiments on the VisDA-2017 dataset using DANN, SWD, and CAN, and on the DomainNet dataset using DANN, SWD, and MDD. We integrated our Local Consistency Module (LCM) into each method, resulting in “+ LM” versions shown in Table \ref{tab:3} and Table \ref{tab:domatne}. Across both datasets, adding the LCM consistently improves performance by at least 1\%. On VisDA-2017, for example, SWD+LM achieves a 1.2\% improvement over SWD. Additionally, within specific transfer tasks, methods incorporating “+ LM” exhibit further gains. On DomainNet, MDD+LM outperforms MDD by 1.5\% for the S $\to$ R transfer task. These results demonstrate the efficacy of our proposed framework and confirm that the Local Consistency Module can independently enhance performance.

{\bf Influence of hyper-parameters.}
Based on the specific transfer task U$\to$M, we evaluate the effects of the hyper-parameter $\alpha $, which balances the influence of the regularization item in the RLGLC, and the hyper-parameters $\beta$ and $\lambda $, which controls the effectiveness of ${W_{1,{W_{2,{d_\Omega }}}}}\left( {P_s^\varphi\left( Z \right) ,P_t^\varphi \left( Z \right)} \right)$. For $\alpha $, we first fix $\beta=0.4$ and $\lambda=10$, and then select $\alpha $ from the range of $\{ {10^{ - 3}},{10^{ - 2}},{10^{ - 1}},1, 10, {10^2},{10^3}\} $. The results are shown in Figure \ref{fig: sub_figure1}. We can see that our method obtains the best accuracy when $\alpha=1$. 
We then fix $\alpha=1$ and $\beta=0.4$, and then select $\lambda $ from the range of $\{ {10^{ - 3}},{10^{ - 2}},{10^{ - 1}},1, 10, {10^2},{10^3}\} $. From the results in Figure \ref{fig: sub_figure2}, we can observe that the best accuracy is achieved when $\lambda=10$.
Next, we fix $\alpha=1$ and $\lambda=10$, and then select $\beta $ from the range of $\left\{ 0.1, 0.2, 0.3,...,0.9 \right\}$. As we can see from Figure \ref{fig: sub_figure3}, our method obtains the best results when $\beta=0.4$. This suggests that constraining the target domain’s distribution within the source domain’s distribution helps reduce label-irrelevant information, demonstrating the effectiveness of the proposed UDA framework. Notably, when \(\beta = 0.4\), accuracy is the lowest, indicating that overly relaxed constraints lead to the loss of task-relevant information.

\section{Conclusion}
In this paper, we revisit adversarial-based representation learning for UDA from an information-theoretic perspective and show that aligning feature distributions and minimizing source-domain risk alone only ensures the transferability of target-domain features. To preserve their discriminative power, an additional loss term tailored to target-domain data is necessary. Motivated by these findings, we propose a novel adversarial-based representation learning framework that explicitly enforces both transferability and discriminability. Our framework is instantiated as RLGLC. A global  consistency module, AR-WWD, alleviates class imbalance and dimension-insensitivity issues in distribution alignment, while a local consistency module enhances target-domain feature discriminability via conditional mutual information. We theoretically prove that our approach tightens the upper bound of expected risk on the target domain, bridging the gap between theory and practice. Empirical results on multiple benchmark datasets demonstrate that RLGLC consistently outperforms state-of-the-art methods, validating the importance of explicitly ensuring transferability and discriminability in adversarial-based UDA.

\ifCLASSOPTIONcaptionsoff
  \newpage
\fi

\bibliographystyle{IEEEtran}
\bibliography{rlglc}

\clearpage

\vskip -0.5in
\begin{IEEEbiography}[{\includegraphics[width=1in,height=1.25in,clip,keepaspectratio]{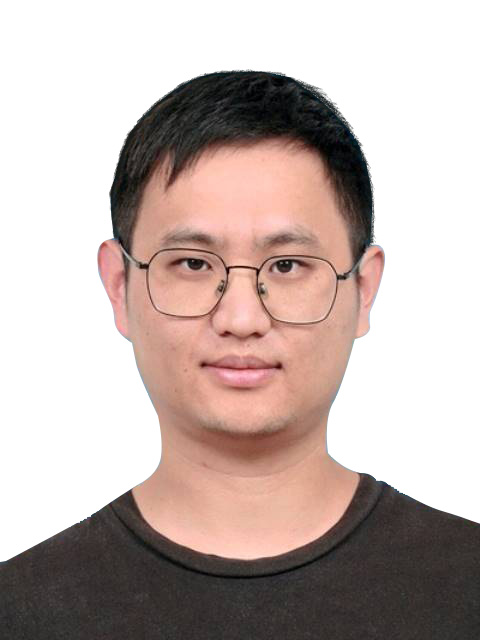}}]{Wenwen Qiang}
	received the MS degree in mathematics from China Agricultural University, Beijing, in 2018, and PhD degree in software engineering from the University of Chinese Academy of Sciences, Beijing, in 2022. Currently, he is assistant professor with the Institute of Software Chinese Academy of Science. His research interests include transfer learning, self-supervised learning, and causal inference.
\end{IEEEbiography}
\vskip -0.25in

\begin{IEEEbiography}[{\includegraphics[width=1in,height=1.25in,clip,keepaspectratio]{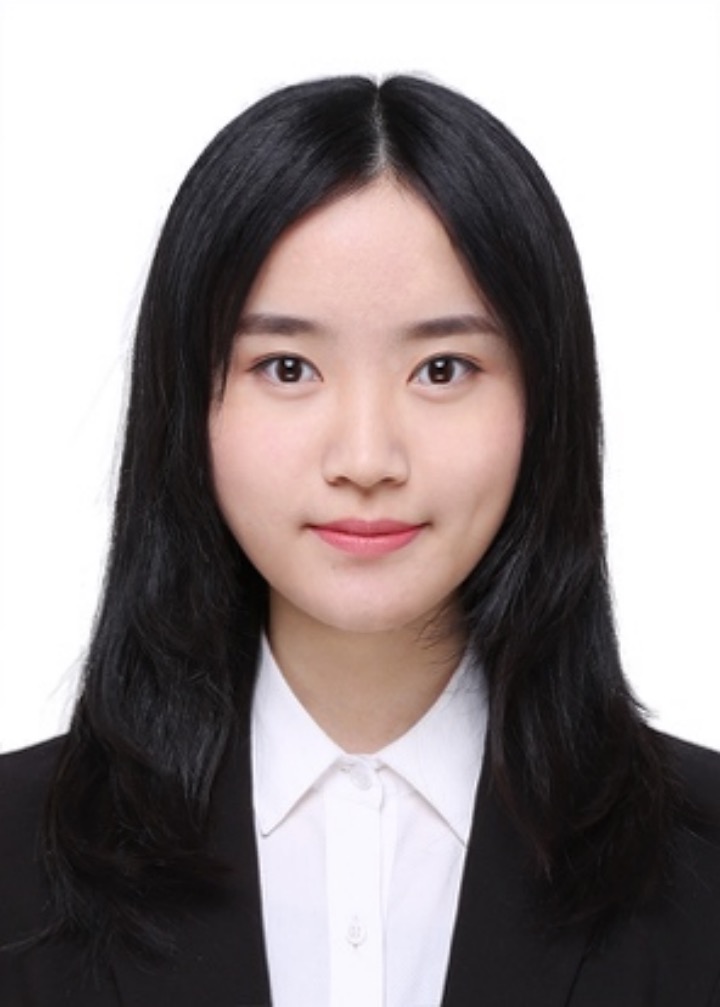}}]{Ziyin Gu}
	received the bachelor degree in Physics from Nankai University, Tianjing, in 2019. She is currently a PhD candidate in Computer Science and Technology at the University of Chinese Academy of Sciences. Her primary research interests focus on representation learning, including transfer learning, human-machine dialogue systems, and sentiment analysis.
    
\end{IEEEbiography}
\vskip -0.25in

\begin{IEEEbiography}[{\includegraphics[width=1in,height=1.25in,clip,keepaspectratio]{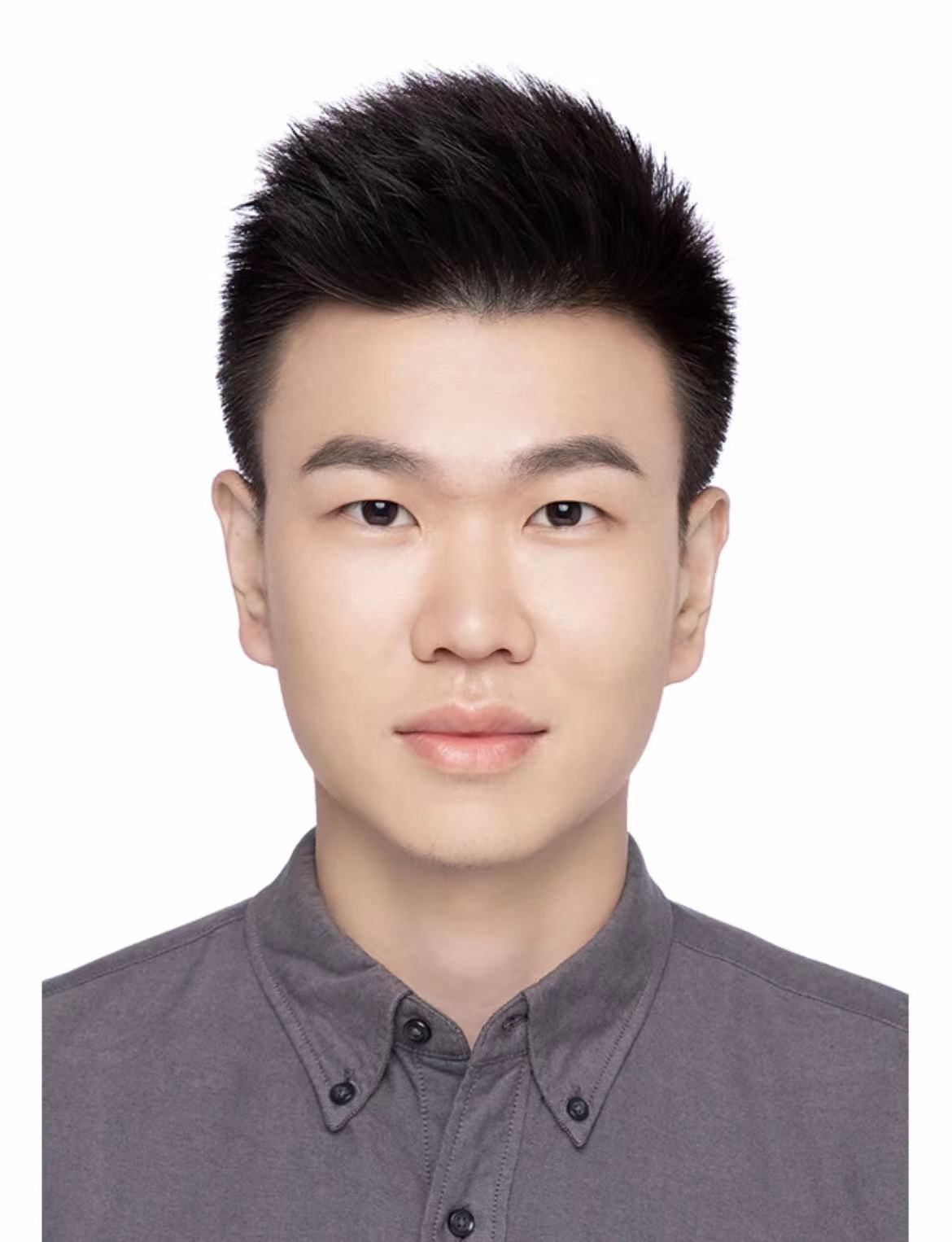}}]{Lingyu Si}
	received the MS degree of advanced computing in the department of engineering, University of Bristol, UK in 2018, and PhD degree in software engineering from the University of Chinese Academy of Sciences, Beijing, in 2023.    
    He is currently a senior engineering with the Institute of Software Chinese Academy of Science. His research interests include transfer learning, computer vision, and representation learning.
\end{IEEEbiography}
\vskip -0.25in

\begin{IEEEbiography}[{\includegraphics[width=1in,height=1.25in,clip,keepaspectratio]{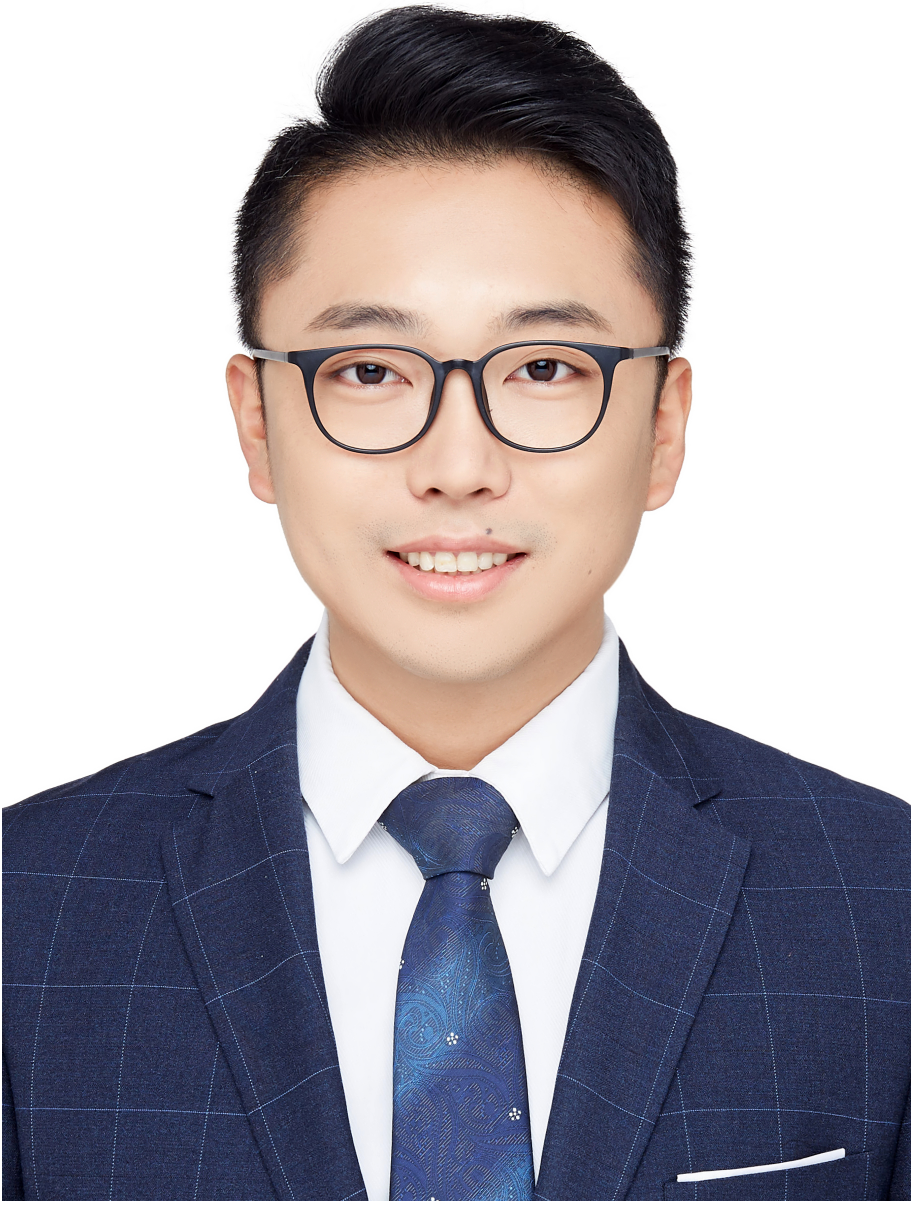}}]{Jiangmeng Li}
	received the MS degree with concentration of data science, School of Professional Studies, New York University, New York, New York, USA, in 2018, and PhD degree in software engineering from the University of Chinese Academy of Sciences, Beijing, in 2023.    
    He is currently an assistant professor with the Institute of Software Chinese Academy of Science. His research interests include transfer learning, deep learning, and machine learning.
\end{IEEEbiography}
\vskip -0.25in

\begin{IEEEbiography}[{\includegraphics[width=1in,height=1.25in,clip,keepaspectratio]{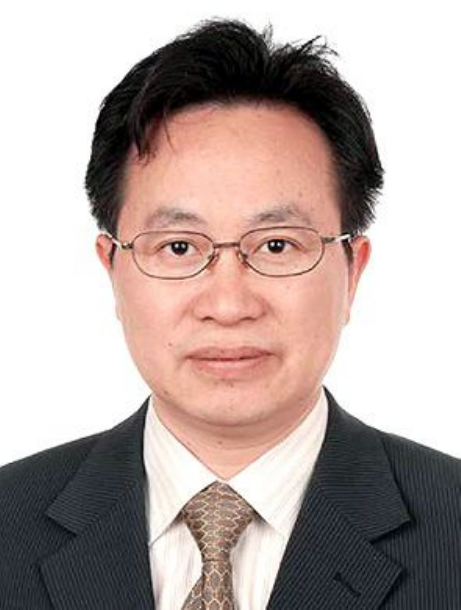}}]{Changwen Zheng}
	received the Ph.D. degree in Huazhong University of Science and Technology. He is currently a professor in the Institute of Software Chinese Academy of Science and the University of Chinese Academy of Sciences. His research interests include computer graph, artificial intelligence, and causal inference.
\end{IEEEbiography}
\vskip -0.25in

\begin{IEEEbiography}[{\includegraphics[width=1in,height=1.25in,clip,keepaspectratio]{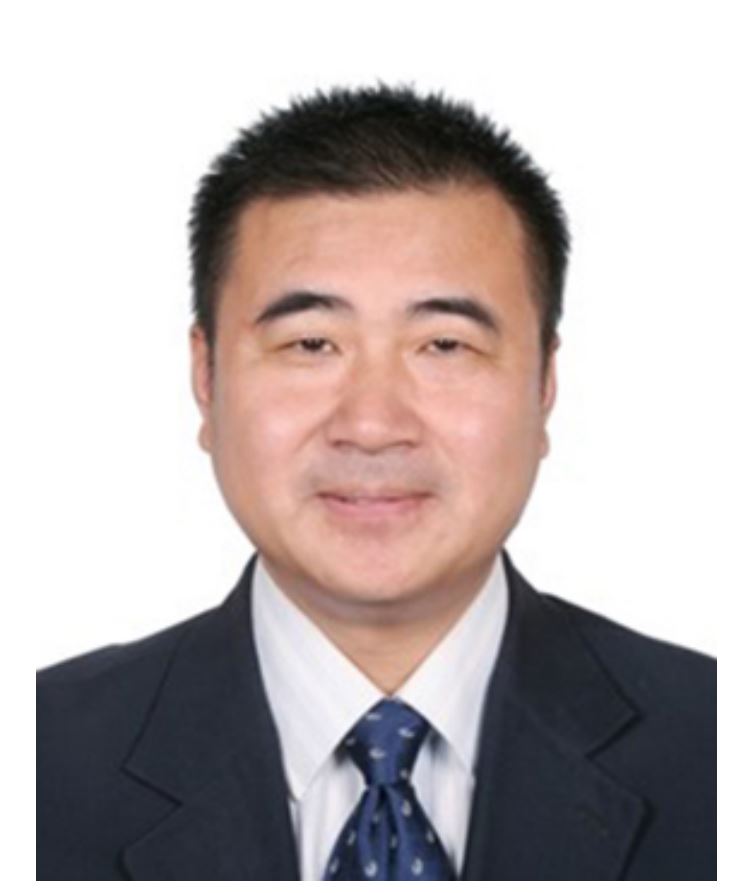}}]{Fuchun Sun}
	received the PhD degree in computer science and technology from Tsinghua University, Beijing, China, in 1997. He is currently a Professor with the Department of Computer Science and Technology and President of Academic Committee of the Department, Tsinghua University, deputy director of State Key Lab. of Intelligent Technology and Systems, Beijing, China. His research interests include intelligent control and robotics, information sensing and processing in artificial cognitive systems, and networked control systems. He was recognized as a Distinguished Young Scholar in 2006 by the Natural Science Foundation of China. He serves as Editor-in-Chief of International Journal on Cognitive Computation and Systems, and an Associate Editor for a series of international journals including the IEEE Transactions on Cognitive and Developmental Systems, the IEEE Transactions on Fuzzy Systems, and the IEEE Transactions on Systems, Man, and Cybernetic: Systems. He was elected an IEEE Fellow in 2019.
\end{IEEEbiography}
\vskip -0.25in

\begin{IEEEbiography}[{\includegraphics[width=1in,height=1.25in,clip,keepaspectratio]{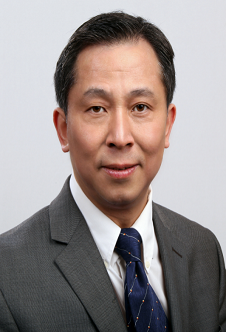}}]{Hui Xiong}
	is currently a Chair Professor at the Hong Kong University of Science and Technology (Guangzhou). Dr. Xiong’s research interests include data mining, mobile computing, and their applications in business. Dr. Xiong received his PhD in Computer Science from University of Minnesota, USA. He has served regularly on the organization and program committees of numerous conferences, including as a Program Co-Chair of the Industrial and Government Track for the 18th ACM SIGKDD International Conference on Knowledge Discovery and Data Mining (KDD), a Program Co-Chair for the IEEE 2013 International Conference on Data Mining (ICDM), a General Co-Chair for the 2015 IEEE International Conference on Data Mining (ICDM), and a Program Co-Chair of the Research Track for the 2018 ACM SIGKDD International Conference on Knowledge Discovery and Data Mining. He received the 2021 AAAI Best Paper Award and the 2011 IEEE ICDM Best Research Paper award. For his outstanding contributions to data mining and mobile computing, he was elected an AAAS Fellow and an IEEE Fellow in 2020.	
\end{IEEEbiography}

\clearpage

\section{Appendix}
\subsection{Proofs}
\begin{lemma}
	\label{l1}
	Suppose $Z_u$ is a deterministic function of $X_u$, where $u \in \left\{ {s,t} \right\}$, then the following Markov chain holds: ${Z_t} \leftarrow {X_t} \leftrightarrow Y \leftrightarrow {X_s} \to {Z_s}$.
\end{lemma}
\begin{proof}
	When $Z_u$ is a deterministic function of $X_u$, for any $A$ in the sigma-algebra induced by $Z_u$, we have:
	\begin{align}
	  \mathbb{E}\left[ {{\mathbbm{1}_{[{Z_s} \in A]}}\left| {{X_s},\left\{ {Y,{X_t}} \right\}} \right.} \right] &= \mathbb{E}\left[ {{\mathbbm{1}_{[{Z_s} \in A]}}\left| {{X_s},{X_t}} \right.} \right] \\
	&= \mathbb{E}\left[ {{\mathbbm{1}_{[{Z_s} \in A]}}\left| {{X_s}} \right.} \right],\\
	\mathbb{E}\left[ {{\mathbbm{1}_{[{Z_t} \in A]}}\left| {{X_t},\left\{ {Y,{X_s}} \right\}} \right.} \right] &= \mathbb{E}\left[ {{\mathbbm{1}_{[{Z_t} \in A]}}\left| {{X_t},{X_s}} \right.} \right] \\
	&= \mathbb{E}\left[ {{\mathbbm{1}_{[{Z_s} \in A]}}\left| {{X_t}} \right.} \right].  
	\end{align}
	
	This implies ${X_t} \upmodels {Z_s}\left| {{X_s}} \right.,Y \upmodels {Z_s}\left| {{X_s}} \right.,{X_s} \upmodels {Z_t}\left| {{X_t}} \right.$, and $Y \upmodels {Z_t}\left| {{X_t}} \right.$. Thus, we obtain the Markov chain: ${Z_t} \leftarrow {X_t} \leftrightarrow Y \leftrightarrow {X_s} \to {Z_s}$.
\end{proof}

\label{12}
\begin{proposition}
	\label{p1}
	Based on \textbf{Definition \ref{d1}}, we have: $I\left( {{X_s};{X_t}\left| {{Z_u}} \right.} \right) =0 \Leftrightarrow I\left( {{X_u};Y} \right) = I\left( {{Z_u};Y} \right)$.
\end{proposition}
\begin{proof}
\begin{equation}	
	\begin{array}{l}
		I\left(X_u;Y\left| Z_u \right. \right) \\
		= I\left(X_u;Y \right) - I\left(X_u;Y;Z_u \right) \\
		= I\left(X_u;Y \right) - I\left(Y;Z_u \right) + I\left(Z_u;Y\left| X_u \right. \right)\\
		= I\left(X_u;Y \right) - I\left(Y;Z_u \right).
	\end{array}
\end{equation}

Since $Z_u$ is a represention of $X_u$, we have
\begin{equation}
\begin{array}{l}
I\left( {Y;{X_s}\left| {{Z_s}} \right.} \right)\\
= I\left( {{X_s};Y\left| {{X_t}{Z_s}} \right.} \right) + I\left( {{X_s};{X_t};Y\left| {{Z_s}} \right.} \right)\\
= I\left( {{X_s};Y\left| {{X_t}} \right.} \right) - I\left( {{X_s};{Z_s};Y\left| {{X_t}} \right.} \right) + I\left( {{X_s};{X_t};Y\left| {{Z_s}} \right.} \right)\\
= I\left( {{X_s};Y\left| {{X_t}} \right.} \right) - I\left( {{Z_s};Y\left| {{X_t}} \right.} \right) + I\left( {{Z_s};Y\left| {{X_t}{X_s}} \right.} \right) \\ \qquad\qquad\qquad\qquad\qquad+ I\left( {{X_s};{X_t};Y\left| {{Z_s}} \right.} \right)\\
\le I\left( {{X_s};Y\left| {{X_t}} \right.} \right) + I\left( {{Z_s};Y\left| {{X_t}{X_s}} \right.} \right) + I\left( {{X_s};{X_t};Y\left| {{Z_s}} \right.} \right)\\
= I\left( {{X_s};Y\left| {{X_t}} \right.} \right) + I\left( {{X_s};{X_t};Y\left| {{Z_s}} \right.} \right)\\
= I\left( {{X_s};Y\left| {{X_t}} \right.} \right) + I\left( {{X_s};{X_t}\left| {{Z_s}} \right.} \right) - I\left( {{Z_s};{Z_t}\left| {{Z_s}Y} \right.} \right)\\
\le I\left( {{X_s};Y\left| {{X_t}} \right.} \right) + I\left( {{X_s};{X_t}\left| {{Z_s}} \right.} \right).
\end{array}
\end{equation}

Similarly, we have $I\left( {Y;{X_t}\left| {{Z_t}} \right.} \right) \le I\left( {{X_t};Y\left| {{X_s}} \right.} \right) + I\left( {{X_s};{X_t}\left| {{Z_t}} \right.} \right)$.

Therefore, we can obtain that
\begin{equation}
\begin{array}{l}
I\left( {{X_s};Y\left| {{Z_s}} \right.} \right) \le I\left( {{X_s};{X_t}\left| {{Z_s}} \right.} \right),\\
{\rm{ }}I\left( {{X_t};Y\left| {{Z_t}} \right.} \right) \le I\left( {{X_s};{X_t}\left| {{Z_t}} \right.} \right),\\
I\left( {{X_s};{X_t}\left| {{Z_u}} \right.} \right) = 0 \Rightarrow I\left( {{X_u};Y\left| {{Z_u}} \right.} \right) = 0.
\end{array}
\end{equation}

Therefore, we have $I\left( {{X_s};{X_t}\left| {{Z_u}} \right.} \right)=0 \Leftrightarrow I\left( {{X_u};Y} \right) = I\left( {{Z_u};Y} \right)$.
\end{proof}

\begin{proposition}
	\label{p3}
	In the extreme case, $X_s$ and $X_t$ only share label information, the proposed method is equivalent to the supervised information bottleneck method without needing to access the labels.
\end{proposition}
\begin{proof}
	If $X_s$ and $X_t$ share only label information, we can obtain that
	\begin{equation}
	I\left( {{X_s};{X_t}} \right) = I\left( {{X_s};Y} \right) = I\left( {{X_t};Y} \right) = H\left( Y \right).
	\end{equation}
	
	Then, we have
	\begin{equation}
	I\left( {{X_s};{X_t}\left| {{Z_u}} \right.} \right)=0 \Leftrightarrow I\left( {{X_u};Y} \right) = I\left( {{Z_u};Y} \right) = H\left( Y \right).
	\end{equation}
	
	We can obtain
	\begin{equation}
	\begin{array}{l}
	I\left( {{X_s};{Z_s}} \right)\\
	= I\left( {{X_s};{Z_s}\left| {{X_t}} \right.} \right) + I\left( {{X_s};{X_t};{Z_s}} \right)\\
	= I\left( {{X_s};{Z_s}\left| {{X_t}} \right.} \right) + I\left( {{X_s};{X_t}} \right) - I\left( {{X_s};{X_t}\left| {{Z_s}} \right.} \right)\\
	= I\left( {{X_s};{Z_s}\left| {{X_t}} \right.} \right) + I\left( {{X_s};{X_t}} \right)\\
	= I\left( {{X_s};{Z_s}\left| {{X_t}} \right.} \right) + I\left( {{X_s};Y} \right).
	\end{array}
	\end{equation}
	
	Similarly, we can also obtain that $I\left( {{X_t};{Z_t}} \right) = I\left( {{X_t};{Z_t}\left| {{X_s}} \right.} \right) + I\left( {{X_t};Y} \right)$.
	
	Then, we can obtain that
	\begin{equation}
	\begin{array}{l}
	I\left( {{X_u};{Z_u}} \right)\\
	= I\left( {{X_u};{Z_u}\left| Y \right.} \right) + I\left( {{X_u};{Z_u};Y} \right)\\
	= I\left( {{X_u};{Z_u}\left| Y \right.} \right) + I\left( {{X_u};Y} \right) - I\left( {{X_u};Y\left| {{Z_u}} \right.} \right)\\
	= I\left( {{X_u};{Z_u}\left| Y \right.} \right) + I\left( {{X_u};Y} \right).
	\end{array}
	\end{equation}
	
	Therefore, we conclude that $I\left( {{X_s};{Z_s}\left| Y \right.} \right) = I\left( {{X_s};{Z_s}\left| {{X_t}} \right.} \right),{\rm{ }}I\left( {{X_t};{Z_t}\left| Y \right.} \right) = I\left( {{X_t};{Z_t}\left| {{X_s}} \right.} \right)$. Note that $Z_s$ which minimizes $I\left( {{X_s};{Z_s}\left| {{X_t}} \right.} \right)$ is also minimizing $I\left( {{X_s};{Z_s}\left| Y \right.} \right)$, and $Z_t$ which minimizes $I\left( {{X_t};{Z_t}\left| {{X_s}} \right.} \right)$ is also minimizing $I\left( {{X_t};{Z_t}\left| Y \right.} \right)$. When $I\left( {{X_u};{Z_u}\left| Y \right.} \right)$ is minimal, $I\left( {{Y};{Z_u}} \right)$ is also maximal and $I\left( {{X_u};{Z_u}} \right)= H\left( {{Y}} \right)$. As a consequence, a minimal sufficient representation $Z_u$ of $X_u$ is the representation for which mutual information ($I\left( {{X_u};{Z_u}} \right)$) is maximal, no superfluous information ($I\left( {{X_u};{Z_u}\left| Y \right.} \right)$) can be identified and removed.
\end{proof}

\begin{theorem}
	\label{asa}
	Suppose the representations $Z_s$ and $Z_t$ for the source domain and the target domain are obtained by minimizing the objective function (\ref{asd}). Then, the discriminability and transferability of $Z_s$ are increased, while only the transferability of $Z_t$ is improved.
\end{theorem}
\begin{proof}
    First, we have:
	\begin{equation}
	\begin{array}{l}
		I\left( {{X_s};{Z_s}\left| {{X_t}} \right.} \right)\\
		= {\mathbb{E}_{P\left( {{X_s},{X_t}} \right)}}{\mathbb{E}_{P\left( {{Z_s}\left| {{X_s}} \right.} \right)}}\left[ {\log \frac{{P\left( {{Z_s}\left| {{X_s}} \right.} \right)}}{{P\left( {{Z_s}\left| {{X_t}} \right.} \right)}}} \right]\\
		= {\mathbb{E}_{P\left( {{X_s},{X_t}} \right)}}{\mathbb{E}_{P\left( {{Z_s}\left| {{X_s}} \right.} \right)}}\left[ {\log \frac{{P\left( {{Z_s}\left| {{X_s}} \right.} \right)P\left( {{Z_t}\left| {{X_t}} \right.} \right)}}{{P\left( {{Z_s}\left| {{X_t}} \right.} \right)P\left( {{Z_t}\left| {{X_t}} \right.} \right)}}} \right]\\
={\mathbb{E}_{P\left( {{X_s},{X_t}} \right)}}{\mathbb{E}_{P\left( {{Z_s}\left| {{X_s}} \right.} \right)}}\left[ {\log \frac{{P\left( {{Z_s}\left| {{X_s}} \right.} \right)P\left( {{Z_t}\left| {{X_t}} \right.} \right)}}{{P\left( {{Z_s}\left| {{X_t}} \right.} \right)P\left( {{Z_t}\left| {{X_t}} \right.} \right)}}} \right]\\
 \le {\mathbb{E}_{P\left( {{X_s},{X_t}} \right)}}{\mathbb{E}_{P\left( {{Z_s}\left| {{X_s}} \right.} \right)}}\left[ {\log \frac{{P\left( {{Z_s}\left| {{X_s}} \right.} \right)}}{{P\left( {{Z_t}\left| {{X_t}} \right.} \right)}}} \right]\\
 = {\mathbb{E}_{P\left( {{X_s},{X_t}} \right)}}{\mathbb{E}_{P\left( {{Z_s}\left| {{X_s}} \right.} \right)}}\left[ {\log \frac{{{{P\left( {{Z_s},{X_s}} \right)} \mathord{\left/
 {\vphantom {{P\left( {{Z_s},{X_s}} \right)} {P\left( {{X_s}} \right)}}} \right.
 \kern-\nulldelimiterspace} {P\left( {{X_s}} \right)}}}}{{{{P\left( {{Z_t},{X_t}} \right)} \mathord{\left/
 {\vphantom {{P\left( {{Z_t},{X_t}} \right)} {P\left( {{X_t}} \right)}}} \right.
 \kern-\nulldelimiterspace} {P\left( {{X_t}} \right)}}}}} \right]\\
 = {\mathbb{E}_{P\left( {{X_s},{X_t}} \right)}}{\mathbb{E}_{P\left( {{Z_s}\left| {{X_s}} \right.} \right)}}\left[ {\log \frac{{P\left( {{X_t}} \right)P\left( {{Z_s},{X_s}} \right)}}{{P\left( {{X_s}} \right)P\left( {{Z_t},{X_t}} \right)}}} \right]\\
 = {\mathbb{E}_{P\left( {{X_s},{X_t}} \right)}}{\mathbb{E}_{P\left( {{Z_s}\left| {{X_s}} \right.} \right)}}\left[ {\log \frac{{P\left( {{X_t}} \right)P({Z_s})P\left( {{X_s}\left| {{Z_s}} \right.} \right)}}{{P\left( {{X_s}} \right)P\left( {{Z_t}} \right)P\left( {{X_t}\left| {{Z_t}} \right.} \right)}}} \right]\\
 = {\mathbb{E}_{P\left( {{X_s},{X_t}} \right)}}\log \frac{{P\left( {{X_t}} \right)}}{{P\left( {{X_s}} \right)}} + KL\left( {P\left( {{Z_s}} \right)\left\| {P\left( {{Z_t}} \right)} \right.} \right) \\
 \qquad\qquad\qquad+ {\mathbb{E}_{P\left( {{X_s},{X_t}} \right)}}{\mathbb{E}_{P\left( {{Z_s}\left| {{X_s}} \right.} \right)}}\log \frac{{P\left( {{X_s}\left| {{Z_s}} \right.} \right)}}{{P\left( {{X_t}\left| {{Z_t}} \right.} \right)}}\\
 = {\mathbb{E}_{P\left( {{X_s}\left| {{X_t}} \right.} \right)}}{\mathbb{E}_{P\left( {{X_t}} \right)}}\log \frac{{P\left( {{X_t}} \right)}}{{P\left( {{X_s}} \right)}} + KL\left( {P\left( {{Z_s}} \right)\left\| {P\left( {{Z_t}} \right)} \right.} \right) \\
 \qquad\qquad\qquad- {\mathbb{E}_{P\left( {{X_s},{X_t}} \right)}}{\mathbb{E}_{P\left( {{Z_s}\left| {{X_s}} \right.} \right)}}\log \frac{{P\left( {{X_t}\left| {{Z_t}} \right.} \right)}}{{P\left( {{X_s}\left| {{Z_s}} \right.} \right)}}\\
 \le -KL\left( {P\left( {{X_s}} \right)\left\| {P\left( {{X_t}} \right)} \right.} \right) + KL\left( {P\left( {{Z_s}} \right)\left\| {P\left( {{Z_t}} \right)} \right.} \right) \\
 \qquad\qquad\qquad- {\mathbb{E}_{P\left( {{X_t}\left| {{Z_t}} \right.} \right)}}{\mathbb{E}_{P\left( {{X_s},{X_t}} \right)}}{\mathbb{E}_{P\left( {{Z_s}\left| {{X_s}} \right.} \right)}}\log \frac{{P\left( {{X_t}\left| {{Z_t}} \right.} \right)}}{{P\left( {{X_s}\left| {{Z_s}} \right.} \right)}}\\
 = KL\left( {P\left( {{Z_s}} \right)\left\| {P\left( {{Z_t}} \right)} \right.} \right) - KL\left( {P\left( {{X_s}} \right)\left\| {P\left( {{X_t}} \right)} \right.} \right) \\
 \qquad\qquad\qquad- {\mathbb{E}_{P\left( {{Z_s}\left| {{X_s}} \right.} \right)}}KL\left( {P\left( {{X_t}\left| {{Z_t}} \right.} \right)\left\| {P\left( {{X_s}\left| {{Z_s}} \right.} \right)} \right.} \right)\\
 \le  KL\left( {P\left( {{Z_s}} \right)\left\| {P\left( {{Z_t}} \right)} \right.} \right).
	\end{array}
    \end{equation}
    
    Similarly, we can also obtain:
    \begin{equation}
	I\left( {{X_t};{Z_t}\left| {{X_s}} \right.} \right) \le KL\left( {P\left( {{Z_t}} \right)\left\| {P\left( {{Z_s}} \right)} \right.} \right) .
    \end{equation}
    
    Regarding the KL-divergence and the Wasserstein distance as difference measurements between two distributions, minimizing $KL\left( {P\left( {Z_t} \right)\left\| {P\left( {{Z_s}} \right)} \right.} \right)$ and $KL\left( {P\left( {{Z_s}} \right)\left\| {P\left( {{Z_t}} \right)} \right.} \right)$ is equivalent to minimizing the Wasserstein distance between the distribution ${P\left( {{Z_s}} \right)}$ and the distribution ${P\left( {{Z_t}} \right)}$, thus, when the objective function (2) is minimized (equal to 0), the values of $I\left( {{X_s};{Z_s}\left| {{X_t}} \right.} \right)$ and $I\left( {{X_t};{Z_t}\left| {{X_s}} \right.} \right)$ are also reduced. Based on the \textbf{Definition \ref{d3}}, we can conclude that the transferability of $Z_s$ and the transferability of $Z_t$ are increased.

    As we can see, the second term in objective (2) is the expected risk of the learned classifier in the source domain, we denote the label as $Y$, then we have:
    \begin{equation}
    \begin{array}{l}
    I\left( {{Z_s};Y} \right)\\
    = KL\left( {P\left( {{Z_s},Y} \right)\left\| {P\left( {{Z_s}} \right)P\left( Y \right)} \right.} \right)\\
    \ge {\mathbb{E}_{\left( {{Z_s},Y} \right) \sim P\left( {{Z_s},Y} \right)}}\left( {T\left( {{Z_s},Y} \right)} \right) \\
    \qquad\qquad\qquad- \log {\mathbb{E}_{\left( {{Z_s},Y} \right) \sim P\left( {{Z_s}} \right)P\left( Y \right)}}\left( {{e^{T\left( {{Z_s},Y} \right)}}} \right)\\
    \ge {\mathbb{E}_{\left( {{Z_s},Y} \right) \sim P\left( {{Z_s},Y} \right)}}\left( {T\left( {{Z_s},Y} \right)} \right).
    \end{array}
    \end{equation}
    where ${P\left( {{Z_s},Y} \right)}$ is the joint distribution of $Z_s$ and $Y$, ${P\left( {{Z_s}} \right)P\left( Y \right)}$ is the product of $Z_s$ and $Y$, $T:{Z_s} \times Y \to R$ is a discriminator function modeled by a neural network. If we set $T = H\left(X_s,Y\right)=-{{\mathcal{L}}_{cl}}\left( {\psi \left( {\varphi \left( {{X_s}} \right)} \right),Y} \right)$, then we can obtain that $I\left( {{Z_s},Y} \right) \ge {{\mathcal{L}}_{cl}}\left( {\psi \left( {\varphi \left( {{X_s}} \right)} \right),Y} \right)$. So, minimizing the $L_{cl}$ equals to maximize the $I\left( {{Z_s};Y} \right)$. When $I\left( {{Z_s};Y} \right)$ is maximized ($I\left( {{Z_s};Y} \right) = H\left( Y \right)$), based on the \textbf{Proposition \ref{p1}}, we have $I\left( {{X_s};{X_t}\left| {{Z_s}} \right.} \right) = 0$, thus we can conclude that $Z_s$ is with the discriminability.

    However, there is no obvious term to constrain the learned sample feature representations of the target domain to be with discriminability.
\end{proof}

\begin{proposition} \label{vbjty}
Define the Conditional Noise Contrastive Estimation (CNCE) functional as Equation (\ref{qwasdasdf}), then the following hold:

1) For any choice of \(\phi\) and any \(K \ge 1\),
\begin{equation}
I_{\text{CNCE}}(X_s; X_t \mid Z_t, \phi, K) \;\le\; I(X_s; X_t \mid Z_t).
\end{equation}

2) There exists a function \(\phi^*\) that attains the supremum. Specifically, if
\begin{equation}
\phi^*(X_s,X_t,Z_t) = \log \frac{P(X_t \mid X_s,Z_t)}{P(X_t \mid Z_t)} + c(X_s,Z_t),
\end{equation}
where \(c(X_s,Z_t)\) is a function independent of \(X_t\), then we have:
\begin{equation}
\begin{array}{l}
\sup_{\phi} I_{\text{CNCE}}(X_s; X_t \mid Z_t, \phi, K) \\
= I_{\text{CNCE}}(X_s; X_t \mid Z_t, \phi^*, K) \\
= I(X_s; X_t \mid Z_t).
\end{array}
\end{equation}

3) As the number of negative samples \(K\) grows, the approximation tightens. In particular, we have:
\begin{equation}
\lim_{K \to \infty} \sup_{\phi} I_{\text{CNCE}}(X_s; X_t \mid Z_t, \phi, K) = I(X_s; X_t \mid Z_t).
\end{equation}
\end{proposition}

\begin{proof}

\textbf{Step 1}: Start from the definition of conditional mutual information.

Recall that:
\begin{equation}
\begin{array}{l}
I(X_s; X_t \mid Z_t) =\\
\mathbb{E}_{P(Z_t)}\left[\mathbb{E}_{P(X_s,X_t \mid Z_t)}\left(\log \frac{P(X_t \mid X_s,Z_t)}{P(X_t \mid Z_t)}\right)\right].
\end{array}
\end{equation}

This quantity measures how much knowing \(X_s\) reduces uncertainty in \(X_t\) beyond knowing \(Z_t\).

\textbf{Step 2}: Re-express the ratio using a parametric function \(\phi\).

Define \(\phi: \mathcal{X}_s \times \mathcal{X}_t \times \mathcal{Z}_t \to \mathbb{R}\) and let:
\begin{equation}
f_\phi(X_s,X_t,Z_t) := e^{\phi(X_s,X_t,Z_t)}.
\end{equation}
If we set
\begin{equation}
\phi^*(X_s,X_t,Z_t) = \log \frac{P(X_t \mid X_s,Z_t)}{P(X_t \mid Z_t)} + c(X_s,Z_t),
\end{equation}
then
\begin{equation}
f_{\phi^*}(X_s,X_t,Z_t) = e^{c(X_s,Z_t)} \frac{P(X_t \mid X_s,Z_t)}{P(X_t \mid Z_t)}.
\end{equation}

The function \(c(X_s,Z_t)\) does not affect ratios that normalize over \(X_t\). It can be seen as a “shift” that simplifies the form of the solution.

\textbf{Step 3}: Approximating the normalization with negative samples (noise contrastive estimation).

To rewrite the ratio \(\frac{P(X_t \mid X_s,Z_t)}{P(X_t \mid Z_t)}\), note:
\begin{equation}
\frac{P(X_t \mid X_s,Z_t)}{P(X_t \mid Z_t)} = \frac{f_\phi(X_s,X_t,Z_t)}{\mathbb{E}_{P(x'_t \mid Z_t)}[f_\phi(X_s,x'_t,Z_t)]}.
\end{equation}

We approximate the expectation in the denominator by Monte Carlo sampling \(K-1\) negative samples \(X_t^{(2)},\ldots,X_t^{(K)}\) from \(P(X_t \mid Z_t)\):
\begin{equation}
\mathbb{E}_{P(x'_t\mid Z_t)}[f_\phi(X_s,x'_t,Z_t)] \approx \frac{1}{K}\sum_{k=1}^K e^{\phi(X_s,X_t^{(k)},Z_t)},
\end{equation}
where \(X_t^{(1)}=X_t\) is the positive sample.

Substitute this approximation:
\begin{equation}
\frac{P(X_t \mid X_s,Z_t)}{P(X_t \mid Z_t)} \approx \frac{e^{\phi(X_s,X_t,Z_t)}}{\frac{1}{K}\sum_{k=1}^K e^{\phi(X_s,X_t^{(k)},Z_t)}}.
\end{equation}

\textbf{Step 4}: Define the CNCE estimator.

Replacing the exact ratio in \(I(X_s; X_t \mid Z_t)\) by the approximate ratio:
\begin{equation}
\begin{array}{l}
I(X_s; X_t \mid Z_t) \approx \mathbb{E}_{P(Z_t,X_s,X_t)}\mathbb{E}_{P(X_t^{(2:K)}\mid Z_t)}\\
\left[\log\left(\frac{e^{\phi(X_s,X_t,Z_t)}}{\frac{1}{K}\sum_{k=1}^K e^{\phi(X_s,X_t^{(k)},Z_t)}}\right)\right].
\end{array}
\end{equation}

Define this approximation as:
\begin{equation}
I_{\text{CNCE}}(X_s; X_t \mid Z_t,\phi,K).
\end{equation}

By construction, \(I_{\text{CNCE}}\) is a lower bound on \(I(X_s; X_t \mid Z_t)\) for any \(\phi\). This is because the log operation and the finite sample approximation introduce a variational lower bound structure similar to the InfoNCE bound used in contrastive representation learning. Thus:
\begin{equation}
I_{\text{CNCE}}(X_s; X_t \mid Z_t, \phi, K) \le I(X_s; X_t \mid Z_t).
\end{equation}

\textbf{Step 5}: Optimality and the choice of \(\phi^*\).

If we select \(\phi=\phi^*\) as defined in Step 2, then:
\begin{equation}
\phi^*(X_s,X_t,Z_t) = \log \frac{P(X_t \mid X_s,Z_t)}{P(X_t \mid Z_t)} + c(X_s,Z_t).
\end{equation}

In this case, the ratio is exactly represented up to a constant shift. As \(K\) grows large, the Monte Carlo estimate of the denominator converges almost surely to its expectation (by the Law of Large Numbers). Hence, for large \(K\):
\begin{equation}
I_{\text{CNCE}}(X_s; X_t \mid Z_t,\phi^*,K) \to I(X_s; X_t \mid Z_t).
\end{equation}

Since for finite \(K\), no other \(\phi\) can exceed \(I(X_s;X_t\mid Z_t)\), it follows that:
\begin{equation}
\begin{array}{l}
\sup_{\phi} I_{\text{CNCE}}(X_s; X_t \mid Z_t,\phi,K) \\
= I_{\text{CNCE}}(X_s; X_t \mid Z_t,\phi^*,K) \\
\le I(X_s; X_t \mid Z_t),
\end{array}
\end{equation}
and equality holds in the limit.

\textbf{Step 6}: Taking the limit \(K \to \infty\).

As \(K \to \infty\), the Monte Carlo approximation becomes exact:
\begin{equation}
\frac{1}{K}\sum_{k=1}^{K} e^{\phi(X_s,X_t^{(k)},Z_t)} \to \mathbb{E}_{P(x'_t\mid Z_t)}[e^{\phi(X_s,x'_t,Z_t)}].
\end{equation}

Therefore,
\begin{equation}
\lim_{K\to\infty}\sup_{\phi} I_{\text{CNCE}}(X_s; X_t \mid Z_t,\phi,K) = I(X_s; X_t \mid Z_t).
\end{equation}

This shows the estimator is consistent and achieves the true conditional mutual information as we increase the number of negative samples and choose \(\phi\) appropriately. We have defined a CNCE-based estimator \(I_{\text{CNCE}}\) for the conditional mutual information \(I(X_s; X_t \mid Z_t)\). This estimator forms a lower bound, can be made tight by selecting an optimal \(\phi^*\), and converges to the true mutual information in the limit of large \(K\). 

\end{proof}

\begin{theorem}
	\label{qwwq}
For an arbitrary learned feature representation pair \((Z_s,Z_t)\) derived from \((X_s,X_t)\), and for any decision rule that outputs a hypothesis \(\hat{Y}\) of \(Y\), we have the following upper bounds on the average error probability \(\bar{P}_e\):

1. Bound with \( I(X_s; Z_s \mid X_t) \):
   \begin{equation}
   \bar{P}_e \leq 1 - \exp\bigl[-H(Y) + I(X_s; Z_s \mid X_t)\bigr].
   \end{equation}

2. Bound with \( I(X_t; Z_t \mid X_s) \):
   \begin{equation}
   \bar{P}_e \leq 1 - \exp\bigl[-H(Y) + I(X_t; Z_t \mid X_s)\bigr].
   \end{equation}

3. Bound with \( I(X_s; X_t \mid Z_s) \):
   \begin{equation}
   \bar{P}_e \leq 1 - \exp\bigl[-H(Y) + I(X_s; X_t \mid Z_s)\bigr].
   \end{equation}

4. Bound with \( I(X_s; X_t \mid Z_t) + \Delta(\phi,\psi) \):
   \begin{equation}
   \bar{P}_e \leq 1 - \exp\bigl[-H(Y) + I(X_s; X_t \mid Z_t) + \Delta(\phi,\psi)\bigr].
   \end{equation}

\end{theorem}

\begin{proof}
\textbf{Step 1}: Relating Error Probability to Entropy and Mutual Information

We begin with a fundamental link between error probability and the entropy of the label \(Y\). Consider a classifier or decoder that attempts to guess \( Y \) from some variables (which may be \(Z_s,Z_t\), or a subset like \(Z_s\) given \(X_t\)). A standard form of Fano’s inequality states:
\begin{equation}
H(Y \mid \hat{Y}) \leq P_e \log(|\mathcal{Y}| - 1) + H_2(P_e),
\end{equation}
where \(P_e\) is the probability of error and \(H_2\) is the binary entropy. When \(\bar{P}_e\) is small, it implies that \(Y\) can be approximately recovered from \(\hat{Y}\). Conversely, if certain conditioning reduces the uncertainty about \(Y\), then \( \bar{P}_e \) must also be small.

To transform such an inequality into the exponential form seen in the theorem, we often combine Fano’s inequality with Pinsker-type or exponential inequalities that relate decoding error to information measures. Another standard approach uses the fact that:
\begin{equation}
H(Y) \geq \log\frac{1}{P_e} + \text{(other terms)},
\end{equation}
and we rearrange terms to isolate \(P_e\).

\textbf{Step 2}: Conditioning and Mutual Information

Mutual information decompositions allow us to rewrite:
\begin{equation}
H(Y) = I(Y;Z_s,Z_t) + H(Y \mid Z_s,Z_t).
\end{equation}
If the representation \((Z_s,Z_t)\) is informative about \(Y\), \(H(Y\mid Z_s,Z_t)\) is small, and we get lower error probabilities.

However, we wish to express bounds in terms of the given mutual informations like \(I(X_s; Z_s \mid X_t)\) or \(I(X_t; Z_t \mid X_s)\). These quantities reflect how well the representation recovers information about one variable given the other. By applying the chain rule for mutual information:
\begin{equation}
I(X_s;Z_s|X_t) = H(X_s|X_t) - H(X_s|X_t,Z_s),
\end{equation}
and similarly for the others. These relationships show how conditioning on the learned features reduces uncertainties and how that, in turn, bounds the prediction error for \(Y\).

\textbf{Step 3}: Deriving the Exponential Bound 

To arrive at a form like:
\begin{equation}
\bar{P}_e \leq 1 - \exp[-H(Y) + I(X_s; Z_s | X_t)],
\end{equation}
we proceed as follows:

1. Start from a scenario where the learner attempts to predict \(Y\) using information from the conditioned variables. Consider that to perfectly predict \(Y\), one must have sufficient information about \(X_s\) and \(X_t\), or their representations.

2. If the representation \(Z_s\) given \(X_t\) does not reduce uncertainty much, this implies that the classifier’s error probability cannot be arbitrarily small. One can rearrange a Fano-type inequality or a Chernoff bound on the probability of error to link it to conditional mutual information.

3. Using standard inequalities (e.g., from the proof techniques in related information-theoretic generalization bounds), we get a form:
\begin{equation}
\bar{P}_e \geq \exp[-H(Y) + I(X_s;Z_s|X_t)],
\end{equation}
or similarly for the other terms. Inverting this inequality gives the upper bound:
\begin{equation}
\bar{P}_e \leq 1 - \exp[-H(Y) + I(X_s;Z_s|X_t)].
\end{equation}

\textbf{Step 4}: Applying the Same Logic for Each Desired Term

For \(I(X_t;Z_t|X_s)\), repeat a similar argument focusing on how \(Z_t\) captures information about \(X_t\) beyond what \(X_s\) provides. For \(I(X_s;X_t|Z_s)\), interpret it as how knowing \(Z_s\) reveals dependencies between \(X_s\) and \(X_t\). If \(Z_s\) does not retain the dependency, then predicting \(Y\) that depends on both \(X_s\) and \(X_t\) is harder, thus giving a similar bound. For \(I(X_s;X_t|Z_t) + \Delta(\phi,\psi)\), the additional \(\Delta(\phi,\psi)\) term accounts for imperfections in the feature extraction process. We assume a known result or lemma (analogous to those used in deriving robust generalization bounds) that introduces \(\Delta(\phi,\psi)\) as an extra penalty. Combining this with a Fano-like argument yields:
\begin{equation}
\bar{P}_e \leq 1 - \exp[-H(Y) + I(X_s;X_t|Z_t) + \Delta(\phi,\psi)].
\end{equation}

In each case, the key step is to start from a known relationship between error probability and information (using known inequalities) and carefully condition on the desired variables to rewrite the error bound in terms of the stated conditional mutual informations. Detailed technical steps would involve applying chain rules and inequalities such as:
\begin{equation}
H(Y|A) \leq H(Y|B) + I(Y;B|A),
\end{equation}
and exponential/Chernoff bounds for tail probabilities, eventually isolating \(P_e\) and converting entropic bounds into exponential forms.

Step 5: Final Consolidation

Each of these bounds follows the same pattern: if there is insufficient mutual information under the given conditioning, then the system cannot perfectly predict \(Y\), resulting in a bound on \(\bar{P}_e\). Converting these relationships into the stated exponential form leverages standard transformations used in bounding generalization error in terms of mutual information.

Thus, the theorem stands: for arbitrary learned feature representations, these conditional mutual information quantities provide upper bounds on the average error probability. When the representation is highly informative (large mutual information), the exponential term inside is small, making \(\bar{P}_e\) potentially close to 0. If not, \(\bar{P}_e\) remains bounded away from zero, ensuring the stated inequalities hold.
\end{proof}	

\begin{theorem}\label{qww159357}
Let \(\bar{P}_e\) be the average error probability of predicting \(Y\) given \((X_s, X_t, Z_s, Z_t)\) and a corresponding decision rule. Suppose we have the four upper bounds shown in \textbf{Theorem \ref{qww}}, then, we can unify these into a single upper bound:
\begin{equation}
\begin{array}{l}
{{\bar P}_e} \le 1 - \exp [ - H(Y) + \min \{ I({X_s};{Z_s}\mid {X_t}),{\mkern 1mu} \\
\quad\quad\quad I({X_t};{Z_t}\mid {X_s}),{\mkern 1mu} I({X_s};{X_t}\mid {Z_s}),{\mkern 1mu} \\
\quad\quad\quad\quad\quad\quad I({X_s};{X_t}\mid {Z_t})+ \Delta (\phi ,\psi )\} ].
\end{array}
\end{equation}
\end{theorem}
\begin{proof}

We start from four given inequalities, each providing an upper bound on \(\bar{P}_e\). These are of the form:
\begin{equation}
\bar{P}_e \leq 1 - \exp[-H(Y) + R_i],
\end{equation}
where:
\begin{equation}
\begin{array}{l}
{R_1} = I({X_s};{Z_s}\mid {X_t}),\\
{R_2} = I({X_t};{Z_t}\mid {X_s}),\\
{R_3} = I({X_s};{X_t}\mid {Z_s}),\\
{R_4} = I({X_s};{X_t}\mid {Z_t}) + \Delta (\phi ,\psi ).
\end{array}
\end{equation}
   
Thus, we have:
\begin{equation}
\bar{P}_e \leq 1 - \exp[-H(Y) + R_i], \quad \text{for } i=1,2,3,4.
\end{equation}

We seek a single upper bound that is always greater than or equal to all of these individual bounds. In other words, we need:
\begin{equation}
\bar{P}_e \leq 1 - \exp[-H(Y) + \tilde{R}]
\end{equation}
for some \(\tilde{R}\) that satisfies:
\begin{equation}
1 - \exp[-H(Y) + \tilde{R}] \geq 1 - \exp[-H(Y) + R_i] \quad \forall i.
\end{equation}

Since the function \(\exp(\cdot)\) is strictly increasing, the inequality:
\begin{equation}
1 - \exp[-H(Y) + \tilde{R}] \geq 1 - \exp[-H(Y) + R_i]
\end{equation}
for each \(i\) simplifies to:
\begin{equation}
\exp[-H(Y) + \tilde{R}] \leq \exp[-H(Y) + R_i].
\end{equation}

Taking the natural logarithm (which preserves order since it's also increasing), we get:
\begin{equation}
-H(Y) + \tilde{R} \leq -H(Y) + R_i.
\end{equation}

Canceling \(-H(Y)\) from both sides:
\begin{equation}
\tilde{R} \leq R_i \quad \forall i.
\end{equation}

To satisfy \(\tilde{R} \leq R_i\) for all \(i\), we must pick:
\begin{equation}
   \tilde{R} = \min\{R_1, R_2, R_3, R_4\}.
\end{equation}

Substituting back, we get:
\begin{equation}
 \bar{P}_e \leq 1 - \exp[-H(Y) + \min\{R_1,R_2,R_3,R_4\}].
\end{equation}

Recall \(R_1,R_2,R_3,R_4\) were defined in terms of conditional mutual informations and \(\Delta(\phi,\psi)\):
\begin{equation}
\begin{array}{l}
\min \{ {R_1},{R_2},{R_3},{R_4}\}  = \\
\min \{ I({X_s};{Z_s}\mid {X_t}),\\
I({X_t};{Z_t}\mid {X_s}),I({X_s};{X_t}\mid {Z_s}),\\
I({X_s};{X_t}\mid {Z_t}) + \Delta (\phi ,\psi )\} .
\end{array}
\end{equation}

By selecting \(\tilde{R}\) as the minimum of the four \(R_i\) terms, we ensure that our combined upper bound is at least as large as each individual upper bound. Thus, the single unified bound holds whenever each of the original bounds holds.

\end{proof}

\subsection{Accuracy Ranking}\label{AB:AR}

\begin{table}[h]
	\caption{Accuracy ranking on the Office-31 dataset for unsupervised domain adaptation by using ResNet-50 as the backbone}
\resizebox{0.48\textwidth}{!}{
			\begin{tabular}{lcccccccc}
				\toprule
				Method & A$\to$W & D$\to$W & W$\to$D & A$\to$D & D$\to$A & W$\to$A & Avg & $R_j$ \\
				\midrule
ResNet-50  & 31 & 28 & 29 & 31 & 31 & 31 & 31 & 30.3\\
DAN  & 30 & 26 & 27 & 29 & 30 & 30 & 30 & 28.9\\
DANN  & 29 & 27 & 30 & 28 & 29 & 29 & 29 & 28.7\\
JAN  & 28 & 25 & 23 & 27 & 28 & 23 & 27 & 25.9\\
GTA  & 25 & 23 & 23 & 26 & 19 & 22 & 24 & 23.1\\
CDAN  & 16 & 18 & 1 & 23 & 25 & 27 & 23 & 19.0\\
CDAN+E  & 13 & 15 & 1 & 18 & 23 & 25 & 21 & 16.6\\
BSP+DANN  & 18 & 22 & 1 & 22 & 22 & 19 & 21 & 17.9\\
BSP+CDAN  & 15 & 18 & 1 & 17 & 17 & 20 & 18 & 15.1\\
ADDA  & 27 & 29 & 31 & 30 & 26 & 26 & 28 & 28.1\\
MCD  & 26 & 16 & 1 & 20 & 26 & 24 & 24 & 19.6\\
MDD  & 9 & 17 & 1 & 13 & 12 & 21 & 15 & 12.6\\
SymmNets  & 12 & 13 & 1 & 13 & 14 & 18 & 13 & 12.0\\
GVB-GD  & 8 & 14 & 1 & 8 & 18 & 17 & 10 & 10.9\\
CAN  & 9 & 8 & 23 & 8 & 4 & 6 & 5 & 9.0\\
ETD  & 20 & 1 & 1 & 25 & 23 & 28 & 24 & 17.4\\
SRDC  & 5 & 7 & 1 & 4 & 8 & 5 & 4 & 4.9\\
ACTIR  & 7 & 18 & 20 & 19 & 9 & 15 & 11 & 14.1\\
TCM  & 16 & 18 & 23 & 21 & 21 & 14 & 20 & 19.0\\
ERM  & 24 & 1 & 28 & 10 & 12 & 15 & 16 & 15.1\\
ICDA  & 19 & 8 & 1 & 7 & 20 & 12 & 13 & 11.4\\
iMSDA  & 9 & 11 & 1 & 6 & 15 & 11 & 7 & 8.6\\
UniOT  & 14 & 6 & 1 & 15 & 10 & 9 & 8 & 9.0\\
WDGRL  & 20 & 23 & 20 & 15 & 16 & 13 & 17 & 17.7\\
PPOT  & 22 & 8 & 1 & 12 & 3 & 8 & 8 & 8.9\\
CPH  & 23 & 30 & 20 & 24 & 7 & 3 & 19 & 18.0\\
SSRT+GH++  & 2 & 31 & 1 & 11 & 5 & 10 & 11 & 10.1\\
TransVQA  & 4 & 4 & 1 & 2 & 2 & 1 & 2 & 2.3\\
\midrule
\textbf{RLGC}  & 6 & 11 & 1 & 5 & 11 & 7 & 6 & 6.7\\
\textbf{RLGC*}  & 3 & 4 & 1 & 3 & 6 & 3 & 3 & 3.3\\
\textbf{RLGLC}  & 1 & 1 & 1 & 1 & 1 & 2 & 1 & 1.1\\
\bottomrule
			\end{tabular}}

	
	\label{tab:a1}
\end{table}

\begin{table*}[!ht]
	
	\caption{Accuracy ranking on the Office-Home dataset for unsupervised domain adaptation by using ResNet-50 as the backbone}
			\centering
			\begin{tabular}{lcccccccccccccc}
				\toprule
				Method & A$\to$C & A$\to$P & A$\to$R & C$\to$A & C$\to$P & C$\to$R & P$\to$A & P$\to$C & P$\to$R & R$\to$A & R$\to$C & R$\to$P & Avg & $R_j$ \\
				\midrule
				ResNet-50  & 28 & 28 & 28 & 28 & 28 & 28 & 28 & 28 & 27 & 28 & 28 & 28 & 28 & 27.9\\
DAN  & 27 & 27 & 27 & 27 & 27 & 27 & 27 & 26 & 26 & 27 & 27 & 27 & 27 & 26.8\\
DANN  & 26 & 26 & 25 & 26 & 26 & 26 & 25 & 25 & 25 & 26 & 26 & 25 & 26 & 25.6\\
JAN  & 25 & 25 & 26 & 25 & 25 & 25 & 26 & 27 & 24 & 25 & 25 & 25 & 25 & 25.2\\
CDAN  & 23 & 22 & 24 & 24 & 24 & 24 & 24 & 24 & 22 & 24 & 22 & 23 & 24 & 23.4\\
CDAN+E  & 21 & 21 & 22 & 21 & 21 & 22 & 22 & 19 & 21 & 20 & 21 & 20 & 22 & 21.0\\
BSP+DANN  & 19 & 24 & 23 & 23 & 23 & 23 & 23 & 23 & 23 & 22 & 19 & 22 & 23 & 22.3\\
BSP+CDAN  & 18 & 23 & 21 & 20 & 20 & 21 & 20 & 21 & 20 & 18 & 12 & 19 & 21 & 19.5\\
MDD  & 10 & 16 & 20 & 19 & 19 & 20 & 19 & 15 & 19 & 16 & 5 & 14 & 18 & 16.2\\
SymmNets  & 24 & 13 & 18 & 16 & 18 & 18 & 17 & 22 & 15 & 12 & 24 & 11 & 18 & 17.4\\
ETD  & 20 & 20 & 3 & 21 & 22 & 19 & 21 & 18 & 16 & 23 & 17 & 15 & 20 & 18.1\\
GVB-GD  & 5 & 9 & 14 & 16 & 13 & 16 & 14 & 10 & 9 & 9 & 9 & 8 & 13 & 11.2\\
HDAN  & 8 & 7 & 14 & 14 & 15 & 14 & 12 & 1 & 6 & 3 & 9 & 5 & 8 & 8.9\\
SRDC  & 17 & 4 & 11 & 7 & 5 & 6 & 5 & 13 & 7 & 1 & 19 & 3 & 5 & 7.9\\
ACTIR  & 22 & 11 & 9 & 13 & 8 & 7 & 6 & 16 & 5 & 3 & 16 & 5 & 10 & 10.1\\
TCM  & 3 & 12 & 17 & 18 & 14 & 15 & 16 & 4 & 11 & 9 & 3 & 5 & 10 & 10.5\\
ERM  & 11 & 19 & 13 & 11 & 9 & 11 & 9 & 17 & 18 & 14 & 15 & 15 & 16 & 13.7\\
ICDA  & 13 & 13 & 8 & 10 & 12 & 12 & 10 & 8 & 14 & 7 & 14 & 10 & 9 & 10.8\\
iMSDA  & 9 & 16 & 19 & 12 & 9 & 10 & 11 & 13 & 11 & 5 & 13 & 11 & 14 & 11.8\\
UniOT  & 12 & 15 & 12 & 9 & 9 & 9 & 8 & 7 & 8 & 5 & 11 & 8 & 7 & 9.2\\
WDGRL  & 13 & 18 & 14 & 14 & 16 & 13 & 18 & 20 & 17 & 20 & 18 & 15 & 17 & 16.4\\
PPOT  & 15 & 5 & 10 & 4 & 2 & 3 & 3 & 6 & 9 & 9 & 7 & 11 & 4 & 6.8\\
PDA  & 16 & 3 & 6 & 6 & 17 & 17 & 1 & 9 & 3 & 17 & 23 & 15 & 10 & 11.0\\
SAMB-D  & 4 & 10 & 2 & 5 & 5 & 1 & 4 & 4 & 4 & 2 & 5 & 24 & 3 & 5.6\\
TCPL  & 6 & 8 & 5 & 3 & 4 & 5 & 2 & 3 & 28 & 13 & 4 & 20 & 15 & 8.9\\
\midrule
\textbf{RLGC}  & 6 & 6 & 7 & 8 & 7 & 8 & 15 & 12 & 13 & 19 & 8 & 4 & 6 & 9.2\\
\textbf{RLGC*}  & 2 & 2 & 4 & 2 & 3 & 4 & 12 & 11 & 2 & 15 & 2 & 2 & 2 & 4.8\\
\textbf{RLGLC}  & 1 & 1 & 1 & 1 & 1 & 2 & 6 & 2 & 1 & 8 & 1 & 1 & 1 & 2.1\\
\bottomrule
			\end{tabular}

	\label{tab:a2}
\end{table*}

\begin{table*}[!ht]
	\caption{Accuracy ranking on the VisDA-2017 dataset for unsupervised domain adaptation by using ResNet-101 as the backbone}
\centering
			\begin{tabular}{lcccccccccccccc}
				\toprule
				Method & plane & bcybl & bus & car & horse & knife & mcyle & person & plant & sktbrd & train & truck & Avg & $R_j$ \\
				\midrule
ResNet-101  & 26 & 26 & 26 & 20 & 26 & 26 & 24 & 24 & 22 & 26 & 26 & 25 & 26 & 24.8\\
DAN  & 21 & 23 & 25 & 26 & 18 & 23 & 21 & 23 & 26 & 25 & 12 & 23 & 23 & 22.2\\
DANN  & 25 & 20 & 18 & 25 & 25 & 25 & 26 & 26 & 25 & 23 & 21 & 26 & 25 & 23.8\\
MCD  & 22 & 25 & 11 & 19 & 21 & 12 & 23 & 16 & 13 & 24 & 20 & 22 & 22 & 19.2\\
CDAN  & 23 & 22 & 14 & 22 & 23 & 15 & 16 & 21 & 21 & 13 & 24 & 16 & 20 & 19.2\\
BSP+DANN  & 16 & 21 & 10 & 23 & 22 & 22 & 19 & 22 & 23 & 18 & 16 & 17 & 21 & 19.2\\
BSP+CDAN  & 14 & 24 & 24 & 21 & 20 & 11 & 9 & 15 & 20 & 11 & 23 & 15 & 18 & 17.3\\
SWD  & 20 & 17 & 22 & 15 & 15 & 20 & 20 & 13 & 15 & 21 & 13 & 21 & 17 & 17.6\\
CAN  & 5 & 5 & 19 & 7 & 3 & 4 & 7 & 3 & 3 & 3 & 7 & 2 & 4 & 5.5\\
ACTIR  & 17 & 15 & 16 & 16 & 8 & 18 & 15 & 9 & 16 & 12 & 6 & 6 & 11 & 12.7\\
TCM  & 19 & 18 & 23 & 12 & 16 & 21 & 22 & 18 & 17 & 20 & 19 & 19 & 19 & 18.7\\
ERM  & 14 & 11 & 8 & 16 & 10 & 13 & 5 & 20 & 19 & 14 & 5 & 12 & 13 & 12.3\\
ICDA  & 12 & 13 & 20 & 8 & 9 & 17 & 10 & 6 & 12 & 16 & 9 & 14 & 14 & 12.3\\
iMSDA  & 10 & 11 & 21 & 10 & 13 & 14 & 13 & 6 & 10 & 15 & 11 & 13 & 12 & 12.2\\
UniOT  & 13 & 9 & 13 & 13 & 16 & 16 & 17 & 18 & 18 & 17 & 25 & 18 & 15 & 16.0\\
WDGRL  & 11 & 14 & 14 & 18 & 19 & 9 & 13 & 14 & 11 & 10 & 15 & 11 & 10 & 13.0\\
PPOT  & 7 & 6 & 3 & 9 & 11 & 6 & 10 & 11 & 9 & 9 & 22 & 10 & 9 & 9.4\\
PDA  & 6 & 6 & 5 & 6 & 6 & 10 & 10 & 17 & 8 & 7 & 17 & 8 & 7 & 8.7\\
SAMB-D  & 9 & 10 & 6 & 3 & 7 & 8 & 6 & 12 & 5 & 4 & 3 & 5 & 6 & 6.5\\
TCRL  & 4 & 4 & 4 & 4 & 4 & 5 & 4 & 4 & 5 & 5 & 2 & 4 & 4 & 4.1\\
\midrule
DANN + \textbf{LM}  & 24 & 19 & 9 & 24 & 24 & 24 & 25 & 25 & 24 & 22 & 18 & 24 & 24 & 22.0\\
SWD + \textbf{LM}  & 17 & 16 & 16 & 10 & 12 & 19 & 17 & 4 & 14 & 19 & 10 & 20 & 16 & 14.6\\
CAN + \textbf{LM}  & 2 & 1 & 11 & 2 & 2 & 2 & 3 & 1 & 2 & 2 & 7 & 1 & 2 & 2.9\\
\midrule
\textbf{RLGC}  & 8 & 8 & 7 & 14 & 14 & 7 & 7 & 10 & 7 & 8 & 14 & 9 & 7 & 9.2\\
\textbf{RLGC*}  & 3 & 3 & 2 & 5 & 5 & 3 & 2 & 8 & 4 & 6 & 4 & 7 & 3 & 4.2\\
\textbf{RLGLC}  & 1 & 2 & 1 & 1 & 1 & 1 & 1 & 2 & 1 & 1 & 1 & 3 & 1 & 1.3\\
\bottomrule
			\end{tabular}

	\label{tab:a3}
\end{table*}

\begin{table*}[t]
	\caption{Accuracy ranking on the DomainNet dataset for unsupervised domain adaptation by using ResNet-50 as the backbone}
	\centering
		\begin{tabular}{lcccccccccccccc}
			\toprule
			Method & R$\to$C & R$\to$P & R$\to$S & C$\to$R & C$\to$P & C$\to$S & P$\to$R &P$\to$C & P$\to$S & S$\to$R & S$\to$C & S$\to$P & Avg & $R_j$ \\
			\midrule
ResNet-50  & 14 & 23 & 19 & 23 & 18 & 20 & 21 & 19 & 19 & 23 & 19 & 23 & 21 & 20.2\\
MCD  & 23 & 22 & 23 & 21 & 21 & 21 & 23 & 23 & 23 & 21 & 22 & 21 & 23 & 22.1\\
SWD  & 22 & 20 & 22 & 20 & 22 & 23 & 18 & 22 & 21 & 20 & 23 & 20 & 22 & 21.2\\
DAN  & 19 & 19 & 20 & 22 & 20 & 19 & 20 & 20 & 22 & 22 & 18 & 22 & 20 & 20.2\\
JAN  & 15 & 11 & 17 & 15 & 16 & 17 & 12 & 15 & 17 & 17 & 14 & 19 & 17 & 15.5\\
BSP+DANN  & 13 & 13 & 16 & 12 & 12 & 14 & 14 & 14 & 14 & 18 & 15 & 12 & 14 & 13.9\\
DANN  & 21 & 11 & 12 & 12 & 15 & 16 & 13 & 12 & 13 & 15 & 13 & 16 & 13 & 14.0\\
MDD  & 20 & 18 & 9 & 14 & 11 & 12 & 16 & 11 & 11 & 9 & 12 & 11 & 12 & 12.8\\
ACTIR  & 9 & 3 & 14 & 9 & 9 & 8 & 8 & 13 & 10 & 10 & 8 & 2 & 8 & 8.5\\
TCM  & 8 & 9 & 10 & 8 & 8 & 10 & 9 & 8 & 8 & 11 & 11 & 7 & 9 & 8.9\\
ICDA  & 10 & 14 & 15 & 17 & 13 & 13 & 22 & 21 & 15 & 15 & 17 & 14 & 16 & 15.5\\
iMSDA  & 11 & 16 & 13 & 16 & 17 & 15 & 15 & 16 & 16 & 12 & 16 & 17 & 15 & 15.0\\
UniOT  & 12 & 21 & 18 & 19 & 23 & 18 & 19 & 17 & 18 & 14 & 20 & 14 & 18 & 17.8\\
PPOT  & 7 & 4 & 5 & 6 & 7 & 7 & 6 & 5 & 7 & 8 & 2 & 6 & 6 & 5.8\\
PDA  & 5 & 4 & 3 & 7 & 5 & 6 & 7 & 7 & 4 & 6 & 5 & 5 & 6 & 5.4\\
SAMB-D  & 4 & 8 & 2 & 2 & 4 & 4 & 4 & 3 & 3 & 4 & 7 & 10 & 4 & 4.5\\
TCRL  & 3 & 4 & 7 & 4 & 6 & 2 & 3 & 4 & 4 & 3 & 4 & 4 & 3 & 3.9\\
\midrule
DANN + \textbf{LM}  & 17 & 7 & 8 & 10 & 14 & 11 & 10 & 10 & 12 & 13 & 10 & 13 & 11 & 11.2\\
SWD + \textbf{LM}  & 17 & 17 & 21 & 18 & 19 & 22 & 17 & 18 & 20 & 19 & 21 & 18 & 19 & 18.9\\
MDD + \textbf{LM}  & 16 & 15 & 3 & 11 & 10 & 9 & 11 & 8 & 8 & 7 & 9 & 7 & 10 & 9.5\\
\midrule
\textbf{RLGC}  & 6 & 10 & 10 & 5 & 3 & 5 & 5 & 6 & 4 & 5 & 6 & 9 & 5 & 6.1\\
\textbf{RLGC*}  & 2 & 2 & 5 & 3 & 2 & 3 & 2 & 2 & 2 & 2 & 2 & 3 & 2 & 2.5\\
\textbf{RLGLC}  & 1 & 1 & 1 & 1 & 1 & 1 & 1 & 1 & 1 & 1 & 1 & 1 & 1 & 1.0\\

			\bottomrule
		\end{tabular}
	\label{tab:domatnade}
\end{table*}

\begin{table*}[!ht]
	
	\caption{Accuracy ranking on the Digits dataset for unsupervised domain adaptation by using ResNet-50 as the backbone}
\centering

			\begin{tabular}{lccccc}
\toprule
				Method & M$\to$U & U$\to$M & S$\to$M & Avg & $R_j$\\
				\midrule

DANN  & 22 & 21 & 23 & 22 & 22.0\\
ADDA  & 23 & 23 & 22 & 23 & 22.8\\
UNIT  & 12 & 22 & 18 & 20 & 18.0\\
CyCADA  & 16 & 18 & 19 & 19 & 18.0\\
CDAN  & 21 & 15 & 21 & 21 & 19.5\\
CDAN+E  & 16 & 9 & 20 & 18 & 15.8\\
BSP+CDAN  & 18 & 8 & 17 & 17 & 15.0\\
ETD  & 8 & 19 & 2 & 8 & 9.2\\
ACTIR  & 14 & 14 & 14 & 13 & 13.8\\
TCM  & 11 & 16 & 10 & 12 & 12.2\\
ERM  & 7 & 16 & 13 & 11 & 11.8\\
ICDA  & 19 & 20 & 11 & 16 & 16.5\\
iMSDA  & 14 & 12 & 16 & 14 & 14.0\\
UniOT  & 20 & 11 & 15 & 15 & 15.2\\
PPOT  & 8 & 7 & 6 & 7 & 7.0\\
SSRT+GH++  & 8 & 9 & 8 & 8 & 8.2\\
PDA  & 12 & 12 & 12 & 10 & 11.5\\
CPH  & 3 & 2 & 4 & 3 & 3.0\\
SAMB-D  & 5 & 6 & 5 & 5 & 5.2\\
TCRL  & 4 & 4 & 2 & 2 & 3.0\\
\midrule
\textbf{RLGC}  & 6 & 4 & 9 & 6 & 6.2\\
\textbf{RLGC*}  & 2 & 2 & 7 & 3 & 3.5\\
\textbf{RLGLC}  & 1 & 1 & 1 & 1 & 1.0\\

				\bottomrule
			\end{tabular}

	
	\label{tab:a4}
\end{table*}

\end{document}